\documentclass[11pt]{article}
\usepackage{epsf,latexsym,amsfonts,amsbsy}
\usepackage{abstract,amsmath, amssymb, amsthm,amscd,amsxtra}
\usepackage{graphicx}
\usepackage{xcolor}
\usepackage{yhmath}
\usepackage{algorithm}
\usepackage{algorithmicx}
\usepackage{algpseudocode}
\usepackage{tikz}
\usetikzlibrary{calc}
\usetikzlibrary{intersections,decorations.markings}
\usepackage{booktabs}
\usepackage{threeparttable}
\usepackage{pgfplots}
\pgfplotsset{compat=1.17}
\usepackage{caption}
\usepackage{hyperref} 
\usepackage{booktabs}
\usepackage[percent]{overpic}
\usepackage{multirow} 
\usepackage{enumitem}
\usepackage{bm}
\usepackage{indentfirst}
\usetikzlibrary{arrows.meta,shapes.geometric}
\hypersetup{
	colorlinks=true,
	linkcolor=blue,
	anchorcolor=blue,
	citecolor=blue}

\usepackage[a4paper,left=1.6cm,right=1.6cm,top=2.0cm,bottom=2.0cm]{geometry}

\makeatletter 

\@addtoreset{equation}{section}

\@addtoreset{figure}{section}

\@addtoreset{table}{section}
\makeatother 

\newtheorem{theorem}{Theorem}[section]
\newtheorem{assumption}[theorem]{Assumption}
\newtheorem{definition}[theorem]{Definition}

\newtheorem{corollary}[theorem]{Corollary}

\newtheorem{lemma}[theorem]{Lemma}

\newtheorem{remark}[theorem]{Remark}

\newlength{\bibitemsep}
\newlength{\bibparskip}
\let\oldthebibliography\thebibliography
\renewcommand\thebibliography[1]{%
	\oldthebibliography{#1}%
	\setlength{\parskip}{\bibitemsep}%
	\setlength{\itemsep}{\bibparskip}%
}
\usepackage{graphicx}
\usepackage{subcaption}

\title{\textbf{SRRM: Improving Recursive Transport Surrogates in the Small-Discrepancy Regime} }

\author{Yufei Zhang, Tao Wang, Jingyi Zhang\thanks{Y.~Zhang is affiliated with the School of Mathematical Sciences, Beijing University of Posts and Telecommunications, Beijing, China.
		J.~Zhang is affiliated with the School of Mathematical Sciences, Beijing University of Posts and Telecommunications, and with the Key Laboratory of Mathematics and Information Networks (Beijing University of Posts and Telecommunications), Ministry of Education, Beijing, China.
		T.~Wang is affiliated with the Institute of Statistics and Big Data, Renmin University of China, Beijing, China. Correspondence: J.~Zhang .Email: jyzhang2024@bupt.edu.cn}}
\date{}

\begin{document}
	\maketitle
	
	\begin{abstract}
	Recursive partitioning methods provide computationally efficient surrogates for the Wasserstein distance, yet their statistical behavior and their resolution in the small-discrepancy regime remain insufficiently understood. We study Recursive Rank Matching (RRM) as a representative instance of this class under a population-anchored reference. In this setting, we establish consistency and an explicit convergence rate for the anchored empirical RRM under the \(L_2\) cost. We then identify a dominant mismatch mechanism responsible for the loss of resolution in the small-discrepancy regime. Based on this analysis, we introduce Selective Recursive Rank Matching (SRRM), which suppresses the resulting dominant mismatches and yields a higher-fidelity practical surrogate for the Wasserstein distance at moderate additional computational cost.
			
	\end{abstract}
	

	\section{Introduction}

	   \indent Optimal Transport provides a unified theoretical framework for comparing and aligning probability distributions\cite{monge1781,kantorovich2006problemofmonge}, and has been widely adopted in tasks such as point cloud registration\cite{kolouri2019gsw,wang2025gaussian}, shape interpolation\cite{bonneel2011,mccann1997}, and color transfer\cite{refRabin2014,refMeng2019}. However, classical OT typically requires solving large-scale optimization problems, and its computational cost grows rapidly with the sample size\cite{peyre2019ot,zhang2021review,li2023gw,zhang2023projection}. To improve scalability, existing approaches introduce approximations such as entropic regularization\cite{cuturi2013sinkhorn,li2023importance} and sliced methods\cite{bonneel2015sliced,li2023screening}. Nevertheless, these methods often involve a pronounced trade-off among speed, accuracy, and stability, and they remain challenging in applications that require high-precision maps.
	   
	   In recent years, a class of fast OT approximation approaches based on
	   recursive partitioning has attracted increasing attention. The core idea is to recursively bisect the space, for example via an axis-recursive mass-median partition. At the population level, this recursive construction defines an encoding through the left/right branch paths on the resulting partition tree, thereby inducing a structured correspondence between two measures. In empirical implementations, this correspondence is realized by matching samples one-to-one within cells at the same depth, yielding the No-collision Transportation Map (NCTM) proposed in \cite{ref10}. Building on this construction, \cite{ref11} further introduced the corresponding No-collision Transportation Map Distance (NCTMD).
	   
	   Space-filling-curve methods also draw on the idea of recursive partitioning. For instance, the Hilbert Curve Projection (HCP) distance\cite{li2024hcp,li2025hyperbolic,wang2026himap} maps high dimensional samples to a one dimensional sequence via a Hilbert space-filling curve, and then constructs an approximate transport plan along this sequence to estimate the transport cost. In practice and in subsequent related work, its recursive partition structure can be further exploited to develop faster curve construction and traversal algorithms\cite{ref18,ref68,ref69}, thereby improving computational efficiency and enhancing empirical robustness while maintaining approximation quality.

	   BSP-OT\cite{ref12} further advances this paradigm by merging transport maps from multiple partitioning schemes, yielding a more accurate map and reducing the bias of any single partition. Yet two key issues remain open. First, the empirical use of recursive partitioning methods still calls for a basic statistical justification. Second, the mechanism by which these methods lose accuracy in the small-discrepancy regime remains unclear.
	   
	   To address the first question in a clean and analytically tractable setting, we study \emph{Recursive Rank Matching} (RRM), a representative recursive partitioning surrogate. Whereas more general recursive partitioning methods may allow flexible choices of splitting directions and quantile thresholds, RRM adopts axis-recursive partitioning with equal-mass bisection, yielding a particularly transparent and computationally attractive construction. Under a population-anchored reference, we prove consistency and derive an explicit convergence rate for the anchored empirical RRM under the \(L_2\) cost.
	   
	   To address the second question, we identify the dominant mismatch mechanism that limits the performance of recursive partitioning methods in the small-discrepancy regime. In this regime, small perturbations near early partition boundaries can destabilize the recursive addresses and amplify local mismatches in the resulting alignment. We refer to this phenomenon as the \emph{last-mile problem}: the difficulty of producing a reliable transport-type alignment when two distributions are very close.
	   
	   As illustrated in Figure~\ref{fig:five-in-a-row1}, when two sampled distributions are very close, the ideal transport is dominated by short-range local correspondences. Recursive partitioning methods, however, may fail to preserve these nearby correspondences because the induced alignment is constrained by recursive addresses. Consequently, some points are matched over relatively long distances even when much closer counterparts are available. Among the methods discussed above, this phenomenon is particularly pronounced in HCP, which relies on a single partition induced by the curve projection, whereas BSP can partially alleviate it by aggregating maps obtained from multiple partitions.
	   
	   \begin{figure}[t]
	   	\centering
	   	\begin{subfigure}{0.3\textwidth}
	   		\centering
	   		\includegraphics[width=\linewidth]{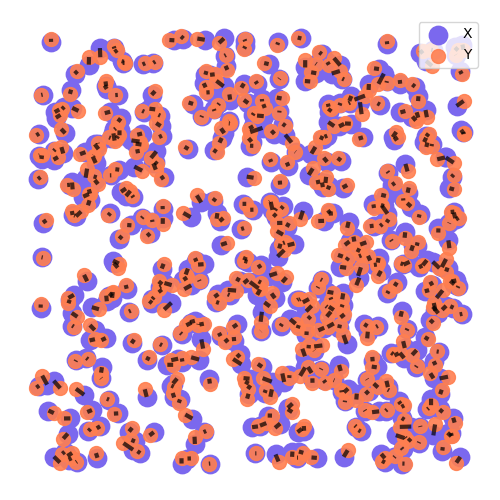}
	   		\caption{TRUE map}
	   	\end{subfigure}
	   	\begin{subfigure}{0.3\textwidth}
	   		\centering
	   		\includegraphics[width=\linewidth]{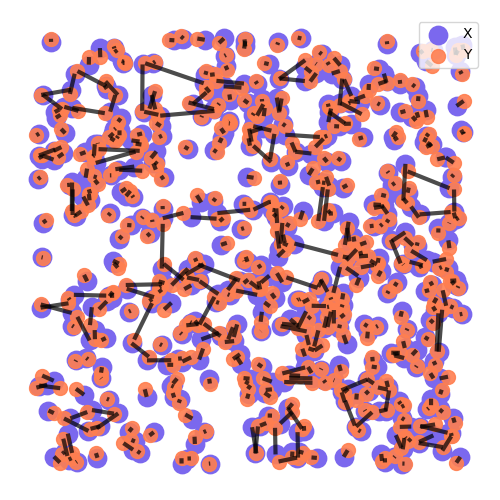}
	   		\caption{HCP}
	   	\end{subfigure}
	   	\begin{subfigure}{0.3\textwidth}
	   		\centering
	   		\includegraphics[width=\linewidth]{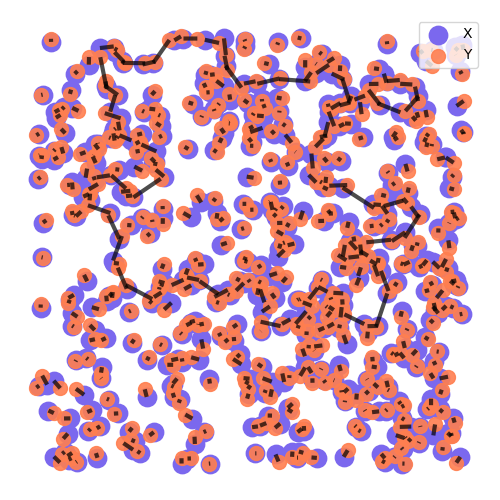}
	   		\caption{BSP}
	   	\end{subfigure}
	   	
	   	\caption{The \emph{last-mile} phenomenon of recursive partitioning methods.}
	   	\label{fig:five-in-a-row1}
	   \end{figure}
	   
	   Guided by this analysis, we propose SRRM, a selection-based refinement framework that suppresses dominant last-mile mismatches and yields a higher-quality transport coupling. In the present paper, we instantiate SRRM using repeated RRM maps, although the framework is applicable more broadly to recursive-partition methods that produce transport-type maps.Across multiple benchmark datasets and experimental settings, SRRM consistently outperforms competing baselines, demonstrating improved matching accuracy, stability, and robustness.

	   The remainder of this paper is organized as follows.
	   Section~\ref{sec:Related Works} briefly reviews related work and empirically examines potential shortcomings.
	   Section~\ref{sec:Proposed Method} introduces the proposed RRM distance and the associated theoretical analysis.
	   Section~\ref{sec:The last-mile phenomenon and Numerical Implementation of SRRM} analyzes the last-mile phenomenon and presents the numerical implementation of SRRM.
	   Section~\ref{sec:Simulation} presents numerical simulations that illustrate and validate the theoretical results and complexity analysis developed earlier.
	   Section~\ref{sec:EXPERIMENTS} offers a detailed experimental evaluation of SRRM, demonstrating its performance and effectiveness.
	   Section~\ref{sec:CONCLUSION} discusses the limitations of our approach and outlines directions for future research.

	   \paragraph{Notation}
	   $P_\infty(\mathbb{R}^d)$ denotes the set of Borel probability measures on $\mathbb{R}^d$ with bounded support.
	   For a measure $\mu$, $\operatorname{supp}(\mu)$ denotes its support.
	   The symbol $\delta_x$ stands for the Dirac measure at $x$.
	   $\mathrm{Unif}[0,1]$ denotes the uniform distribution on $[0,1]$, $\lambda$ the Lebesgue measure on $[0,1]$.

\section{Related work}
\label{sec:Related Works}

Optimal transport has a rich theoretical foundation and diverse methodological tools. 
We briefly review fast surrogates for the Wasserstein distance and compare our method with two prominent families: slicing-based methods and spatial recursive-partitioning methods, including HCP and BSP.

\subsection{Sliced Wasserstein (SW) distance}

Sliced Wasserstein (SW)\cite{bonneel2015sliced} distance can be understood through the simplicity of discrete optimal transport in one dimension. For $d$-dimensional point sets, the sliced transportation distance is defined as the average of the one-dimensional transport costs obtained by projecting the points onto random 1D lines.
To enhance discriminative power, projection-selection variants such as max-sliced Wasserstein (Max-SW)\cite{deshpande2019maxsliced} optimize over directions.
Moreover, some methods employ nonlinear projections to capture the complex structure of data distributions, such as generalized sliced Wasserstein (GSW)\cite{kolouri2019gsw}.

However, in some cases, SW-based methods may fail to faithfully reflect the true Wasserstein distance. 
Using the Gaussian mixture distributions in Fig.~\ref{fig:two_big11}(a) as an example, we fix the source distribution (shown in blue) and displace one Gaussian component of the target distribution (shown in orange) from $(0.5,0.5)$ along the direction indicated in the figure; the offset parameter $\alpha \in [0,0.4]$ controls the displacement magnitude. 
Fig.~\ref{fig:two_big11}(b) illustrates how various distance metrics vary with respect to $\alpha$. 
Empirically, as $\alpha$ increases, the trends of SW and its variants can, in some regimes, even be opposite to that of the true Wasserstein distance.

\begin{figure}[htbp]
	\centering
	\captionsetup{justification=raggedright,singlelinecheck=false}
	\begin{overpic}[width=\linewidth]{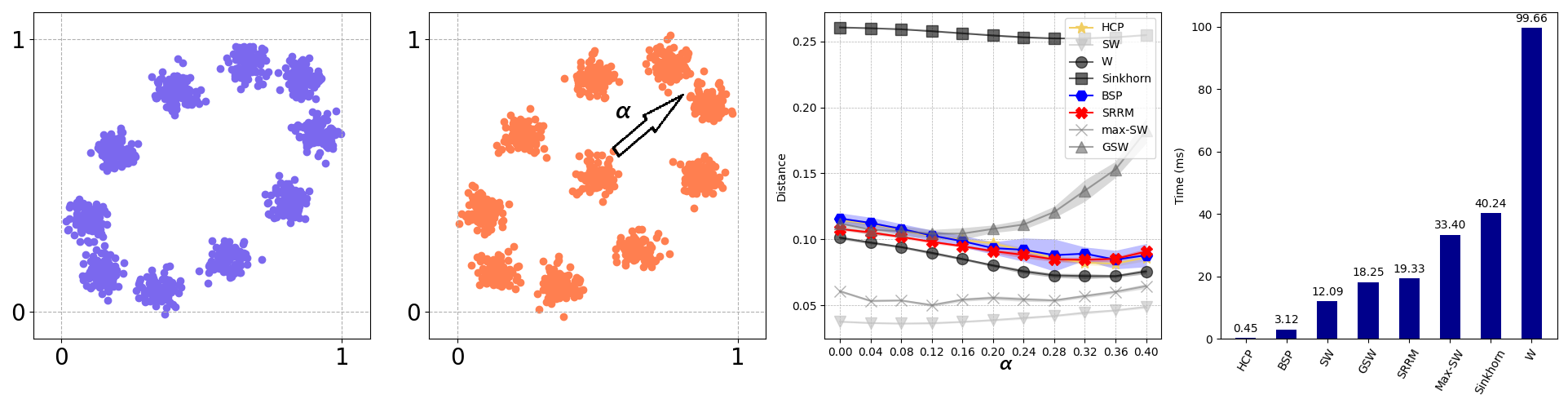}
		\put(25,0){\small (a)}
		\put(65,0){\small (b)}
		\put(85,0){\small (c)}
		
	\end{overpic}
	
	\caption{(a) Samples of the source and target distributions. 
		(b) Illustrations of how various distance metrics change as $\alpha$ increases. 
		(c) Comparisons of different distance metrics in terms of runtime.}
	\label{fig:two_big11}
\end{figure}

\subsection{Hilbert curve projection (HCP) distance}

Hilbert Curve Projection (HCP) distance \cite{li2024hcp} further applies Hilbert space-filling curves to map samples in high dimensions to a sequence in one dimension on which an approximate transport plan can be constructed to estimate the transport cost. Its effectiveness is often attributed to the locality-preserving property\cite{he_owen2016_extensible_grids,zumbusch2012_parallel_multilevel} of the Hilbert curve. However, this same property also requires high-dimensional points to be ordered according to their Hilbert-based locality when mapped into a one-dimensional sequence. 
	
	Taking the two-dimensional case as an example, the Hilbert space-filling curve is illustrated in Fig.~\ref{fig:example2}.
	Here, lighter colors indicate locations that are farther from the starting point of the curve. One can observe that, in the two-dimensional plane, two points in the lower-left and lower-right corners may be close in the Euclidean sense, yet be farthest apart along the traversal order of the curve. In other words, ``closeness'' measured by the Euclidean distance does not necessarily imply ``closeness'' in the one-dimensional ordering induced by the Hilbert curve. This mismatch is often even more pronounced in higher-dimensional settings.

	\begin{figure}[h]  
			\centering  
			\includegraphics[width=0.9\textwidth]{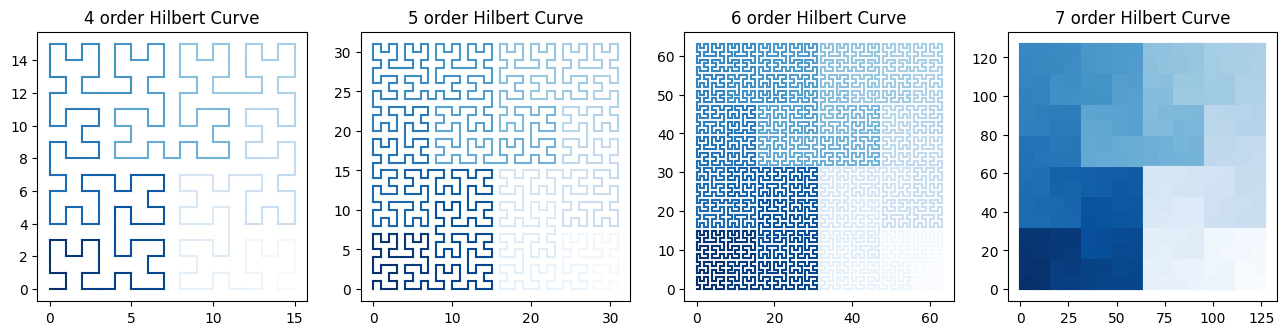}  
			\caption{Hilbert space-filling curve Visualization.}  
			\label{fig:example2}  
		\end{figure}
	
	To make this mismatch concrete, we construct paired synthetic point clouds as shown in Fig.~\ref{fig:example3}. Specifically, the Gaussian cluster center of the source distribution (purple) is fixed at $\mu_X = (0.2,0.5)$, whereas the Gaussian cluster center of the target distribution (orange) moves vertically along the line $x=0.8$, from the starting point $(0.8,0.2)$ to the ending point $(0.8,0.8)$. Empirically, when the HCP distance is used to measure discrepancies between distributions, the target Gaussian center exhibits an apparent tendency to ``approach'' the source center along the Hilbert space-filling curve induced ordering, which introduces a bias inconsistent with the true geometric displacement and can lead to distorted distance estimates.

	\begin{figure}[h]  
			\centering  
			\includegraphics[width=0.9\textwidth,height=3.5cm]{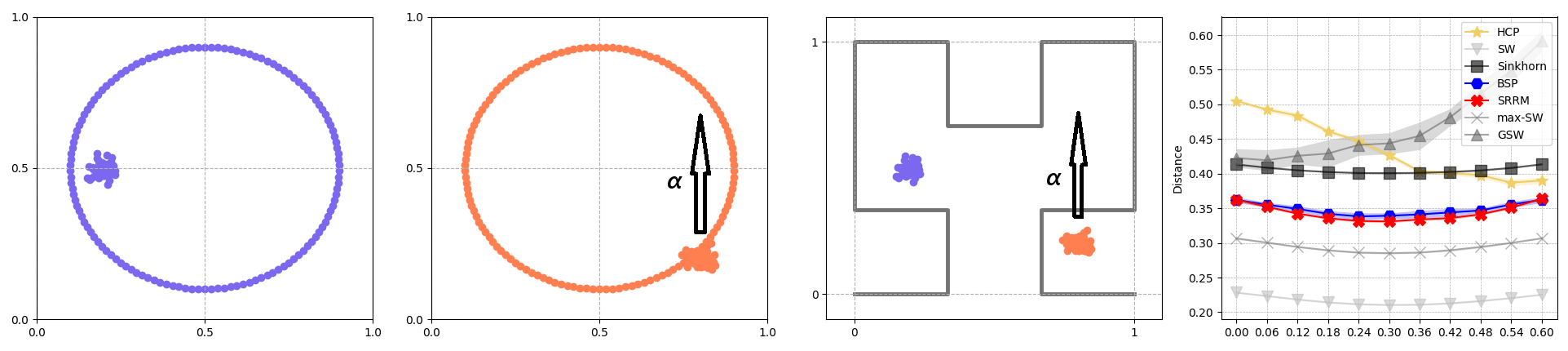}  
			\caption{Samples of the distributions and illustrations of how various distance metrics change.}  
			\label{fig:example3}  
		\end{figure}

    In this work, we mainly consider a recursive partitioning implementation of this idea, as in \cite{ref18,ref68} and in implementations based on the C++ library CGAL \cite{ref69}. Instead of explicitly evaluating Hilbert indices, this method recursively bisects the sample along coordinate axes, with approximately equal mass on each side, and constructs a Hilbert-like matching path for the samples. In our experiments, this recursive implementation yields better empirical performance than the traditional fixed-curve implementation.

\subsection{Binary space partitioning (BSP) distance}
	
The Binary Space Partitioning (BSP) distance \cite{ref10,ref11}  constructs a collision-free transport map by simultaneously partitioning two measures into two BSP trees (typically implemented as kd-trees in practice) and matching their leaf nodes.
Although each individual sparse coupling is not optimal, \cite{ref12} proposes an efficient merging scheme that exploits sparsity to substantially reduce the transportation cost and, in practice, yields approximations close to the optimal transport solution.
Despite its strong empirical performance across a variety of experiments, the method still exhibits the last-mile phenomenon.

\begin{figure}[htbp]  
	\centering  
	\includegraphics[width=\textwidth]{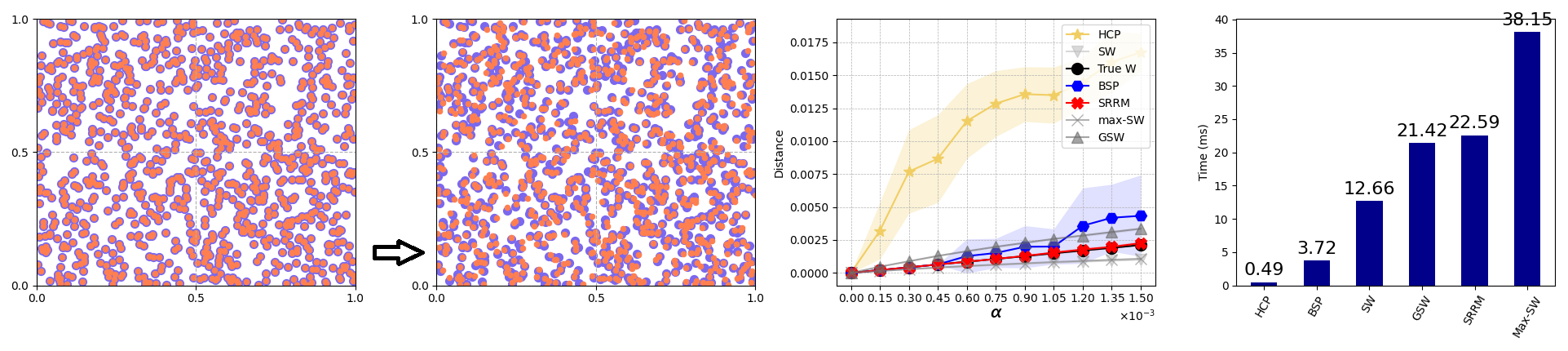}  
	\caption{Samples of the last-mile problem.}  
	\label{fig:example4}  
\end{figure}

We design a controlled experiment (as shown in Fig.~\ref{fig:example4}) to verify this phenomenon by fixing the point set $\{x_i\}$ and constructing its perturbed counterpart $\{x_i+\alpha \epsilon_i\}$, where $\epsilon_i\sim\mathcal{N}(0,I_d)$.
We vary the perturbation scale $\alpha$ in the range $[0,\,0.0015]$, so that the target point set deviates progressively from $\{x_i\}$ as $\alpha$ increases.
Under this setup, BSP still exhibits a last-mile problem under sufficiently small perturbations.
Moreover, since HCP also relies on a recursive construction but is effectively only applied once, this issue becomes even more pronounced for HCP.

\section{Proposed method}
\label{sec:Proposed Method}

This section introduces the RRM framework as a simple and analyzable representative of recursive partitioning transport surrogates. Specifically, we consider a class of probability measures on $\mathbb{R}^d$ with bounded support. 
Section~\ref{sec:Recursive Rank Matching Distance} introduces the recursive partition construction and defines the RRM distance based on this.
Section~\ref{sec:Statistical} studies the statistical behavior of the empirical estimator and analyzes its convergence properties.

\subsection{Recursive rank matching distance}
\label{sec:Recursive Rank Matching Distance}

To construct a mass-aligned parametrization for multidimensional distributions on $\mathbb{R}^d$,
we introduce an axis-recursive equal-mass partition based on coordinate-wise conditional median splits.

We follow the equal-mass splitting construction in~\cite{ref10}, but make two modifications:
(i) at depth $h$ we deterministically split along the cycling coordinate
$j(h)=1+(h \bmod d)$ (instead of using random directions); and
(ii) we encode left/right choices using the natural dyadic expansion associated with the underlying binary tree,
which simplifies subsequent derivations and notation.
A full recursion is given in Appendix D~\ref{definition1}.

\begin{remark}
	Unlike \cite{ref10}, which uses a ternary parametrization to resolve endpoint ambiguities,
	we adopt a dyadic parametrization consistent with the binary partition tree.
	Because dyadic endpoints form a countable set of Lebesgue measure zero,
	this choice does not affect the construction and simplifies the analysis.
\end{remark}

This induces a binary recursive partition $\{Q_{h,k}\}_{h\ge0,\,1\le k\le2^h}$ of a bounding box
$Q_{0,1}\subset \mathbb R^d$ containing the support of $\mu$, i.e., $\mathrm{supp}(\mu)\subset Q_{0,1}$.
Each cell $Q_{h,k}$ is identified by its unique 0/1 path code $(s_1,\ldots,s_h)\in\{0,1\}^h$.
Associate to this code the dyadic interval
$
I(s_1,\ldots,s_h):=
\Big[
\sum_{i=1}^h s_i2^{-i},\,
\sum_{i=1}^h s_i2^{-i}+2^{-h}
\Big),
$
and for each $h\ge0$ define the cell-valued map
$T_\mu^{(h)}:[0,1]\to\{Q_{h,k}\}_{k=1}^{2^h}$ by
$
T_\mu^{(h)}(t):=Q_{h,k}
\quad\text{whenever}\quad
t\in I\bigl(s_1(Q_{h,k}),\ldots,s_h(Q_{h,k})\bigr).
$
Consequently,
\[
(T_\mu^{(h)})^{-1}(Q_{h,k})
=
I\bigl(s_1(Q_{h,k}),\ldots,s_h(Q_{h,k})\bigr),
\lambda\!\left((T_\mu^{(h)})^{-1}(Q_{h,k})\right)
=
\mu(Q_{h,k})
=
2^{-h}.
\]

\begin{definition}[Mass-median axis-recursive RRM tree curve and induced transport map]\label{lem:conditional-holder1}
	Denote
	$
	Q_h(t):=T_\mu^{(h)}(t).
	$
	Then for any fixed $t$, the family $\{Q_h(t)\}_{h\ge0}$ is nested.
	Under the axis-cycling recursive construction, \cite{ref10} shows that
	$
	\mathrm{diam}(Q_h(t))\to 0
	$
	for all $t$ outside the set of dyadic endpoints; hence
	$
	\bigcap_{h\ge0}Q_h(t)
	$
	is a singleton for all such $t$.
	We define the point-valued map
	\[
	T_\mu:[0,1]\to\mathbb{R}^d,\qquad
	T_\mu(t):=\text{the unique point in }\bigcap_{h\ge0}Q_h(t),
	\]
	and assign an arbitrary value to $T_\mu(t)$ on dyadic endpoints.
	
	
\end{definition}

Our construction is conceptually close to \cite{ref10,ref11}, but we make explicit the common one-dimensional parametrization underlying the comparison. Namely, each measure is mapped to a shared parameter domain $t\in[0,1]$, and the RRM distance is defined by comparing the induced $T_\mu$ at matched values of $t$. For completeness, we restate the construction and distance definition in our own notation.

Let $\mu\in\mathcal P_\infty(\mathbb R^d)$. Consider the axis-recursive mass-median partition of $\mathbb R^d$ together with the induced dyadic parametrization on $[0,1]$.

\begin{lemma}[Pushforward property of the induced map]
	\label{lem:pushforward_treecurve}
	Let $T_\mu:[0,1]\to\mathbb{R}^d$ in Definition~\ref{lem:conditional-holder1}.
	Then $T_\mu$ is Borel measurable and satisfies
	$
	(T_\mu)_\# \mathrm{Unif}[0,1] = \mu.
	$
\end{lemma}

It ensures that $T_\mu$ is a valid transport-type parameterization of $\mu$,
so that different measures can be compared on the common parameter domain $t\in[0,1]$.

\begin{definition}\label{lem:conditional-holder3}
	Let $\mu,\nu\in\mathcal P_\infty(\mathbb{R}^d)$, and let $T_\mu,T_\nu:[0,1]\to\mathbb R^d$ denote the induced RRM maps from Definition~\ref{lem:conditional-holder1}. We define the RRM distance by
	\begin{equation}\label{eq:conditional}
		\mathrm{RRM}(\mu,\nu)
		:=
		\left(
		\int_0^1
		\big\|
		T_\mu(t)-T_\nu(t)
		\big\|_2^{\,2}\,dt
		\right)^{1/2}.
	\end{equation}
\end{definition}

With this common parametrization in hand, we measure discrepancy by the $L_2$ distance between the induced maps, yielding the RRM distance.

\begin{theorem}\label{lem:conditional-holder4}
	Assume that $\mathrm{RRM}(\mu,\nu)$ is constructed as in Definition~\ref{lem:conditional-holder3}.Then
	$
	W_2(\mu,\nu)\le \mathrm{RRM}(\mu,\nu),
	$
	and $\mathrm{RRM}$ is a metric on $\mathcal P_\infty(\mathbb{R}^d)$.
\end{theorem}

RRM distance defines a proxy geometry: it upper-bounds \(W_2\) and provides a proper metric structure for comparing multivariate distributions.

\subsection{Statistical convergence of the anchored empirical RRM}
\label{sec:Statistical}

In this section, we study consistency and convergence rates for the population-anchored empirical RRM.
Let \(\{x_i\}_{i=1}^n \sim \mu\), and denote the associated empirical measure by
$
\mu_n=\frac{1}{n}\sum_{i=1}^n \delta_{x_i}.
$

Directly analyzing the statistical convergence of the fully empirical RRM is challenging, because the underlying mass-median partition is itself data-dependent: although the axis schedule and the tree-consistent traversal order are fixed, the recursive splitting thresholds are determined by the measure. Consequently, the partition tree induced by the empirical measure \(\mu_n\) may differ from that induced by the population measure \(\mu\), which leads to sample-dependent cells and a random tree-based map \(T_{\mu_n}\). This coupling between the empirical partition and the induced transport map makes the convergence analysis substantially more delicate.

To decouple these effects, we adopt a population-anchored formulation. We measure the empirical fluctuation of $\mu_n$ through the mass accumulated on the image of the parameter prefix $[0,t]$ under $T_\mu$. Accordingly, define the right-continuous empirical prefix-mass function and its generalized inverse by
\begin{equation}\label{eq:prefix_mass_empirical}
	g_{\mu_n}(t)
	=
	\inf_{s\in\mathcal{K},\, s\ge t}\,
	\mu_n\!\bigl(T_\mu([0,s])\bigr),
	\qquad
	g_{\mu_n}^{-1}(t)
	=
	\inf_{s\in[0,1],\, g_{\mu_n}(s)>t}\; s,
\end{equation}
where \(\mathcal{K}\) denotes the set of dyadic points.
Accordingly, we define
\begin{equation}\label{eq:tcp_modified_empirical}
	\mathrm{RRM}(\mu,\mu_n)
	=
	\left(
	\int_0^1
	\bigl\|
	T_\mu(t)-T_\mu\!\bigl(g_{\mu_n}^{-1}(t)\bigr)
	\bigr\|_2^2\,dt
	\right)^{1/2}.
\end{equation}

The key point is that the tree-based ordering and prefix structure in \eqref{eq:tcp_modified_empirical} are fixed by the deterministic map $T_\mu$, so the only stochastic fluctuation comes from the one-dimensional empirical prefix-mass process $g_{\mu_n}$.
This allows us to isolate the sampling fluctuation from the additional randomness induced by the sample-dependent recursive partition.

\begin{assumption}[Bounded-density setting]
	\label{ass:mass_median_setting}
	Assume that
	$
	\operatorname{supp}(\mu)\subset Q_{0,1}\subset \mathbb{R}^d,
	$
	and that \(\mu\) admits a density \(f\) on \(Q_{0,1}\) satisfying
	$
	0<m\le f(x)\le M<\infty, \forall x\in Q_{0,1},
	$
	for some constants \(m,M>0\).
\end{assumption}

\begin{lemma}[Global H\"older control]
	\label{lem:conditional-holder}
	Let $d\ge 1$ and $T_\mu:[0,1]\to\mathbb{R}^d$ in Definition~\ref{lem:conditional-holder1}.
	Assume that the underlying probability measure $\mu$ admits a density $f$ satisfying Assumption~\ref{ass:mass_median_setting}. 
	Then there exist constants $C_d>0$ and $\alpha>0$, depending only on $d,m,M$,
	such that for every $t,t'\in[0,1)$,
	\begin{equation}\label{eq:Globa}
		\|T_\mu(t)-T_\mu(t')\|_2
		\;\le\;
		C_d\,|t-t'|^{\alpha}.
	\end{equation}
	where
	$
	\rho:=1-\frac{m}{2M}\in(0,1),
	\alpha:=\frac1d\log_2\!\Big(\frac1\rho\Big)>0.
	$
\end{lemma}

The map $T_\mu$ enjoys a global H\"older regularity under the mass-median partition,
which provides the key geometric control needed to convert one-dimensional empirical fluctuations on $[0,1]$
into an $L_2$ error bound for $\mathrm{RRM}(\mu,\mu_n)$

\begin{theorem}[Consistency and rate for the anchored empirical RRM]
	\label{thm:rrm_pure_good}
	Under Assumption~\ref{ass:mass_median_setting}, and using the Hölder control from Lemma~\ref{lem:conditional-holder}.
	Then:
	$\mathrm{RRM}(\mu,\mu_n)\to 0\,\text{almost surely as } n\to\infty.$
	Moreover,
	\begin{equation}\label{eq:rate}
		\mathbb{E}\,\mathrm{RRM}(\mu,\mu_n)
		=
		O\!\left(n^{-\min(\alpha/2,\,1/4)}\right).
	\end{equation}
\end{theorem}

The anchored one-sample result above provides the basis for the following two-sample stability statement.

\begin{corollary}[Two-sample stability of $\mathrm{RRM}$]
	\label{cor:rrm_two_sample}
	Let $\{x_i\}_{i=1}^n$ and $\{y_i\}_{i=1}^n$ be two independent i.i.d.\ samples generated from
	probability measures $\mu$ and $\nu$, and define the empirical measures
	$
	\mu_n:=\frac1n\sum_{i=1}^n\delta_{x_i},
	\nu_n:=\frac1n\sum_{i=1}^n\delta_{y_i}.
	$
	Let $\mathrm{RRM}(\mu,\nu)$ denote the population RRM defined in Definition~\ref{lem:conditional-holder3}, and define $\mathrm{RRM}(\mu_n,\nu_n)$ by the corresponding anchored empirical construction for $\mu$ and $\nu$.
	Assume the conditions of Theorem~\ref{thm:rrm_pure_good} hold for both $\mu$ and $\nu$.
	Let
	$
	\alpha_*:=\min\{\alpha_\mu,\alpha_\nu\}.
	$
	Then
	$
	\mathrm{RRM}(\mu_n,\nu_n)\to \mathrm{RRM}(\mu,\nu) \text{almost surely as }n\to\infty.
	$
	Moreover,
	\begin{equation}\label{eq:rate1}
		\mathbb{E}\Big|\mathrm{RRM}(\mu_n,\nu_n)-\mathrm{RRM}(\mu,\nu)\Big|
		=
		O\!\left(n^{-\min(\alpha_*/2,\,1/4)}\right).
	\end{equation}
\end{corollary}

This result statistically justifies the direct computation of RRM from the empirical measures \(\mu_n\) and \(\nu_n\). Although $T_\mu$ is not directly observable in practice, the empirical map $T_{\mu_n}$ serves as a finite-depth approximation to $T_\mu$.

\begin{theorem}[Finite-depth consistency of the empirical mass-median tree]\label{thm:finite-depth-tree-consistency}
	Under Assumption~\ref{ass:mass_median_setting},
	fix an integer \(H\ge 1\). For \(0\le h\le H-1\) and \(1\le k\le 2^h\), let \(m_{h,k}\) and \(m_{h,k}^{(n)}\) denote the population and empirical split thresholds, respectively, as defined by Definition~\ref{definition1} and its empirical analogue.
	Define the finite-depth threshold vectors
	$
	\mathbf m^{(H)}
	:=
	\bigl(m_{h,k}: 0\le h\le H-1,\ 1\le k\le 2^h\bigr)$ and $
	\mathbf m_n^{(H)}
	:=
	\bigl(m_{h,k}^{(n)}: 0\le h\le H-1,\ 1\le k\le 2^h\bigr).
	$
	
	Then
	$
	\mathbf m_n^{(H)}\xrightarrow[]{P}\mathbf m^{(H)}
	\text{as }n\to\infty.
	$
	Equivalently,
	$
	\max_{\substack{0\le h\le H-1\\1\le k\le 2^h}}
	|m_{h,k}^{(n)}-m_{h,k}|
	\xrightarrow[]{P}0.
	$
	
	Moreover, if
	$
	\mathcal U_n(H)
	:=
	\Bigl\{
	x\in Q_{0,1}:
	\bigl(s_1^{(n)}(x),\dots,s_H^{(n)}(x)\bigr)
	\neq
	\bigl(s_1(x),\dots,s_H(x)\bigr)
	\Bigr\},
	$
	then
	$
	\mu\bigl(\mathcal U_n(H)\bigr)\xrightarrow[]{P}0.
	$
	
	In particular, for every fixed finite resolution \(H\), the empirical tree \(T_{\mu_n}\) converges to the population tree \(T_\mu\) in the sense that both the split-threshold vector and the induced address prefixes of length \(H\) are asymptotically consistent.
\end{theorem}

Consequently, the anchored construction based on the fixed map \(T_\mu\) is not merely a formal analytical device, but a statistically meaningful proxy. At the same time, this finite-depth consistency justifies the practical use of \(T_{\mu_n}\) as an asymptotic surrogate for \(T_\mu\).

  \section{The last-mile phenomenon and SRRM}
  \label{sec:The last-mile phenomenon and Numerical Implementation of SRRM}
  
  In this section, we study the last-mile phenomenon and introduce the improved framework SRRM. Section~\ref{sec:The last-mile phenomenon} introduces the basic RRM algorithm and further discusses two main factors that affect the plateau level. Section~\ref{sec:RRM and SRRM Algorithms and Complexity} presents the improved SRRM algorithm derived from the preceding analysis.
  
  \subsection{Empirical RRM and the last-mile phenomenon}
  \label{sec:The last-mile phenomenon}
  
  Let $N\in\mathbb{N}$ and let $X=\{x_i\}_{i=1}^N\subset\mathbb{R}^d$ and
  $Y=\{y_j\}_{j=1}^N\subset\mathbb{R}^d$ be two sample sets equipped with
  uniform empirical weights $1/N$. We compute the empirical $\mathrm{RRM}$ distance
  between $X$ and $Y$ as follows.
  
  For each sample set, we build an axis-cycling recursive mass-median partition and assign to each point
  its associated binary path code.
  We then order the points by the \emph{lexicographic order} of their path codes (the RRM tree-curve order).
  This yields permutations
  $\pi_X \leftarrow \mathrm{TreeCurveOrder}(X)$ and
  $\pi_Y \leftarrow \mathrm{TreeCurveOrder}(Y)$.
  
  The induced one-dimensional matching in the common code space is therefore obtained by monotone pairing along the resulting order. Therefore, we evaluate
  the empirical $\mathrm{RRM}$ distance by
  \begin{equation}
  	\mathrm{RRM}(X,Y)
  	=
  	\left(
  	\frac{1}{N}\sum_{k=1}^N
  	\bigl\|x_{\pi_X(k)}-y_{\pi_Y(k)}\bigr\|_2^2
  	\right)^{1/2}.
  \end{equation}
  
  This process is illustrated in Fig.~\ref{fig:five_row}, and a pseudo-code implementation
  can be found in Alg.~\ref{alg:emp_rrm_uniform}.

  \begin{figure}[htbp]
  	\centering
  	\begin{subfigure}[t]{0.25\linewidth}
  		\begin{overpic}[width=\linewidth]{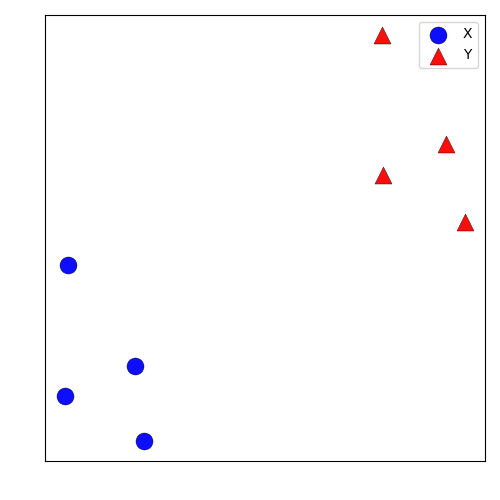}
  			\put(50,0){\small (a)}
  		\end{overpic}
  	\end{subfigure}\hspace{0.5em}
  	\begin{subfigure}[t]{0.25\linewidth}
  		\begin{overpic}[width=\linewidth]{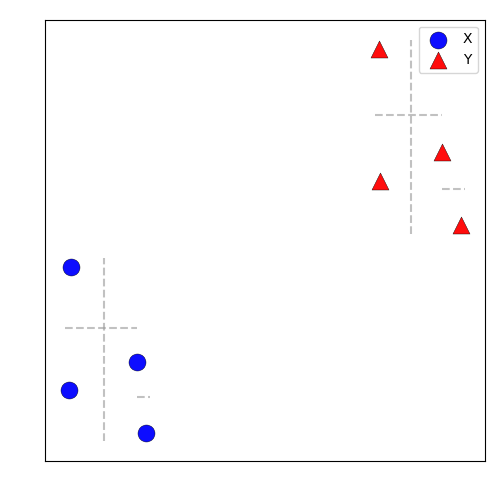}
  			\put(50,0){\small (b)}
  		\end{overpic}
  	\end{subfigure}\hspace{0.5em}
  	\begin{subfigure}[t]{0.25\linewidth}
  		\begin{overpic}[width=\linewidth]{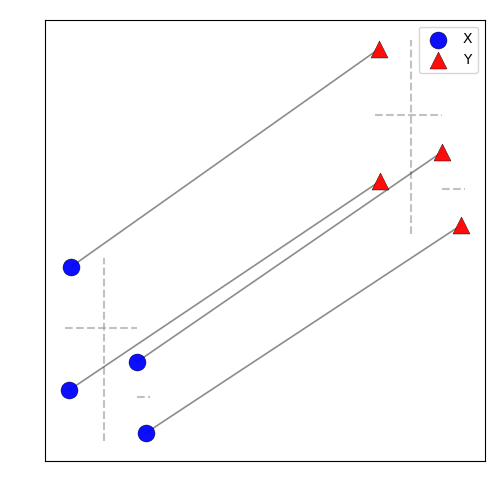}
  			\put(50,0){\small (c)}
  		\end{overpic}
  	\end{subfigure}
  	
  	\caption{An illustration of Alg.~\ref{alg:emp_rrm_uniform} when $d=2$.
  		(a) Source data points (blue) and target data points (red).
  		(b) axis-aligned recursive partitioning.
  		(c) connect corresponding points.}
  	\label{fig:five_row}
  \end{figure}

  \begin{algorithm}[htbp]
  	\caption{Empirical $\mathrm{RRM}$ (uniform weights)}
  	\label{alg:emp_rrm_uniform}
  	\begin{algorithmic}[1]
  		\Require samples $X=\{x_i\}_{i=1}^{N}$, $Y=\{y_j\}_{j=1}^{N}$
  		\Ensure distance $\mathrm{RRM}(X,Y)$
  		\State $\pi_X\gets \Call{TreeCurveOrder}{X}$;\ \ $\pi_Y\gets \Call{TreeCurveOrder}{Y}$
  		\State $\mathrm{RRM}(X,Y)\gets \left(\frac{1}{N}\sum_{k=1}^{N}\left\|x_{\pi_X(k)}-y_{\pi_Y(k)}\right\|_2^{2}\right)^{1/2}$
  		\State \Return $\mathrm{RRM}(X,Y)$
  	\end{algorithmic}
  \end{algorithm}
  
  Although empirical RRM is simple and efficient, experiments reveal a persistent bias floor in the near-zero discrepancy regime. As illustrated in Fig.~\ref{fig:ten-two-rows1}, there exist nearby cross-distribution pairs that would be matched under an ideal coupling, 
  but the empirical recursive partitions place them into different depth-$H$ leaves. 
  Under the address-restricted RRM rule, cross-leaf matching is forbidden, so these near-neighbor pairs can never be aligned. 
  As a result, many points are forced to match with a farther counterpart within its own leaf, yielding a non-vanishing mismatch penalty and hence a positive bias floor. 
  Consequently, even when the true Wasserstein distance is very small (or zero), the recursive-partition distance may remain bounded away from zero, yielding the last-mile effect.

  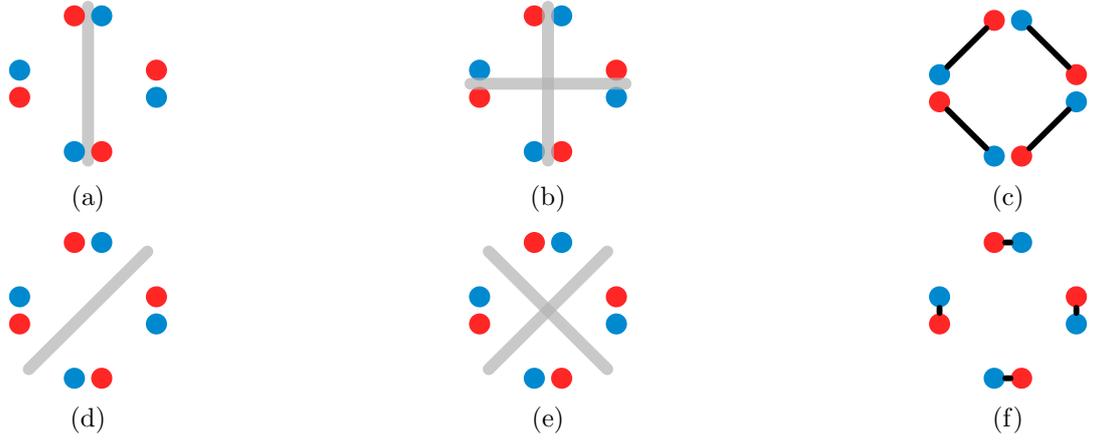
\begin{figure}[t]
  	\centering
  	
  	\begin{subfigure}{0.32\linewidth}
  		\centering
  		\begin{tikzpicture}[scale=0.6,
  			redpt/.style ={circle, fill=red!85,  draw=none, inner sep=2.8pt},
  			bluept/.style={circle, fill=cyan!70!blue, draw=none, inner sep=2.8pt},
  			cutline/.style={line width=4.5pt, draw=gray!60, opacity=0.7, line cap=round}
  			]
  			\draw[cutline] (0, 1.70) -- (0,-1.70);
  			
  			\node[redpt]  at (-0.30, 1.50) {};
  			\node[bluept] at (0.30, 1.50) {};
  			
  			\node[bluept] at (-0.30, -1.50) {};
  			\node[redpt]  at (0.30,-1.50) {};
  			
  			\node[redpt]  at (1.50, 0.30) {};
  			\node[bluept] at (1.50, -0.30) {};
  			
  			\node[redpt]  at (-1.50, -0.30) {};
  			\node[bluept] at (-1.50, 0.30) {};
  		\end{tikzpicture}
  		\caption{}
  	\end{subfigure}\hfill
  	\begin{subfigure}{0.32\linewidth}
  		\centering
  		\begin{tikzpicture}[scale=0.6,
  			redpt/.style ={circle, fill=red!85,  draw=none, inner sep=2.8pt},
  			bluept/.style={circle, fill=cyan!70!blue, draw=none, inner sep=2.8pt},
  			cutline/.style={line width=4.5pt, draw=gray!60, opacity=0.7, line cap=round},
  			match/.style={line width=2.2pt, draw=black, line cap=round}
  			]
  			\node[redpt]  at (-0.30, 1.50) {};
  			\node[bluept] at (0.30, 1.50) {};
  			
  			\node[bluept] at (-0.30, -1.50) {};
  			\node[redpt]  at (0.30,-1.50) {};
  			
  			\node[redpt]  at (1.50, 0.30) {};
  			\node[bluept] at (1.50, -0.30) {};
  			
  			\node[redpt]  at (-1.50, -0.30) {};
  			\node[bluept] at (-1.50, 0.30) {};
  			\draw[cutline]  (0, 1.70) -- (0,-1.70);
  			\draw[cutline]  (-1.70, 0) -- (1.70,0);
  		\end{tikzpicture}
  		\caption{}
  	\end{subfigure}\hfill
  	\begin{subfigure}{0.32\linewidth}
  		\centering
  		\begin{tikzpicture}[scale=0.6,
  			redpt/.style ={circle, fill=red!85,  draw=none, inner sep=2.8pt},
  			bluept/.style={circle, fill=cyan!70!blue, draw=none, inner sep=2.8pt},
  			match/.style={line width=2.2pt, draw=black, line cap=round}
  			]
  			\node[redpt]  (rt) at (-0.30,  1.50) {}; 
  			\node[bluept] (bt) at ( 0.30,  1.50) {}; 
  			
  			\node[bluept] (bb) at (-0.30, -1.50) {}; 
  			\node[redpt]  (rb) at ( 0.30, -1.50) {}; 
  			
  			\node[redpt]  (rr) at ( 1.50,  0.30) {}; 
  			\node[bluept] (br) at ( 1.50, -0.30) {}; 
  			
  			\node[redpt]  (rl) at (-1.50, -0.30) {}; 
  			\node[bluept] (bl) at (-1.50,  0.30) {}; 
  			
  			\draw[match] (rt) -- (bl);
  			\draw[match] (rr) -- (bt);
  			\draw[match] (rb) -- (br);
  			\draw[match] (rl) -- (bb);
  			
  		\end{tikzpicture}
  		\caption{}
  	\end{subfigure}
  	
  	\vspace{1.2ex}
  	
  	\begin{subfigure}{0.32\linewidth}
  		\centering
  		\begin{tikzpicture}[scale=0.6,
  			redpt/.style ={circle, fill=red!85,  draw=none, inner sep=2.8pt},
  			bluept/.style={circle, fill=cyan!70!blue, draw=none, inner sep=2.8pt},
  			cutline/.style={line width=4.5pt, draw=gray!60, opacity=0.7, line cap=round},
  			match/.style={line width=2.2pt, draw=black, line cap=round}
  			]
  			\node[redpt]  at (-0.30, 1.50) {};
  			\node[bluept] at (0.30, 1.50) {};
  			
  			\node[bluept] at (-0.30, -1.50) {};
  			\node[redpt]  at (0.30,-1.50) {};
  			
  			\node[redpt]  at (1.50, 0.30) {};
  			\node[bluept] at (1.50, -0.30) {};
  			
  			\node[redpt]  at (-1.50, -0.30) {};
  			\node[bluept] at (-1.50, 0.30) {};
  			
  			\draw[cutline] (1.30,1.30) -- (-1.30,-1.30);
  		\end{tikzpicture}
  		\caption{}
  	\end{subfigure}\hfill
  	\begin{subfigure}{0.32\linewidth}
  		\centering
  		\begin{tikzpicture}[scale=0.6,
  			redpt/.style ={circle, fill=red!85,  draw=none, inner sep=2.8pt},
  			bluept/.style={circle, fill=cyan!70!blue, draw=none, inner sep=2.8pt},
  			cutline/.style={line width=4.5pt, draw=gray!60, opacity=0.7, line cap=round},
  			match/.style={line width=2.2pt, draw=black, line cap=round}
  			]
  			\node[redpt]  at (-0.30, 1.50) {};
  			\node[bluept] at (0.30, 1.50) {};
  			
  			\node[bluept] at (-0.30, -1.50) {};
  			\node[redpt]  at (0.30,-1.50) {};
  			
  			\node[redpt]  at (1.50, 0.30) {};
  			\node[bluept] at (1.50, -0.30) {};
  			
  			\node[redpt]  at (-1.50, -0.30) {};
  			\node[bluept] at (-1.50, 0.30) {};
  			
  			\draw[cutline] (1.30,1.30) -- (-1.30,-1.30);
  			\draw[cutline] (-1.30,1.30) -- (1.30,-1.30);
  		\end{tikzpicture}
  		\caption{}
  	\end{subfigure}\hfill
  	\begin{subfigure}{0.32\linewidth}
  		\centering
  		\begin{tikzpicture}[scale=0.6,
  			redpt/.style ={circle, fill=red!85,  draw=none, inner sep=2.8pt},
  			bluept/.style={circle, fill=cyan!70!blue, draw=none, inner sep=2.8pt},
  			match/.style={line width=2.2pt, draw=black, line cap=round}
  			]
  			\node[redpt]  (rt) at (-0.30,  1.50) {}; 
  			\node[bluept] (bt) at ( 0.30,  1.50) {}; 
  			
  			\node[bluept] (bb) at (-0.30, -1.50) {}; 
  			\node[redpt]  (rb) at ( 0.30, -1.50) {}; 
  			
  			\node[redpt]  (rr) at ( 1.50,  0.30) {}; 
  			\node[bluept] (br) at ( 1.50, -0.30) {}; 
  			
  			\node[redpt]  (rl) at (-1.50, -0.30) {}; 
  			\node[bluept] (bl) at (-1.50,  0.30) {}; 
  			
  			\draw[match] (rt) -- (bt);
  			\draw[match] (rr) -- (br);
  			\draw[match] (rb) -- (bb);
  			\draw[match] (rl) -- (bl);
  			
  		\end{tikzpicture}
  		\caption{}
  	\end{subfigure}
  	
  	\caption{Top row: Figures (a)–(c) show incorrect partitioning. Bottom row: Figures (d)–(f) show correct partitioning.  }
  	\label{fig:ten-two-rows1}
  \end{figure}
  
  Since a global translation does not affect the quadratic transport geometry, we first center the samples so that their empirical barycenters coincide.
  Fix a truncation depth \(H\) and construct the depth-\(H\) recursive partition trees \(T_\mu\) and \(T_\nu\) as in Definition~\ref{lem:conditional-holder1}. For \(x,y\in\mathbb{R}^d\), let
  $
  s_H(x,y)
  :=
  \max\Bigl\{h\in\{0,\dots,H\}:
  \bigl(s_1(x),\dots,s_h(x)\bigr)=\bigl(s_1(y),\dots,s_h(y)\bigr)
  \Bigr\}.
  $
  
  Given samples $X=\{x_i\}_{i=1}^n$ and $Y=\{y_j\}_{j=1}^n$, for each $x_i$ choose a 
  nearest-neighbor (NN) partner
  $
  j_{\mathrm{nn}}(i)\in\arg\min_{j\in[n]}\ \|x_i-y_j\|_2,
  \delta_i:=\|x_i-y_{j_{\mathrm{nn}}(i)}\|_2 .
  $
  To calibrate geometry with tree depth, define the geometry-calibrated depth.
  \begin{equation}\label{eq: geometry}
  	\ell_{H,\rho}(x,y)
  	:=
  	\min\Big\{
  	H,\,
  	\Big\lceil d\log_{1/\rho}\frac{C}{\|x-y\|_2}\Big\rceil
  	\Big\},
  	\qquad
  	\ell_{H,\rho}(x,x):=H.
  \end{equation}
  
  We say that $x_i$ is prematurely separated if its NN partner
  $y_{j_{\mathrm{nn}}(i)}$ is split too early by the tree, namely if
  $
  s_H\!\big(x_i,y_{j_{\mathrm{nn}}(i)}\big)
  \;<\;
  \ell_{H,\rho}\!\big(x_i,y_{j_{\mathrm{nn}}(i)}\big).
  $
  
  \begin{definition}[Premature splitting under depth-$H$ address restriction]\label{def:premature_split_nn}
  	Define the corresponding bad index set
  	\begin{equation}\label{eq: bad index}
  		I_+
  		:=\Big\{\, i\in[n]:\ s_H\!\big(x_i,y_{j_{\mathrm{nn}}(i)}\big)
  		< \ell_{H,\rho}\!\big(x_i,y_{j_{\mathrm{nn}}(i)}\big)\,\Big\}.
  	\end{equation}
  \end{definition}
  
  \begin{theorem}[NN-referenced proportion--severity structure]\label{thm:nn_plateau}
  	Adopt the notation in Definition~\ref{def:premature_split_nn}.
  	Let $\sigma$ be any depth-$H$ address-restricted permutation.
  	For each $i$, define the NN-excess
  	$
  	\Gamma_i^{\mathrm{NN}}
  	:=\Big(\|x_i-y_{\sigma(i)}\|_2^2-\delta_i^2\Big)_+.
  	$
  	Let $I_+$ be the prematurely separated index set and define
  	$
  	\alpha_H := \frac{|I_+|}{n} , 
  	\bar{\Gamma}_H^{\mathrm{NN}}
  	:=\frac{1}{|I_+|}\sum_{i\in I_+}\Gamma_i^{\mathrm{NN}},
  	$
  	with $\bar{\Gamma}_H^{\mathrm{NN}}:=0$ if $I_+=\emptyset$.
  	Then
  	\begin{equation}\label{eq:NN}
  		\mathrm{RRM}_n^2
  		=
  		\frac1n\sum_{i=1}^n \|x_i-y_{\sigma(i)}\|_2^2
  		\ \ge\
  		\frac1n\sum_{i=1}^n \delta_i^2
  		\ +\ \alpha_H\,\bar{\Gamma}_H^{\mathrm{NN}}.
  	\end{equation}
  \end{theorem}


  Here $\delta_i:=\min_{j\in[n]}\|x_i-y_j\|_2$ is the NN matching baseline, while
  $\alpha_H:=|I_+|/n$ is the fraction of prematurely separated indices whose NN partners are
  rendered infeasible by the depth-$H$ address restriction. The term
  $\bar{\Gamma}_H^{\mathrm{NN}}$ aggregates the severity of these bad indices through the
  NN-excess penalty. 
  
  Consequently, in the near-zero discrepancy regime, the NN baseline
  $\frac1n\sum_i\delta_i^2$ is typically very small. Moreover, in many practically relevant settings, the NN baseline coincides with, or is at least very close to, the OT baseline. In this regime, the RRM plateau is therefore
  dominated by the premature-splitting term $\alpha_H\,\bar{\Gamma}_H^{\mathrm{NN}}$:
  the plateau height is governed jointly by the \emph{fraction} of bad indices $\alpha_H$ and
  their average \emph{severity} $\bar{\Gamma}_H^{\mathrm{NN}}$.
  In short, even when the underlying discrepancy is very small, a non-vanishing bias floor may
  persist as long as a nontrivial mass of prematurely separated NN pairs remains.
  
  \subsection{SRRM algorithms and complexity}
  \label{sec:RRM and SRRM Algorithms and Complexity}

  To mitigate the last-mile issue, one would ideally design each recursive split to avoid separating nearby paired points. In practice, however, controlling the splitting direction at every recursion step is cumbersome, and preventing all such separations is generally impossible. This leads us to focus on the points affected by a fixed splitting rule, and on this basis propose a point-selection heuristic algorithm, summarized as follows:

  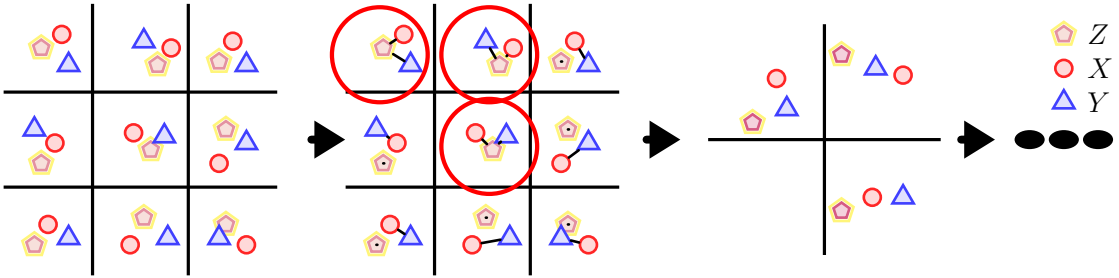
\begin{figure}[htbp]
  	\centering
  	\begin{tikzpicture}[
  		scale=0.9, transform shape,
  		>={Triangle[length=3pt,width=2pt]},
  		bigarrow/.style={-Triangle, line width=3pt, draw=black, line cap=round},
  		gridline/.style={line width=1.2pt, draw=black},
  		badcircle/.style={draw=red, line width=1.6pt},
  		zpt/.style={regular polygon, regular polygon sides=5, draw=purple!65, fill=purple!18, line width=1pt,inner sep=2pt, minimum size=3pt},
  		zptt/.style={regular polygon, regular polygon sides=5, draw=yellow!65, fill=yellow!18,fill opacity=0.35, line width=1pt,inner sep=3pt, minimum size=5pt},
  		xpt/.style={circle, draw=red!85, fill=red!15, line width=1pt,inner sep=2.5pt, minimum size=3pt},
  		ypt/.style={regular polygon, regular polygon sides=3, draw=blue!75, fill=blue!12, line width=1pt,inner sep=1.8pt, minimum size=3pt},
  		match/.style={line width=1pt, draw=black, line cap=round},
  		tinylabel/.style={font=\large}
  		]
  		
  		\newcommand{\gridthree}{%
  			\draw[gridline] (-2,  0.7) -- ( 2,  0.7);
  			\draw[gridline] (-2, -0.7) -- ( 2, -0.7);
  			\draw[gridline] (-0.7,-2)  -- (-0.7, 2);
  			\draw[gridline] ( 0.7,-2)  -- ( 0.7, 2);
  		}
  		
  		\begin{scope}[shift={(0,0)}]
  			\gridthree
  			
  			\node[zpt] at (-1.45, 1.35) {};
  			\node[zptt] at (-1.45, 1.35) {};
  			\node[xpt] at (-1.15, 1.55) {};
  			\node[ypt] at (-1.05, 1.10) {};
  			
  			\node[ypt] at ( 0.05, 1.45) {};
  			\node[xpt] at ( 0.45, 1.35) {};
  			\node[zpt] at ( 0.25, 1.10) {};
  			\node[zptt] at ( 0.25, 1.10) {};
  			
  			\node[xpt] at ( 1.35, 1.45) {};
  			\node[ypt] at ( 1.55, 1.10) {};
  			\node[zpt] at ( 1.15, 1.15) {};
  			\node[zptt] at ( 1.15, 1.15) {};
  			
  			\node[ypt] at (-1.55, 0.15) {};
  			\node[xpt] at (-1.25,-0.05) {};
  			\node[zpt] at (-1.45,-0.35) {};
  			\node[zptt] at (-1.45,-0.35) {};
  			
  			\node[xpt] at (-0.10, 0.10) {};
  			\node[zpt] at ( 0.15,-0.15) {};
  			\node[zptt] at ( 0.15,-0.15) {};
  			\node[ypt] at ( 0.35, 0.05) {};
  			
  			\node[zpt] at ( 1.25, 0.15) {};
  			\node[zptt] at ( 1.25, 0.15) {};
  			\node[ypt] at ( 1.55,-0.05) {};
  			\node[xpt] at ( 1.15,-0.35) {};
  			
  			\node[xpt] at (-1.35,-1.25) {};
  			\node[zpt] at (-1.55,-1.55) {};
  			\node[zptt] at (-1.55,-1.55) {};
  			\node[ypt] at (-1.05,-1.45) {};
  			
  			\node[zpt] at ( 0.05,-1.15) {};
  			\node[zptt] at ( 0.05,-1.15) {};
  			\node[ypt] at ( 0.40,-1.45) {};
  			\node[xpt] at (-0.15,-1.55) {};
  			
  			\node[zpt] at ( 1.25,-1.25) {};
  			\node[zptt] at ( 1.25,-1.25) {};
  			\node[xpt] at ( 1.55,-1.55) {};
  			\node[ypt] at ( 1.15,-1.45) {};
  		\end{scope}
  		
  		\draw[bigarrow] (2.5,0) -- (3,0);
  		
  		\begin{scope}[shift={(5.0,0)}]
  			\gridthree
  			
  			\node[zpt] (zq) at (-1.45, 1.35) {};
  			\node[zptt] (vq) at (-1.45, 1.35) {};
  			\node[xpt] (xq) at (-1.15, 1.55) {};
  			\node[ypt] (yq) at (-1.05, 1.10) {};
  			
  			\draw[match] (zq) -- (xq);
  			\draw[match] (vq) -- (yq);
  			
  			\node[ypt] (zw) at ( 0.05, 1.45) {};
  			\node[xpt] (vw) at ( 0.45, 1.35) {};
  			\node[zpt] (xw) at ( 0.25, 1.10) {};
  			\node[zptt] (yw) at ( 0.25, 1.10) {};
  			\draw[match] (zw) -- (xw);
  			\draw[match] (vw) -- (yw);
  			
  			\node[xpt] (ze) at ( 1.35, 1.45) {};
  			\node[ypt] (ve) at ( 1.55, 1.10) {};
  			\node[zpt] (xe) at ( 1.15, 1.15) {};
  			\node[zptt] (ye) at ( 1.15, 1.15) {};
  			\draw[match] (ze) -- (ve);
  			\draw[match] (xe) -- (ye);
  			
  			\node[ypt] (zr) at (-1.55, 0.15) {};
  			\node[xpt] (vr) at (-1.25,-0.05) {};
  			\node[zpt] (xr) at (-1.45,-0.35) {};
  			\node[zptt] (yr) at (-1.45,-0.35) {};
  			\draw[match] (zr) -- (vr);
  			\draw[match] (xr) -- (yr);
  			
  			\node[xpt] (zt) at (-0.10, 0.10) {};
  			\node[zpt] (vt) at ( 0.15,-0.15) {};
  			\node[zptt] (xt) at ( 0.15,-0.15) {};
  			\node[ypt] (yt) at ( 0.35, 0.05) {};
  			\draw[match] (zt) -- (vt);
  			\draw[match] (xt) -- (yt);
  			
  			\node[zpt] (zy) at ( 1.25, 0.15) {};
  			\node[zptt] (vy)at ( 1.25, 0.15) {};
  			\node[ypt] (xy)at ( 1.55,-0.05) {};
  			\node[xpt] (yy) at ( 1.15,-0.35) {};
  			\draw[match] (zy) -- (vy);
  			\draw[match] (xy) -- (yy);
  			
  			\node[xpt] (xu) at (-1.35,-1.25) {};
  			\node[zpt] (zu) at (-1.55,-1.55) {};
  			\node[zptt] (vu) at (-1.55,-1.55) {};
  			\node[ypt] (yu) at (-1.05,-1.45) {};
  			\draw[match] (zu) -- (vu);
  			\draw[match] (xu) -- (yu);
  			
  			\node[zpt] (zi) at ( 0.05,-1.15) {};
  			\node[zptt] (vi) at ( 0.05,-1.15) {};
  			\node[ypt] (yi) at ( 0.40,-1.45) {};
  			\node[xpt] (xi) at (-0.15,-1.55) {};
  			\draw[match] (zi) -- (vi);
  			\draw[match] (xi) -- (yi);
  			
  			\node[zpt] (zo) at ( 1.25,-1.25) {};
  			\node[zptt] (vo) at ( 1.25,-1.25) {};
  			\node[xpt] (xo) at ( 1.55,-1.55) {};
  			\node[ypt] (yo) at ( 1.15,-1.45) {};
  			\draw[match] (zo) -- (vo);
  			\draw[match] (xo) -- (yo);
  			
  			\draw[badcircle] (-1.5, 1.25) circle (0.70);
  			\draw[badcircle] ( 0.10, 1.25) circle (0.70);
  			\draw[badcircle] ( 0.10,-0.10) circle (0.70);
  		\end{scope}
  		
  		\draw[bigarrow] (7.4,0) -- (7.9,0);
  		
  		
  		\begin{scope}[shift={(10,0)}]
  			
  			\draw[gridline] (-1.7,0) -- (1.7,0);
  			\draw[gridline] (0,-1.7) -- (0,1.7);
  			
  			\node[zptt] at (-1.05,  0.25) {};
  			\node[zpt] at (-1.05,  0.25) {};
  			\node[xpt]  at (-0.70,  0.90) {};
  			\node[ypt]  at (-0.55,  0.45) {};
  			
  			\node[zptt] at ( 0.25,  1.25) {};
  			\node[zpt] at ( 0.25,  1.25) {};
  			\node[xpt]  at ( 1.15,  0.95) {};
  			\node[ypt]  at ( 0.75,  1.05) {};

  			\node[zptt] at ( 0.25, -1.05) {};
  			\node[zpt] at ( 0.25, -1.05) {};
  			\node[xpt]  at ( 0.70, -0.85) {};
  			\node[ypt]  at ( 1.15, -0.85) {};
  			
  		\end{scope}
  		
  		\draw[bigarrow] (12,0) -- (12.5,0);
  		
  		\begin{scope}[shift={(13,0)}]
  			\node[zpt] at (0.5, 1.55) {};
  			\node[zptt] at (0.5, 1.55) {};
  			\node[tinylabel, right=6pt] at (0.5, 1.55) {$Z$};
  			
  			\node[xpt] at (0.5, 1.05) {};
  			\node[tinylabel, right=6pt] at (0.5, 1.05) {$X$};
  			
  			\node[ypt] at (0.5,0.55) {};
  			\node[tinylabel, right=6pt] at (0.5,0.55) {$Y$};
  			
  			\fill (0,0) ellipse (0.22 and 0.14);
  			\fill (0.5,0) ellipse (0.22 and 0.14);
  			\fill (1,0) ellipse (0.22 and 0.14);
  		\end{scope}
  		
  	\end{tikzpicture}
  	\caption{Schematic of the SRRM iterative procedure.}
  	\label{fig:example1}
  \end{figure}

  This point-selection method constructs a strictly bijective matching between two point sets $X$ and $Y$ of equal size. It maintains a global plan $\pi$, together with two index sets $\mathcal{C}_X$ and $\mathcal{C}_Y$ that represent the currently unresolved points on each side. Specifically, $\mathcal{C}_X$ is the set of indices of points in $X$ whose matches have not yet been reliably determined, while $\mathcal{C}_Y$ contains the indices of points in $Y$ that are not yet firmly assigned (i.e., still available). Initially, $\pi(i)=-1$ for all $i$ and $\mathcal{C}_X=\mathcal{C}_Y=\{1,\dots,n\}$.
  
  The key idea is to introduce an anchor set $Z$ to absorb ambiguous matches, so that reliable pairs can be fixed early and the unresolved sets shrink across rounds. At each round, the method operates only on the current subsets $X[\mathcal{C}_X]$ and $Y[\mathcal{C}_Y]$. If a point $x\in X[\mathcal{C}_X]$ is matched to a point $y\in Y[\mathcal{C}_Y]$ (a real-to-real correspondence), this pair is regarded as reliable, written into the global plan $\pi$, and removed from further consideration. Otherwise, if $x$ is matched to an anchor in $Z$, it is treated as a hard point and its index is retained in the updated $\mathcal{C}_X$ for the next round. Symmetrically, any $y\in Y[\mathcal{C}_Y]$ that is ``captured'' by anchors in the augmented matching is retained in the updated $\mathcal{C}_Y$. As the iterations proceed, more reliable pairs are fixed in $\pi$, while $\mathcal{C}_X$ and $\mathcal{C}_Y$ typically shrink to a small set of difficult points.
  
  After the iterative screening, the plan $\pi$ may still contain a few unassigned entries or conflicts. To enforce a strict bijection, the remaining unmatched points in $X$ are paired with the unused points in $Y$ by solving a minimum-cost assignment problem (e.g., with cost $\|x-y\|_2^2$) using the Hungarian algorithm\cite{kuhn1955hungarian}, which completes $\pi$ into a valid permutation.
  
  The quality of the matching $T$ obtained at each iteration directly determines how many anchors will be introduced in the subsequent round. Therefore, to improve the overall performance of the SRRM, it is crucial to enhance the quality of $T$. To this end, we adopt the plan-merging strategy proposed in \cite{ref12} to refine $T$ before performing point selection. Fig.~\ref{fig:ten-two-rows} illustrates how the quality of $T$ influences the behavior and effectiveness of the Point Selection Method.
  
  The SRRM procedure is summarized in Algorithm.~\ref{alg:psm}, with a schematic illustration shown in Fig.~\ref{fig:example1}.
  
  \begin{figure*}[htbp] 
  	\centering
  	
  	\begin{subfigure}[t]{0.16\textwidth}
  		\centering
  		\includegraphics[width=\linewidth]{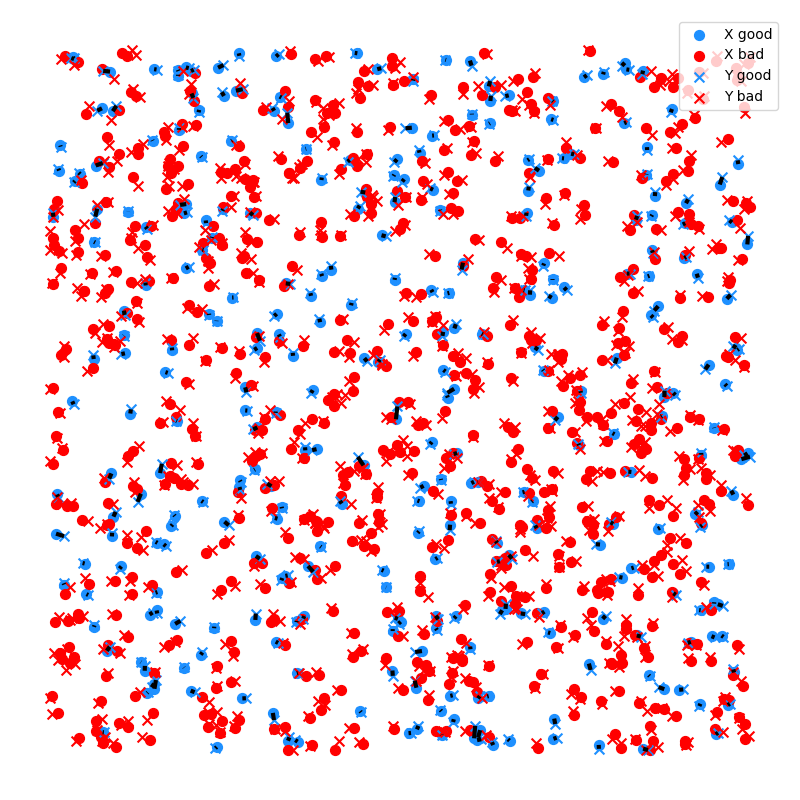}
  		
  	\end{subfigure}\hfill
  	\begin{subfigure}[t]{0.16\textwidth}
  		\centering
  		\includegraphics[width=\linewidth]{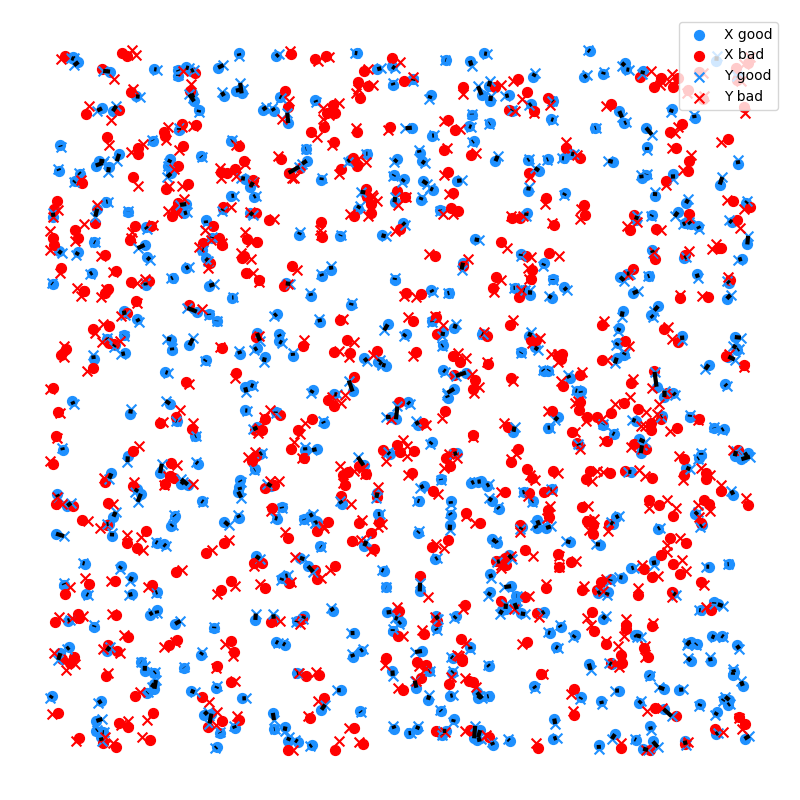}
  		
  	\end{subfigure}\hfill
  	\begin{subfigure}[t]{0.16\textwidth}
  		\centering
  		\includegraphics[width=\linewidth]{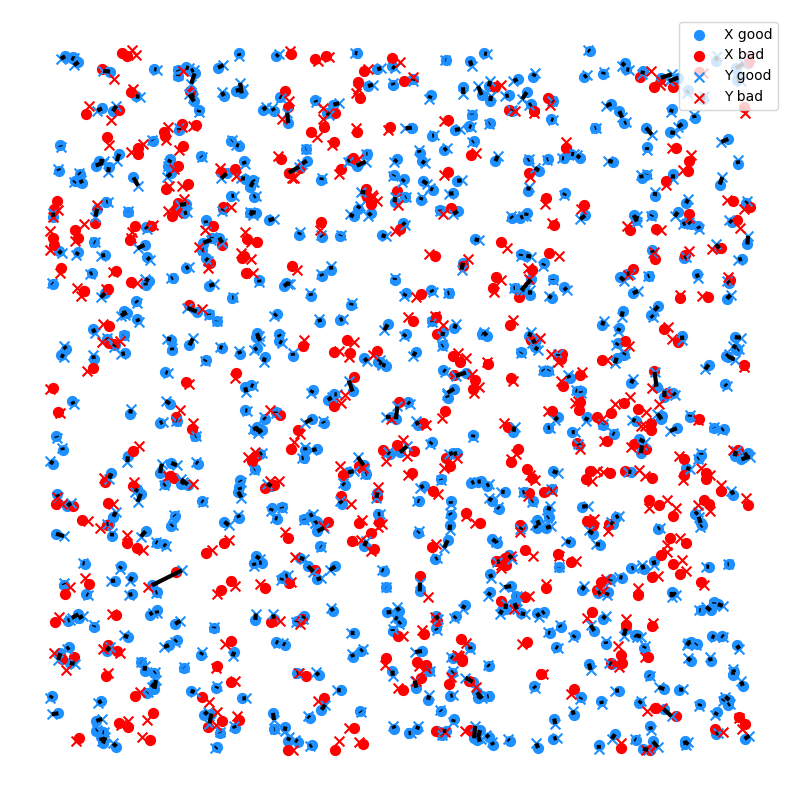}
  		
  	\end{subfigure}\hfill
  	\begin{subfigure}[t]{0.16\textwidth}
  		\centering
  		\includegraphics[width=\linewidth]{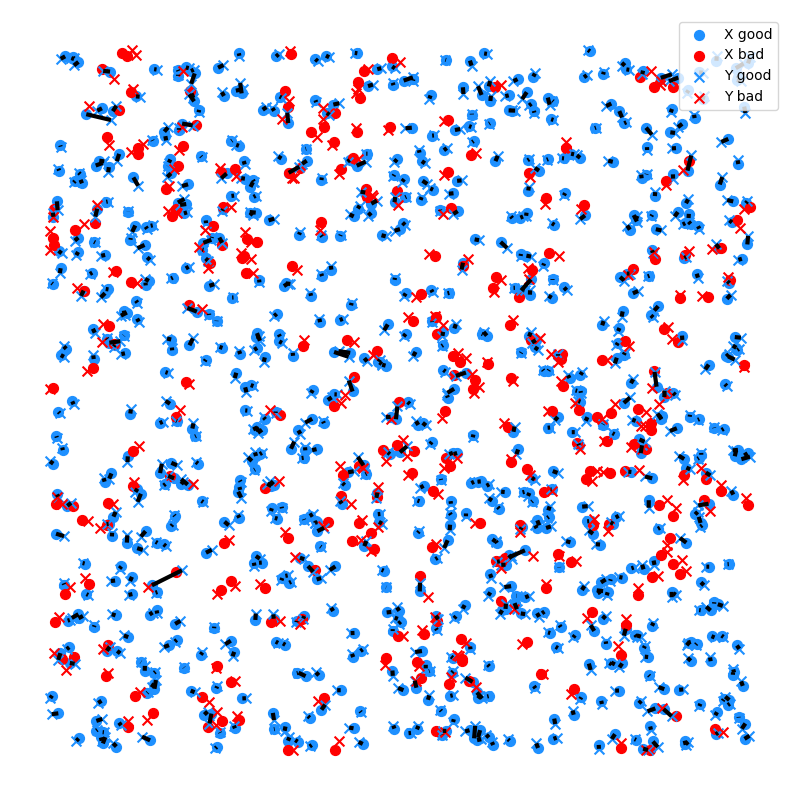}
  		
  	\end{subfigure}\hfill
  	\begin{subfigure}[t]{0.16\textwidth}
  		\centering
  		\includegraphics[width=\linewidth]{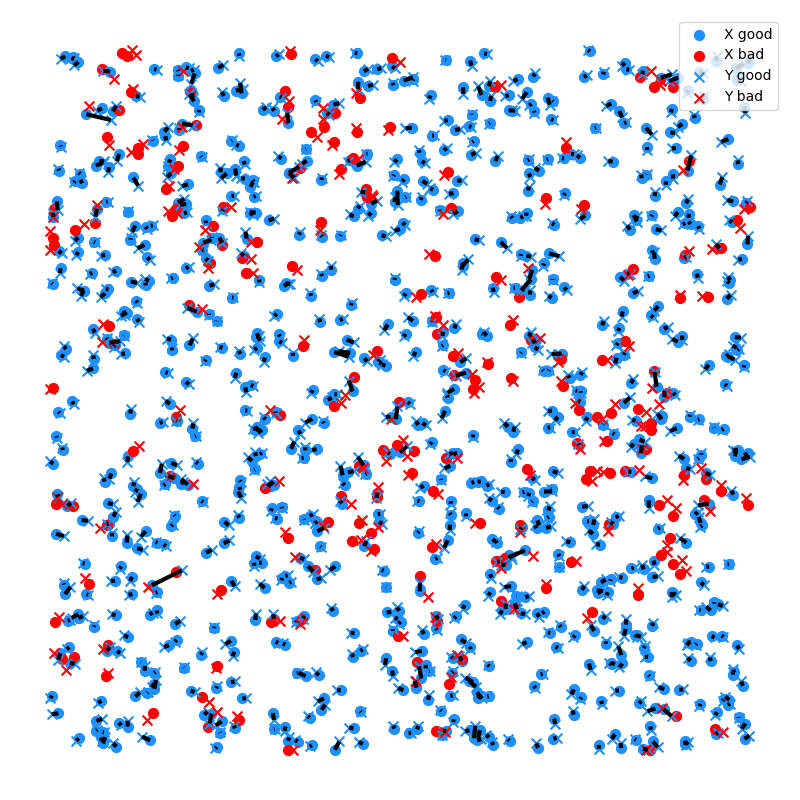}
  		
  	\end{subfigure}
  	\begin{subfigure}[t]{0.16\textwidth}
  		\centering
  		\includegraphics[width=\linewidth]{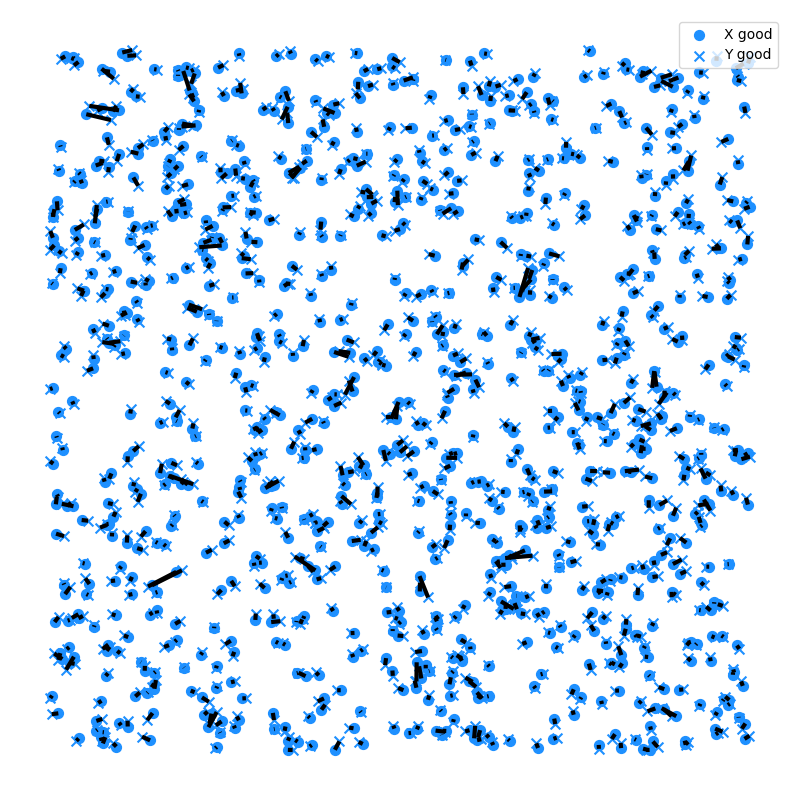}
  		
  	\end{subfigure}
  	
  	\par\smallskip 
  	
  	\begin{subfigure}[t]{0.16\textwidth}
  		\centering
  		\includegraphics[width=\linewidth]{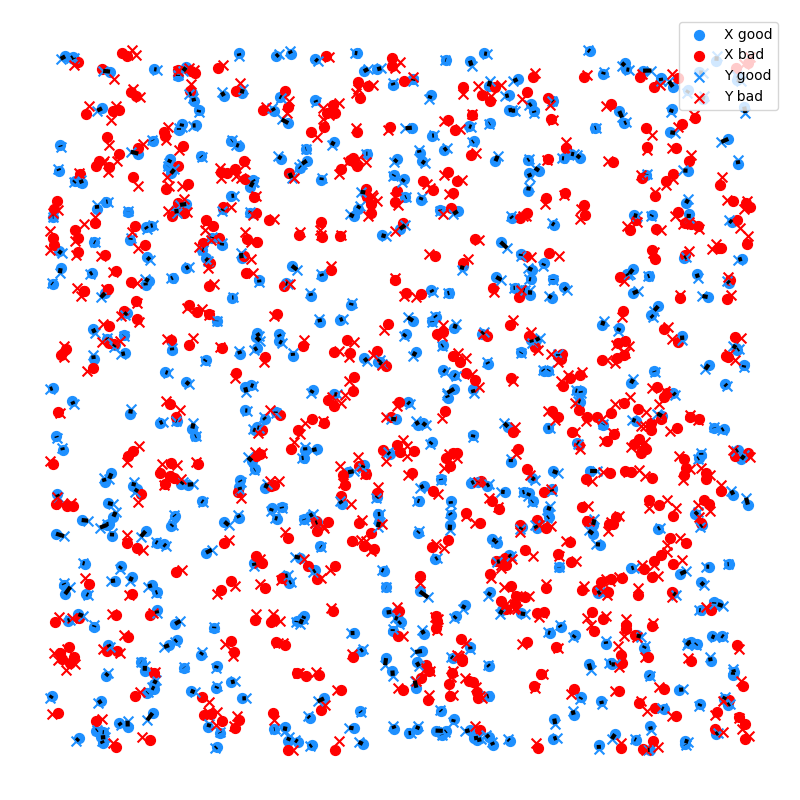}
  		
  	\end{subfigure}\hfill
  	\begin{subfigure}[t]{0.16\textwidth}
  		\centering
  		\includegraphics[width=\linewidth]{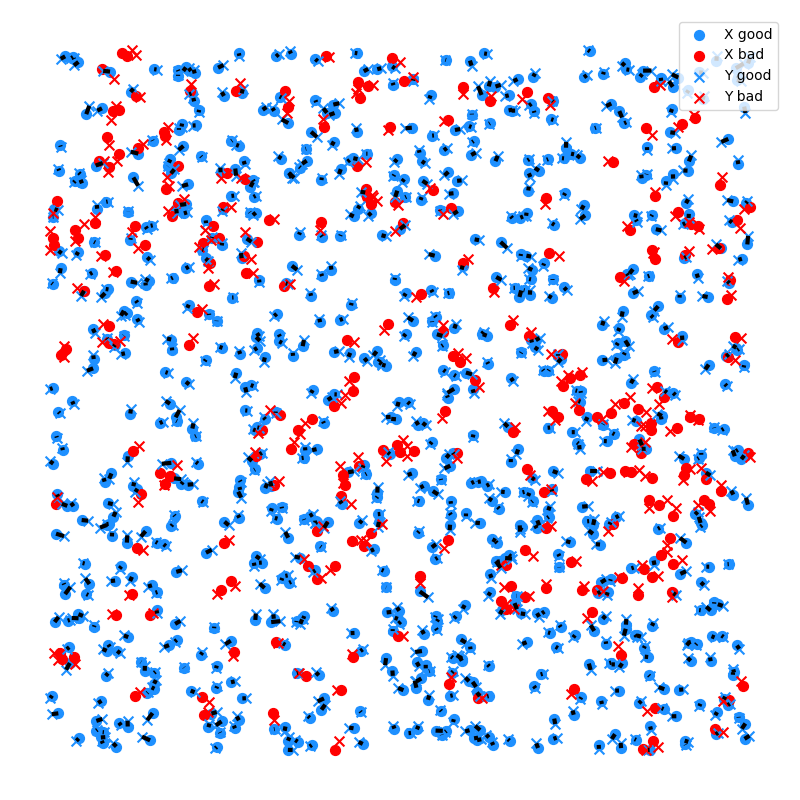}
  		
  	\end{subfigure}\hfill
  	\begin{subfigure}[t]{0.16\textwidth}
  		\centering
  		\includegraphics[width=\linewidth]{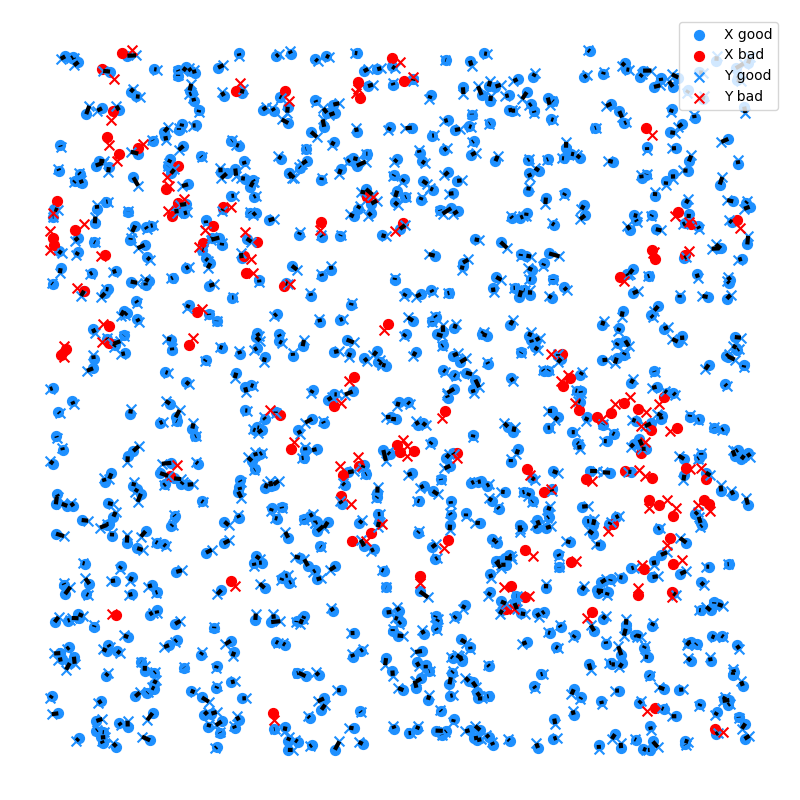}
  		
  	\end{subfigure}\hfill
  	\begin{subfigure}[t]{0.16\textwidth}
  		\centering
  		\includegraphics[width=\linewidth]{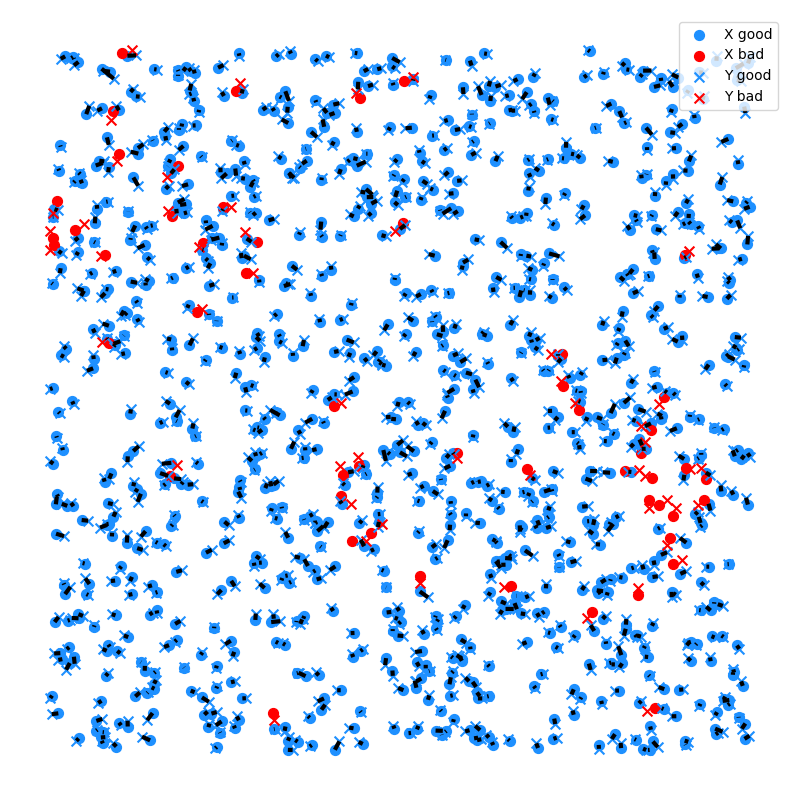}
  		
  	\end{subfigure}\hfill
  	\begin{subfigure}[t]{0.16\textwidth}
  		\centering
  		\includegraphics[width=\linewidth]{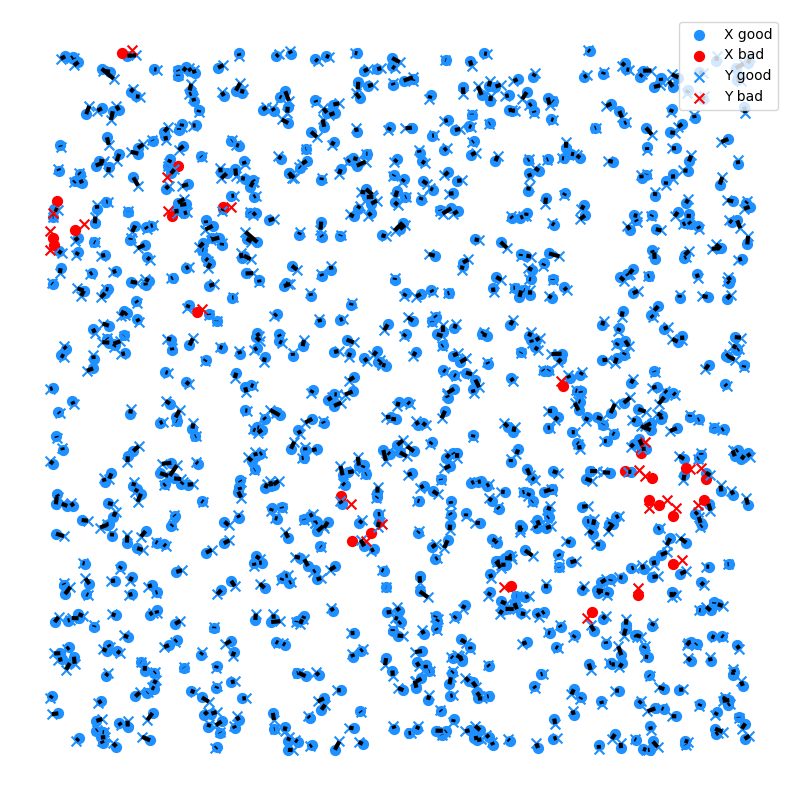}
  		
  	\end{subfigure}
  	\begin{subfigure}[t]{0.16\textwidth}
  		\centering
  		\includegraphics[width=\linewidth]{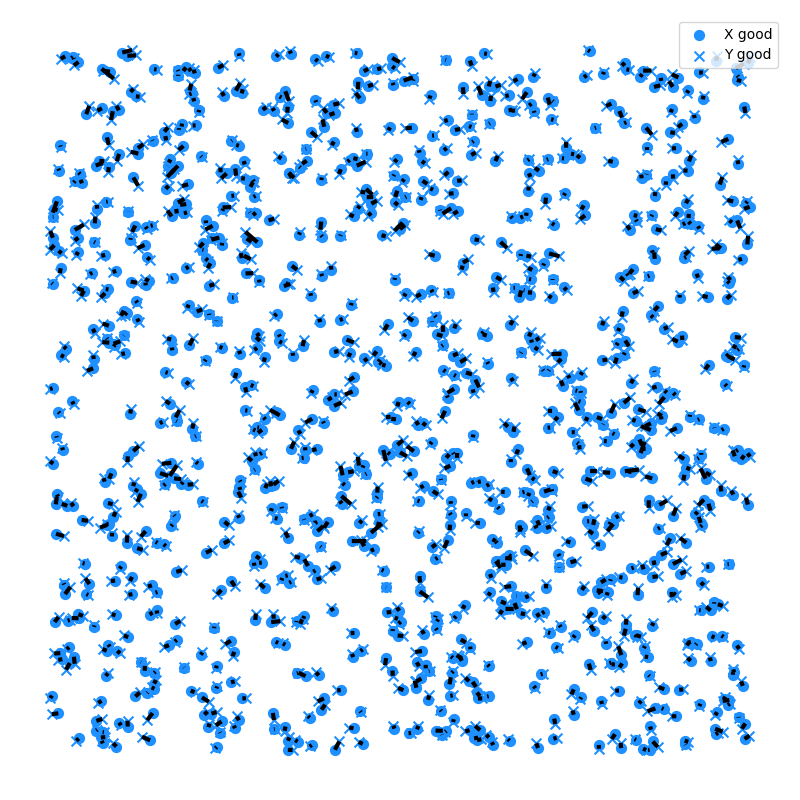}
  		
  	\end{subfigure}

  	\caption{Top row: results for a single coupling.
  		Bottom row: results after merging ten couplings.}
  	\label{fig:ten-two-rows}
  \end{figure*}

  \begin{algorithm}[t]
  	\caption{Selective Recursive Rank Matching(SRRM)}
  	\label{alg:psm}
  	\begin{algorithmic}[1]
  		\Require $X,Y\in\mathbb{R}^{n\times d}$, rounds $R$, anchors per point $k$,merge $K$ 
  		\Ensure bijection $\pi:\{1,\dots,n\}\to\{1,\dots,n\}$
  		
  		\State $\pi \gets -\mathbf{1}$,\quad $\mathcal{C}_X\gets\{1,\dots,n\}$,\quad $\mathcal{C}_Y\gets\{1,\dots,n\}$
  		\For{$r=1$ to $R$}
  		\State $m\gets|\mathcal{C}_X|$; \If{$m=0$} \State \textbf{break} \EndIf
  		\State $X_s\gets X[\mathcal{C}_X]$, $Y_s\gets Y[\mathcal{C}_Y]$
  		\State $Z \gets \textsc{SampleNear}(X_s,k)\ \cup\ \textsc{SampleNear}(Y_s,k)$
  		\State $T \gets \textsc{MergedRRM}([X_s;Z],[Y_s;Z],K)$ \Comment{multi-run RRM + plan merging \cite{ref12}}
  		\State $(\text{good},\ \mathcal{C}_X,\ \mathcal{C}_Y)\gets \textsc{Select}(T,m)$ 
  		\State Write $\pi$ using \text{good} pairs
  		\EndFor
  		\State $\pi \gets \textsc{HungarianFinalize}(X,Y,\pi)$ \Comment{fill remaining to a permutation}
  		\State \Return $\pi$
  	\end{algorithmic}
  \end{algorithm}

  \noindent\textbf{Complexity.}
  The overall complexity of our method is $O\!\left(R\,K\,n\log n\right) + O(u^3).$ Here, $n$ is the number of points in each measure, $R$ is the number of screening rounds in the point-selection procedure, and $K$ denotes the number of candidate plans merged at each round to produce an improved matching plan. Finally, $u$ denotes the number of remaining unmatched points after the iterative screening, and $O(u^3)$ corresponds to the Hungarian completion used to enforce a strict bijection.

\section{Simulations}
\label{sec:Simulation}

In this section, we present numerical simulations to illustrate and validate the theoretical results and complexity analysis developed in Section~\ref{sec:The last-mile phenomenon and Numerical Implementation of SRRM}. Section~\ref{sec:Empirical validation of the last-mile plateau mechanism}  presents a constructed example that reveals how the two identified factors jointly affect the plateau level. Section~\ref{sec:Factors Affecting Algorithmic Efficiency}  uses a concrete example to examine the factors influencing the computational complexity of the SRRM algorithm.

\subsection{Empirical validation of the last-mile plateau mechanism}
\label{sec:Empirical validation of the last-mile plateau mechanism}
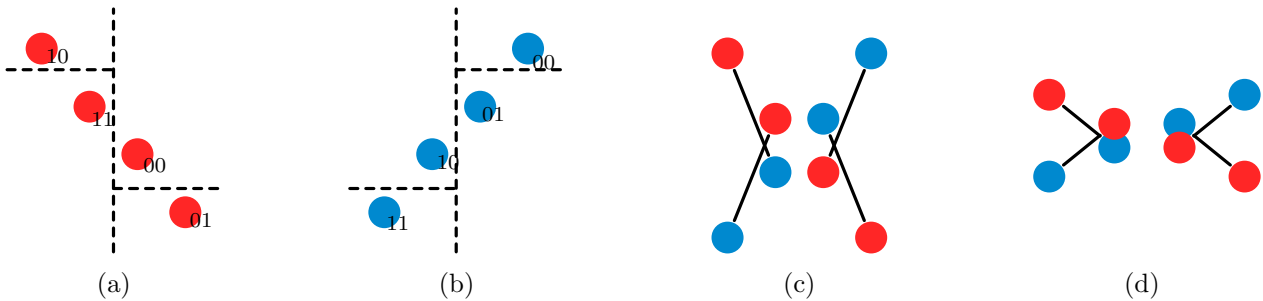
\begin{figure}[htbp]
	\centering
	
	\begin{subfigure}{0.24\linewidth}
		\centering
		\begin{tikzpicture}[scale=0.7,
			redc/.style ={circle, fill=red!85,  line width=1.8pt, minimum size=12pt, inner sep=0pt},
			bluec/.style={circle, fill=cyan!70!blue, line width=1.8pt, minimum size=12pt, inner sep=0pt},
			wall/.style={draw=black, line width=1.2pt, line cap=round, dashed},
			stub/.style={draw=black, line width=1.2pt, line cap=round, dashed},
			code/.style={font=\scriptsize}
			]
			\draw[wall] (0,-2.3) -- (0, 2.3);
			\draw[stub] (-2.0, 1.15) -- (0, 1.15);
			\draw[stub] (0,-1.10) -- (2.0,-1.10);
			
			\node[redc] at (-1.35,  1.55) {};
			\node[redc] at (-0.45,  0.45) {};
			\node[redc] at ( 0.45, -0.45) {};
			\node[redc] at ( 1.35, -1.55) {};
			
			\node[code] at (-1.05, 1.40) {10};
			\node[code] at (-0.20, 0.22) {11};
			\node[code] at ( 0.78,-0.65) {00};
			\node[code] at ( 1.65,-1.70) {01};
		\end{tikzpicture}
		\caption{}
	\end{subfigure}\hfill
	\begin{subfigure}{0.24\linewidth}
		\centering
		\begin{tikzpicture}[scale=0.7,
			redc/.style ={circle, fill=red!85,  line width=1.8pt, minimum size=12pt, inner sep=0pt},
			bluec/.style={circle, fill=cyan!70!blue, line width=1.8pt, minimum size=12pt, inner sep=0pt},
			wall/.style={draw=black, line width=1.2pt, line cap=round, dashed},
			stub/.style={draw=black, line width=1.2pt, line cap=round, dashed},
			code/.style={font=\scriptsize}
			]
			\draw[wall] (0,-2.3) -- (0, 2.3);
			\draw[stub] (0, 1.15) -- (2.0, 1.15);
			\draw[stub] (-2.0,-1.10) -- (0,-1.10);
			
			\node[bluec] at ( 1.35,  1.55) {};
			\node[bluec] at ( 0.45,  0.45) {};
			\node[bluec] at (-0.45, -0.45) {};
			\node[bluec] at (-1.35, -1.55) {};
			
			\node[code] at ( 1.65, 1.30) {00};
			\node[code] at ( 0.70, 0.30) {01};
			\node[code] at (-0.15,-0.60) {10};
			\node[code] at (-1.10,-1.75) {11};
		\end{tikzpicture}
		\caption{}
	\end{subfigure}\hfill
	\begin{subfigure}{0.24\linewidth}
		\centering
		\begin{tikzpicture}[scale=0.7,
			yscale=0.75,
			redc/.style ={circle, fill=red!85,  line width=1.8pt, minimum size=12pt, inner sep=0pt},
			bluec/.style={circle, fill=cyan!70!blue, line width=1.8pt, minimum size=12pt, inner sep=0pt},
			match/.style={line width=1.2pt, draw=black, line cap=round}
			]
			\node[bluec] (bq) at ( 1.35,  2.325) {};
			\node[bluec] (bw) at ( 0.45,  0.675) {};
			\node[bluec] (be) at (-0.45, -0.675) {};
			\node[bluec] (br) at (-1.35, -2.325) {};
			
			\node[redc] (rq) at (-1.35,  2.325) {};
			\node[redc] (rw) at (-0.45,  0.675) {};
			\node[redc] (re) at ( 0.45, -0.675) {};
			\node[redc] (rr) at ( 1.35, -2.325) {};
			
			\draw[match] (bq) -- (re);
			\draw[match] (bw) -- (rr);
			\draw[match] (be) -- (rq);
			\draw[match] (br) -- (rw);
		\end{tikzpicture}
		\caption{}
	\end{subfigure}\hfill
	\begin{subfigure}{0.24\linewidth}
		\centering
		\raisebox{23pt}{
			\begin{tikzpicture}[scale=0.7,
				xscale=0.85,
				redc/.style ={circle, fill=red!85,  line width=1.8pt, minimum size=12pt, inner sep=0pt},
				bluec/.style={circle, fill=cyan!70!blue, line width=1.8pt, minimum size=12pt, inner sep=0pt},
				match/.style={line width=1.2pt, draw=black, line cap=round}
				]
				\node[bluec] (bq) at ( 2.16,  0.775) {};
				\node[bluec] (bw) at ( 0.72,  0.225) {};
				\node[bluec] (be) at (-0.72, -0.225) {};
				\node[bluec] (br) at (-2.16, -0.775) {};
				
				\node[redc] (rq) at (-2.16,  0.775) {};
				\node[redc] (rw) at (-0.72,  0.225) {};
				\node[redc] (re) at ( 0.72, -0.225) {};
				\node[redc] (rr) at ( 2.16, -0.775) {};
				
				\draw[match] (bq) -- (re);
				\draw[match] (bw) -- (rr);
				\draw[match] (be) -- (rq);
				\draw[match] (br) -- (rw);
		\end{tikzpicture}}
		\caption{}
	\end{subfigure}
	
	\caption{When the absolute slope $|m|$ of the red and blue lines tends to zero or infinity, 
		almost all sampled points become bad indices. However, the severity of the resulting mismatches differs between the two regimes.}
	\label{fig:two-panels-codes}
\end{figure}

We construct synthetic distributions supported on $[0,1]^2$ by sampling points on line segments through $(0.5,0.5)$ as shown in Fig.~\ref{fig:two-panels-codes} to empirically validate the last-mile plateau mechanism
revealed by Theorem~\ref{thm:nn_plateau}: the limiting plateau height of recursive partitioning methods is jointly governed by (i) the fraction of prematurely separated indices
$\alpha_H := |I_+|/n$, and (ii) the average mismatch severity among them, quantified by the
NN-excess $\bar{\Gamma}_H^{\mathrm{NN}}$.

Throughout, we use the axis-recursive mass-median realization of HCP. In these experiments, we adopt a simplified BSP-style baseline based on a single split along an arbitrary direction and quantile, as a coarse proxy for general recursive partitioning effects. For SRRM, we use 10 recursive partition steps with merging, followed by 20 refinement iterations with \(k=5\), as a simplified but representative setting.

\begin{figure}[htbp]
	\centering
	\begin{subfigure}{0.25\textwidth}
		\centering
		\includegraphics[width=\linewidth]{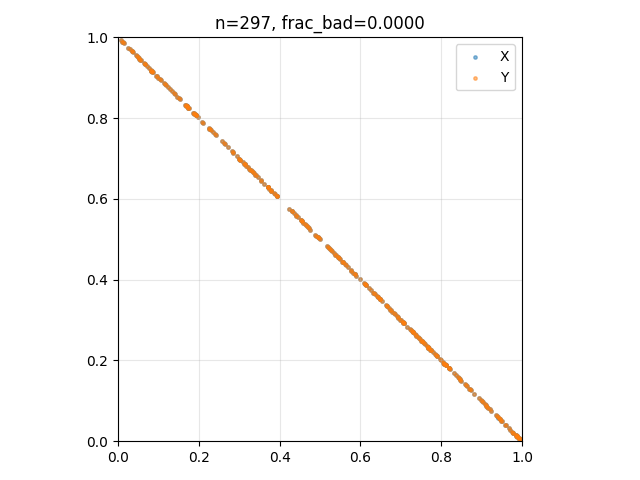}
		
	\end{subfigure}
	\begin{subfigure}{0.25\textwidth}
		\centering
		\includegraphics[width=\linewidth]{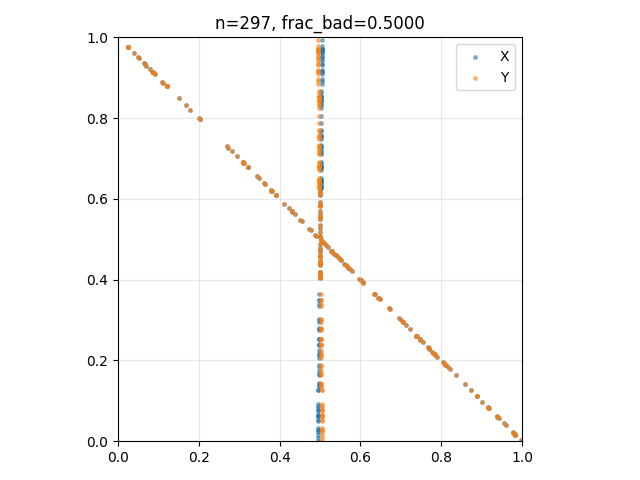}
		
	\end{subfigure}
	\begin{subfigure}{0.25\textwidth}
		\centering
		\includegraphics[width=\linewidth]{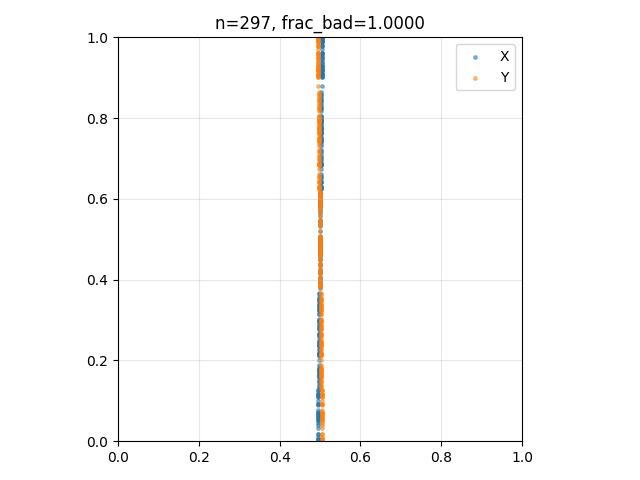}
		
	\end{subfigure}
	\begin{subfigure}{0.25\textwidth}
		\centering
		\includegraphics[width=\linewidth]{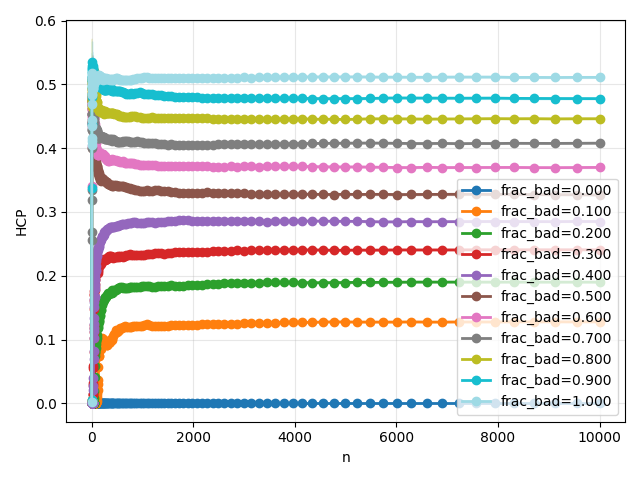}
		
	\end{subfigure}
	\begin{subfigure}{0.25\textwidth}
		\centering
		\includegraphics[width=\linewidth]{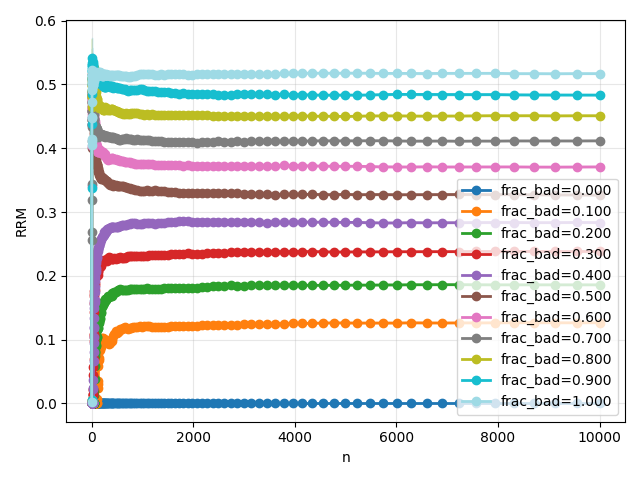}
		
	\end{subfigure}
	
	\caption{Experiment 1  (Effect of bad-index proportion on the bias floor).
	}
	\label{fig:five-in-a-row11}
\end{figure}

\paragraph{\textbf{Experiment 1}}
The good component consists of two lines with identical slope $m=-1$.
Under address-restricted RRM matching, samples from this component exhibit highly consistent tree addresses,
and are therefore aligned without premature separation.
Empirically, this part contributes essentially zero RRM distance and can be regarded as generating good indices.

The bad component consists of two lines whose slopes are opposite in sign,
with large magnitude $|m|=100$.
Although geometrically close cross-distribution neighbors exist,
the recursive partition often separates them at shallow depths.
Since address restriction forbids cross-leaf matching,
these near neighbors cannot be aligned, and samples are forced to match with farther within-leaf points,
producing substantial mismatch penalties.
Empirically, most samples from this component correspond to bad indices.

We control the mixture weight of the bad component $\texttt{frac\_bads}$, defined as the proportion of samples drawn from
the opposite-slope ($|m|=100$) line pair.
Thus, larger $\texttt{frac\_bads}$ corresponds to a higher fraction of bad indices.
The behavior shown in Fig.~\ref{fig:five-in-a-row11} directly corroborates Theorem~\ref{thm:nn_plateau}: as the fraction of prematurely separated indices $\alpha_H$ increases, the limiting plateau height of the RRM distance increases
accordingly.

\begin{figure}[htbp]
	\centering
	\begin{subfigure}{0.25\textwidth}
		\centering
		\includegraphics[width=\linewidth]{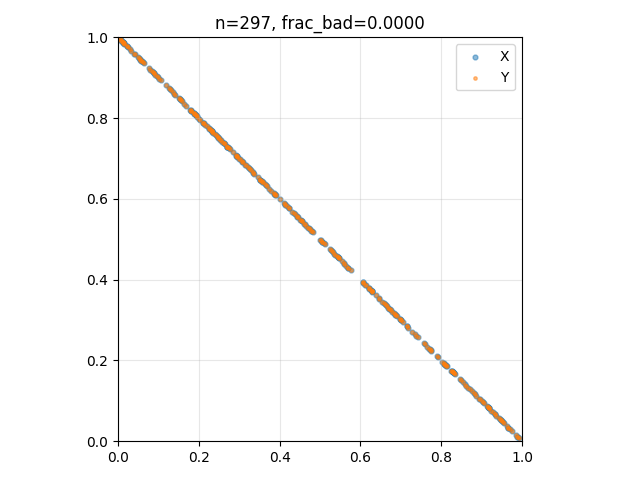}
		
	\end{subfigure}
	\begin{subfigure}{0.25\textwidth}
		\centering
		\includegraphics[width=\linewidth]{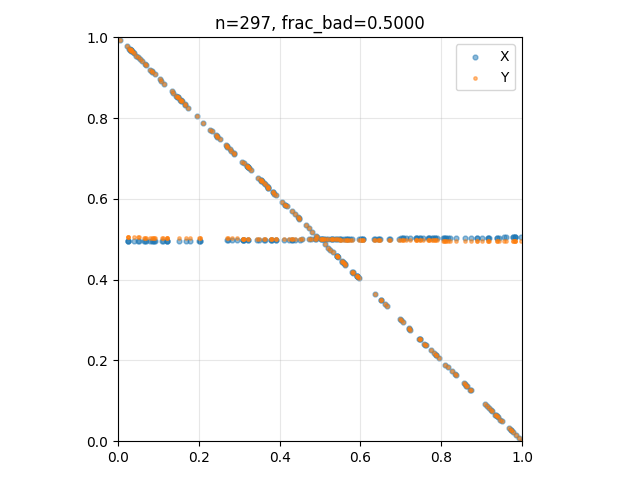}
		
	\end{subfigure}
	\begin{subfigure}{0.25\textwidth}
		\centering
		\includegraphics[width=\linewidth]{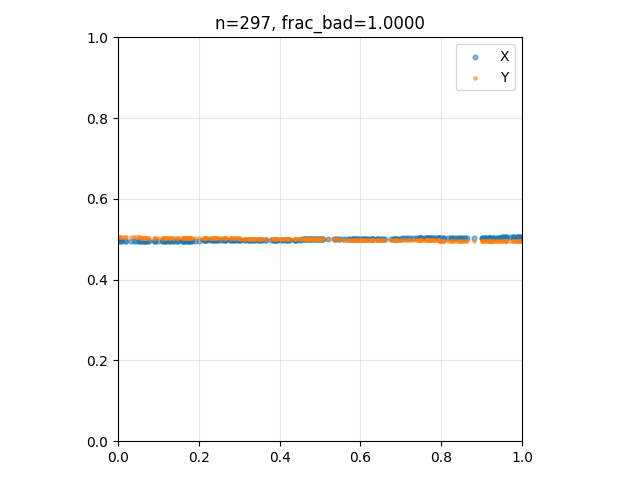}
		
	\end{subfigure}
	\begin{subfigure}{0.25\textwidth}
		\centering
		\includegraphics[width=\linewidth]{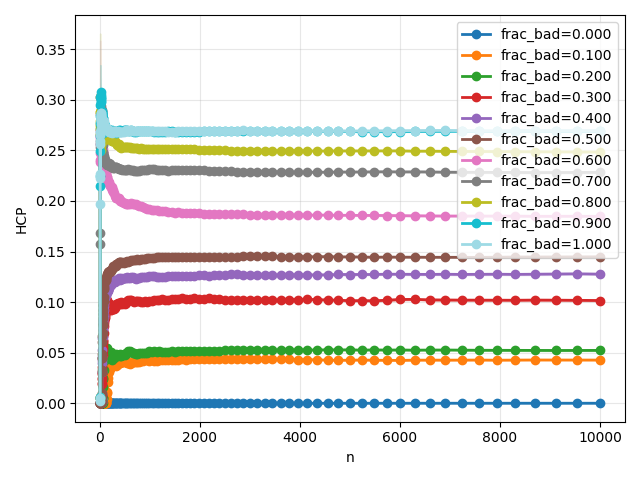}
		
	\end{subfigure}
	\begin{subfigure}{0.25\textwidth}
		\centering
		\includegraphics[width=\linewidth]{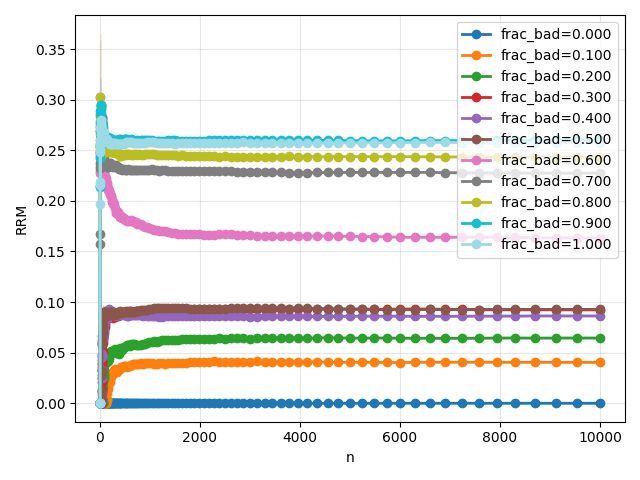}
		
	\end{subfigure}
	\caption{Experiment 2 (Reducing the severity of premature separation).
	}
	\label{fig:five-in-a-row111}
\end{figure}

\paragraph{\textbf{Experiment 2}}
This experiment uses the same mixture construction as Experiment~1 and still treats the two lines with identical slope $m=-1$ as the good component.
The key change is in the bad component: instead of using opposite slopes with large magnitude, we set the opposite-slope pair to have a much smaller magnitude,
$|m|=0.01$.
Relative to the large-slope setting, this modification does not alter the
qualitative mechanism, but it substantially reduces the
severity of premature separation, i.e., it yields smaller $\bar{\Gamma}_H^{\mathrm{NN}}$ for the bad indices.
Consequently, as shown in Fig.~\ref{fig:five-in-a-row111}, while increasing $\texttt{frac\_bads}$ still increases the fraction of bad indices, the resulting limiting plateau is markedly lower than in the large-slope case, consistent with Theorem~\ref{thm:nn_plateau}.

\begin{figure}[htbp]
	\centering
	\begin{subfigure}{0.24\textwidth}
		\centering
		\includegraphics[width=\linewidth]{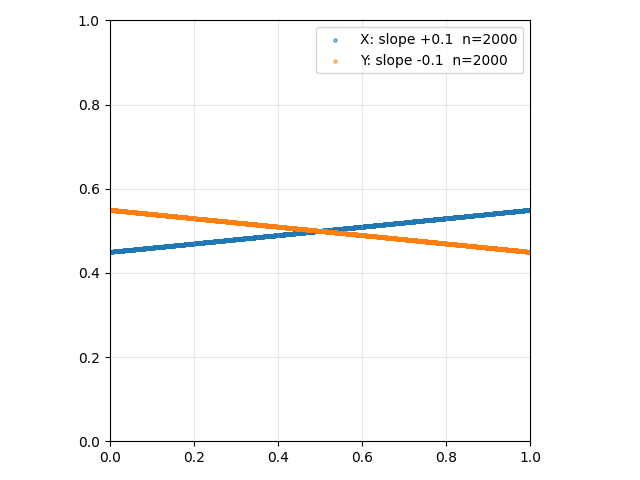}
		
	\end{subfigure}
	\begin{subfigure}{0.24\textwidth}
		\centering
		\includegraphics[width=\linewidth]{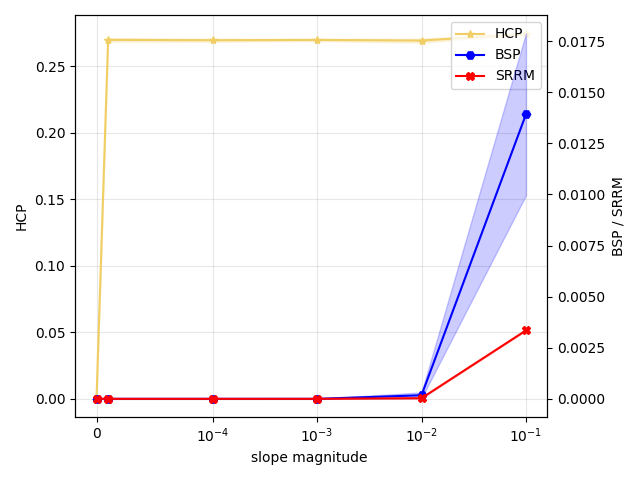}
		
	\end{subfigure}
	\begin{subfigure}{0.24\textwidth}
		\centering
		\includegraphics[width=\linewidth]{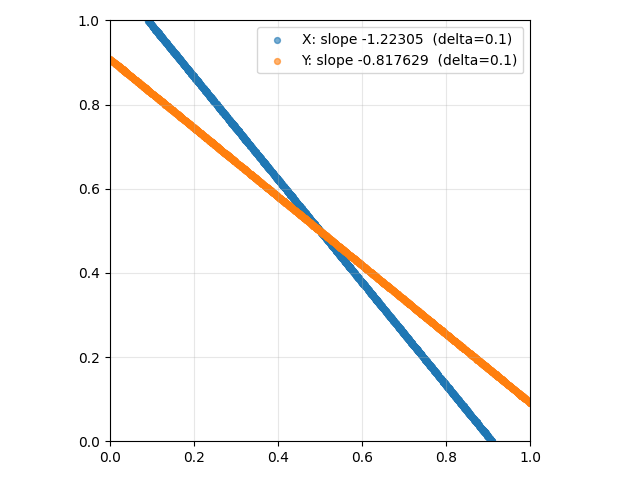}
		
	\end{subfigure}
	\begin{subfigure}{0.24\textwidth}
		\centering
		\includegraphics[width=\linewidth]{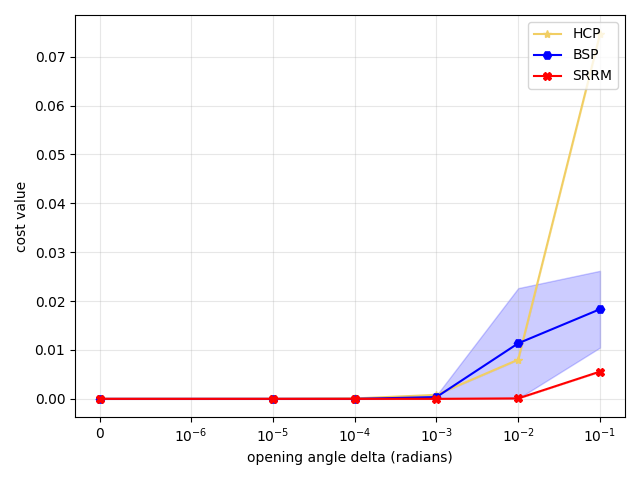}
		
	\end{subfigure}
	
	\caption{Experiment 3 (Separating ``fraction'' from ``severity'': two complementary regimes).
	}
	\label{fig:five-in-a-row1111}
\end{figure}

\paragraph{\textbf{Experiment 3}}(Fig.~\ref{fig:five-in-a-row1111})
Experiment~3 consists of two parts.

\emph{Part I: same-sign perturbations around \(m=0\).}
We consider two line-supported distributions whose slopes have the same sign and differ by a small amount.
We vary the slope magnitude over
$
|m|\in\{0,\,10^{-5},\,10^{-4},\,10^{-3},\,10^{-2},\,10^{-1}\}.
$
When $m=0$ the two supporting lines coincide, and no premature separation is observed.
However, as soon as the lines are not exactly identical (even at $|m|=10^{-5}$ or $10^{-4}$),
prematurely separated indices $I_+$ appear due to the address restriction, and the empirical RRM distance
exhibits a clear nonzero bias floor.
In particular, despite the fact that the geometric discrepancy is extremely small when $|m|$ increases from
$0$ to $10^{-4}$, the emergence of bad indices $I_+$ alone already induces a pronounced plateau.

\emph{Part II:opening-angle perturbations around \(m=-1\).}
We next start from two coincident lines with slope $-1$ and ``open'' them outward by a small angle.

Let the opening angle increment be $\delta\in\{0,\,10^{-5},\,10^{-4},\,10^{-3},\,10^{-2},\,10^{-1}\}.$ As $\delta$ increases, the geometric separation between the two supports increases, and the measured distance increases monotonically as well.
While premature separation may still contribute to the overall behavior, the dominant trend in this regime is
the expected growth of the discrepancy with the opening angle.

Together, the two parts highlight that (i) a noticeable bias floor can emerge as soon as the supports are not exactly coincident, as bad indices begin to appear, and (ii) beyond this effect, the distance still increases with the true geometric separation $\frac1n\sum_{i=1}^n \delta_i^2$.

\subsection{Factors affecting algorithmic efficiency}
\label{sec:Factors Affecting Algorithmic Efficiency}
We generate two independent point clouds 
$X=\{x_i\}_{i=1}^n$ and $Y=\{y_j\}_{j=1}^n$
on the normalized domain $[0,1]^2$. 
The samples are drawn from two truncated isotropic Gaussian distributions:
\[ 
x_i \sim \mathrm{clip}\big(\mathcal N(m_1(t),\,\sigma^2 I_2),\, [0,1]^2\big), 
\qquad
y_j \sim \mathrm{clip}\big(\mathcal N(m_2(t),\,\sigma^2 I_2),\, [0,1]^2\big).
\]

The distribution centers (means) move linearly toward the midpoint $(0.5,0.5)$:
$
m_1(t) = (1-t)(0.8,0.8) + t(0.5,0.5), 
m_2(t) = (1-t)(0.1,0.1) + t(0.5,0.5),
q t \in [0,1].
$

We measure the mean separation
$
\mathrm{sep}(t) = \|m_1(t) - m_2(t)\|_2
$
to characterize the geometric distance between the two distributions.
We set the SRRM parameters to $R=10$, $K=5$, and $k=1$. For each value of $t$, we independently repeat the sampling and matching procedure 10 times, recording at every iteration of the point-screening step the number of indices identified as bad and hence belonging to $I_+$ (see Table~\ref{tab:Iplus_by_sep}), as well as the running time $T(t)$. We report the mean and standard deviation (mean $\pm$ std) of the running time in Fig.~\ref{fig:five}.

The results show that as the two point clouds gradually move closer, the number of remaining unmatched points after iterative screening, denoted by $u$, decreases steadily; in particular, when the two distributions share the same centroid, $u$ attains its minimum. Since the Hungarian completion has computational complexity $O(u^3)$, this indicates that the completion stage is effectively reduced in size, thereby significantly controlling the overall runtime.

\begin{figure}[t]
	\centering
	\begin{subfigure}{0.24\textwidth}
		\centering
		\includegraphics[width=\linewidth]{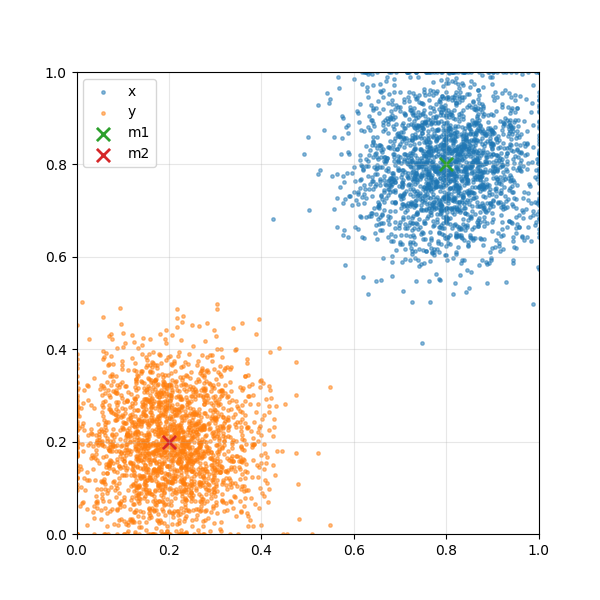}
		\caption{}
	\end{subfigure}
	\begin{subfigure}{0.24\textwidth}
		\centering
		\includegraphics[width=\linewidth]{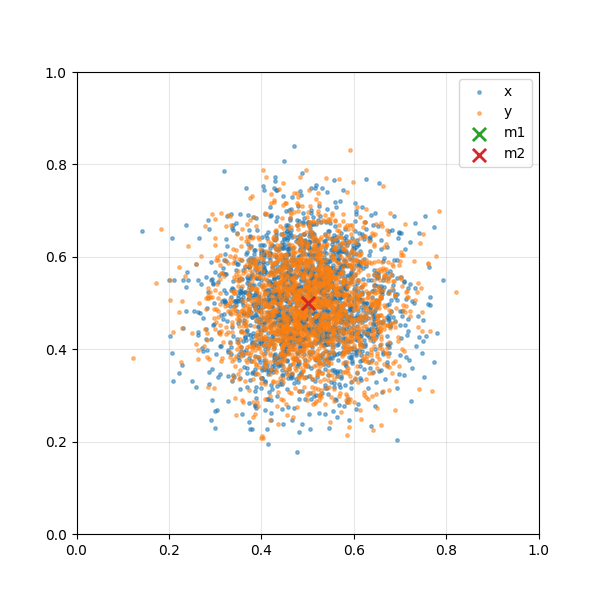}
		\caption{}
	\end{subfigure}
	\begin{subfigure}{0.24\textwidth}
		\centering
		\includegraphics[width=\linewidth]{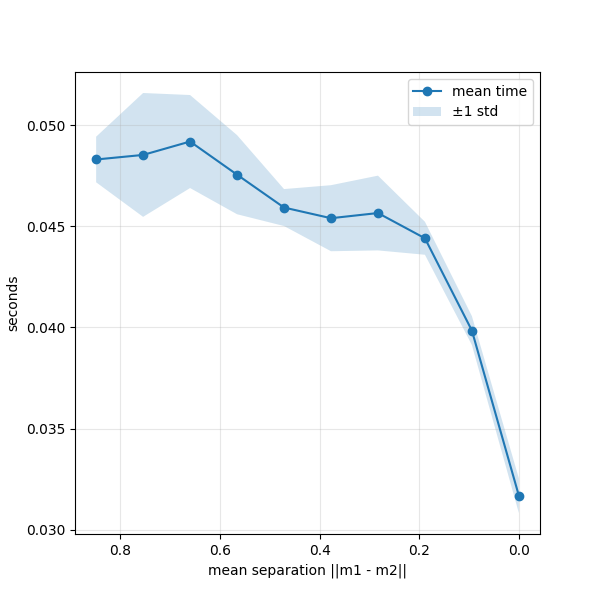}
		\caption{}
	\end{subfigure}
	\begin{subfigure}{0.24\textwidth}
		\centering
		\includegraphics[width=\linewidth]{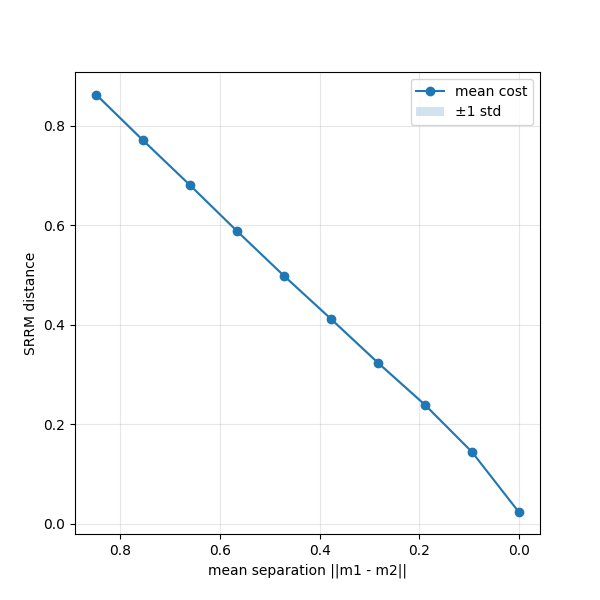}
		\caption{}
	\end{subfigure}
	
	\caption{(a)--(b) The two distributions, each with 2000 sampled points,
		whose centers gradually move toward each other.
		(c) The running time of SRRM decreases as the distribution centers approach.
		(d) The SRRM decreases accordingly as the separation diminishes.
	}
	\label{fig:five}
\end{figure}

\begin{table}[htbp]
	\centering
	\caption{Statistics of $|I_+|$ over iterations.}
	\label{tab:Iplus_by_sep}
	\setlength{\tabcolsep}{3.2pt}
	\renewcommand{\arraystretch}{1.35}
	\resizebox{\textwidth}{!}{%
		\begin{tabular}{ccc*{11}{c}}
			\toprule
			$t$ & $\|m_1-m_2\|$ & Metric
			& iter00 & iter01 & iter02 & iter03 & iter04 & iter05 & iter06 & iter07 & iter08 & iter09 & iter10 \\
			\midrule
			0.000 & 0.8485 & $|I_+|$
			& $1498.50\pm12.28$ & $1123.80\pm8.15$  & $847.80\pm9.02$  & $637.60\pm13.53$ & $477.10\pm10.52$
			& $358.40\pm9.32$  & $268.50\pm8.40$  & $204.80\pm7.24$  & $156.20\pm3.16$  & $119.00\pm4.32$  & $119.00\pm4.32$ \\
			0.111 & 0.7542 & $|I_+|$
			& $1498.10\pm9.33$  & $1118.10\pm6.14$  & $844.10\pm9.77$  & $632.00\pm10.26$ & $472.00\pm11.38$
			& $354.20\pm14.61$ & $268.10\pm11.24$ & $203.40\pm9.80$  & $154.00\pm8.73$  & $117.80\pm7.60$  & $117.80\pm7.60$ \\
			0.222 & 0.6600 & $|I_+|$
			& $1495.80\pm12.42$ & $1118.10\pm13.28$ & $843.60\pm17.55$ & $631.40\pm15.01$ & $476.00\pm6.85$
			& $359.00\pm6.60$  & $266.70\pm6.95$  & $198.10\pm5.59$  & $149.00\pm4.90$  & $112.20\pm4.94$  & $112.20\pm4.94$ \\
			0.333 & 0.5657 & $|I_+|$
			& $1498.50\pm16.95$ & $1119.40\pm14.86$ & $842.20\pm11.43$ & $630.00\pm10.30$ & $469.20\pm6.89$
			& $351.30\pm6.45$  & $263.70\pm4.88$  & $194.60\pm6.42$  & $147.00\pm7.85$  & $109.80\pm8.39$  & $109.80\pm8.39$ \\
			0.444 & 0.4714 & $|I_+|$
			& $1494.20\pm11.56$ & $1119.30\pm13.95$ & $850.00\pm15.38$ & $634.10\pm10.79$ & $477.40\pm10.67$
			& $358.40\pm8.40$  & $268.30\pm9.07$  & $202.60\pm9.73$  & $151.90\pm7.37$  & $115.80\pm5.65$  & $115.80\pm5.65$ \\
			0.556 & 0.3771 & $|I_+|$
			& $1498.50\pm12.58$ & $1123.30\pm14.86$ & $843.30\pm16.85$ & $630.10\pm12.36$ & $467.70\pm11.44$
			& $352.10\pm10.68$ & $264.10\pm8.40$  & $197.30\pm8.62$  & $145.40\pm7.03$  & $108.40\pm7.34$  & $108.40\pm7.34$ \\
			0.667 & 0.2828 & $|I_+|$
			& $1490.20\pm10.46$ & $1113.40\pm10.69$ & $837.60\pm13.96$ & $628.00\pm12.53$ & $472.80\pm11.07$
			& $355.70\pm13.12$ & $268.40\pm12.16$ & $201.30\pm10.18$ & $151.90\pm9.17$  & $113.80\pm7.51$  & $113.80\pm7.51$ \\
			0.778 & 0.1886 & $|I_+|$
			& $1449.20\pm10.79$ & $1072.80\pm16.72$ & $799.40\pm12.00$ & $598.30\pm7.59$  & $450.20\pm6.58$
			& $340.20\pm8.42$  & $257.80\pm5.61$  & $194.20\pm4.92$  & $146.20\pm2.90$  & $110.00\pm5.08$  & $110.00\pm5.08$ \\
			0.889 & 0.0943 & $|I_+|$
			& $1367.50\pm18.36$ & $963.20\pm16.96$  & $692.40\pm14.18$ & $509.80\pm12.97$ & $376.00\pm12.88$
			& $282.50\pm12.82$ & $212.80\pm6.97$  & $156.80\pm6.16$  & $118.20\pm6.78$  & $89.80\pm8.11$   & $89.80\pm8.11$ \\
			1.000 & 0.0000 & $|I_+|$
			& $1214.10\pm21.25$ & $756.20\pm29.91$  & $486.50\pm27.22$ & $315.10\pm17.48$ & $211.90\pm15.44$
			& $145.80\pm14.22$ & $101.10\pm9.24$  & $71.90\pm7.49$   & $49.10\pm7.69$   & $35.30\pm6.91$   & $35.30\pm6.91$ \\
			\bottomrule
		\end{tabular}%
	}
\end{table}

The overall computational complexity of SRRM is $\mathcal O(RK\, n \log n) + \mathcal O(u^3)$.
The first term comes from the recursive partitioning and sorting-based operations, whereas the second is due to solving local matching subproblems of size $u$ via the Hungarian algorithm.
When the two distributions are well separated (i.e., $\mathrm{sep}(t)$ is large), the induced local subproblem size $u$ is typically large, and the cubic term $\mathcal O(u^3)$ dominates the total cost.
Accordingly, the overall running time is relatively high in this regime.
As the distribution centers move closer together, the effective subproblem size $u$ decreases substantially.
As a result, the cubic term $\mathcal O(u^3)$ drops rapidly, leading to a corresponding reduction in the total running time.
Once $u$ becomes sufficiently small, the first term $\mathcal O(RK\, n \log n)$ becomes the leading term.

\section{Experiments}
\label{sec:EXPERIMENTS}

In all experiments, we apply an axis-aligned affine normalization to rescale the sampled points
to the unit box $[0,1]^d$ before computing the distance.
This preprocessing improves numerical stability and makes distances comparable across different
problem instances and scales. 

To demonstrate the feasibility and efficiency of our proposed SRRM, we conducted extensive numerical experiments and compared them with mainstream competitors, including maximum mean discrepancy (MMD) \cite{gretton2012mmd}, the Wasserstein distance, the Sinkhorn distance\cite{cuturi2013sinkhorn}, the SW distance\cite{bonneel2015sliced}, the max-SW distance \cite{deshpande2019maxsliced}, the GSW distance \cite{kolouri2019gsw}, the HCP distance\cite{li2024hcp},and the BSP distance\cite{ref12}.

For all distances, we adopt the Euclidean cost. All experiments are implemented on a 13th Gen Intel Core i7-13700HX CPU and an NVIDIA GeForce RTX 4070 Laptop GPU. For each experiment, we repeat it 20 times and report the average performance.

\subsection{Robustness and efficiency analysis}
The computational cost of our SRRM is mainly governed by two factors: the number of screening rounds $R$ in the point-selection procedure, and the number of anchors $z$ generated in the neighborhood of each unmatched point. To demonstrate the robustness and efficiency of the SRRM distance with respect to these two parameters, we conduct experiments on synthetic data and systematically analyze the influence of $R$ and $z$. Specifically, we draw two sample sets of size $n$ from the uniform distribution on the unit hypercube $[0,1]^3$, and compute the SRRM distance between them. 

Fig.~\ref{fig:example41} shows how the average SRRM and the average CPU time vary with the sample size $n$ under different numbers of anchors $z$, while keeping the merging budget $K$ and the number of screening rounds $R$ fixed.

Fig.~\ref{fig:example42} shows how the average SRRM and the average CPU time vary with the sample size $n$ under different numbers of screening rounds $R$, while keeping the merging budget $K$ and the number of anchors $z$ fixed.

The results show that, with the dimension and the merging budget $K$ fixed, the runtime scales approximately linearly with both $R$ and $z$.

\begin{figure}[htbp]
	\centering
	\begin{subfigure}{0.48\textwidth}
		\centering
		\includegraphics[width=\linewidth]{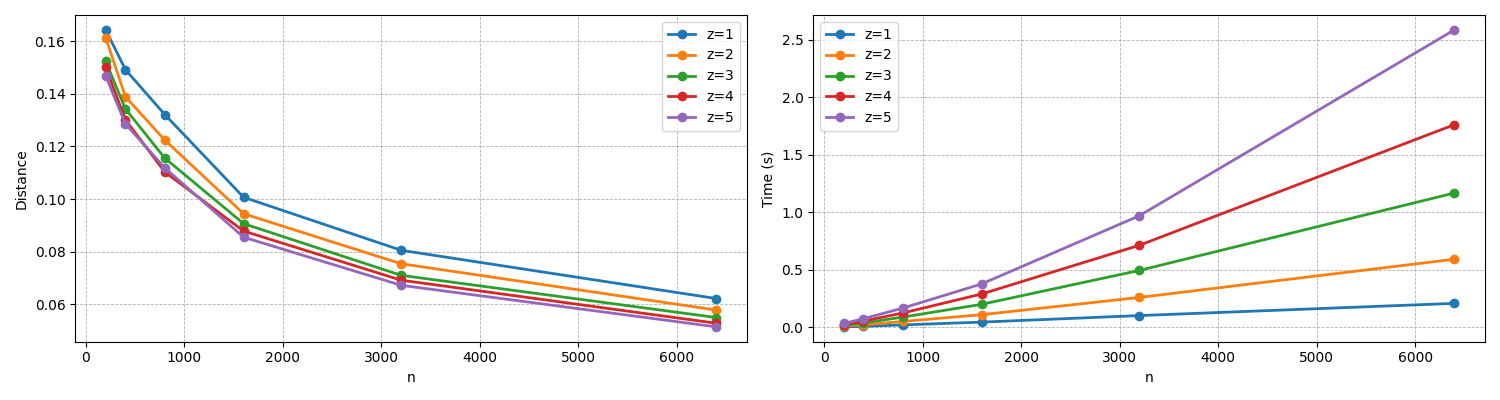}
		\caption{SRRM and its CPU time ($R=5$)}
		\label{fig:example41}
	\end{subfigure}
	\hfill
	\begin{subfigure}{0.48\textwidth}
		\centering
		\includegraphics[width=\linewidth]{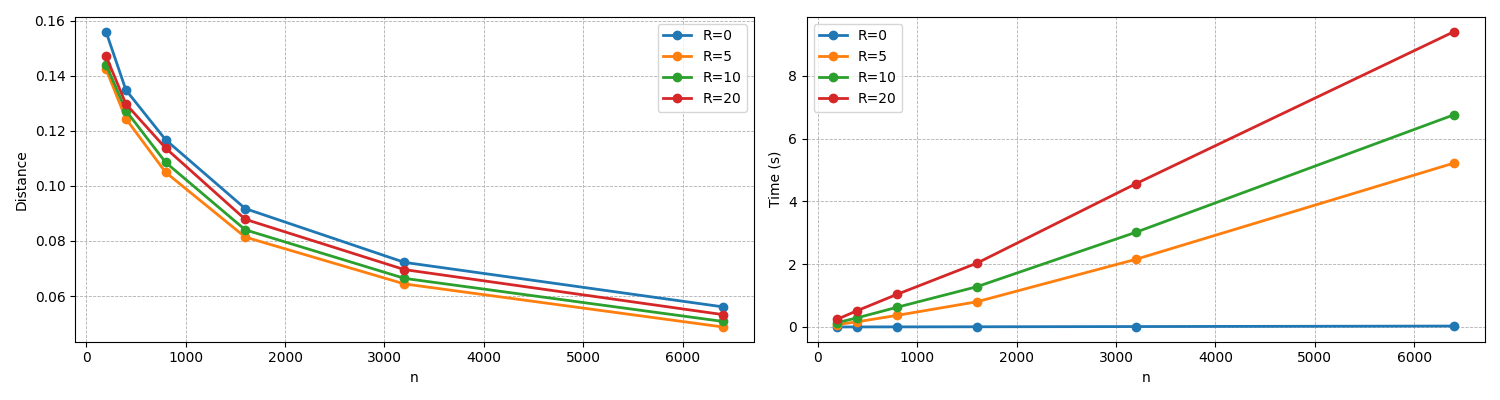}
		\caption{SRRM and its CPU time ($z=10$)}
		\label{fig:example42}
	\end{subfigure}
	
	\caption{Robustness analysis of SRRM under different parameter settings.}
	\label{fig:robustness}
\end{figure}

\subsection{Shape interpolation}
In the bijective setting, we first compute a bijective map $\pi$ between the source point set $X$ and the target point set $Y$, so that each source point $x_i$ is paired with a target point $y_{\pi(i)}$. Based on this map, we can construct the classical displacement interpolation\cite{bonneel2011,mccann1997}:
$
x_i(t) = (1-t)\,x_i + t\,y_{\pi(i)}, t \in [0,1],
$
that is, each point is advected along the straight segment from $x_i$ to its match $y_{\pi(i)}$. When $t=0$ we recover the source set, and when $t=1$ we reach the matched target set.

In practice, instead of varying $t$ continuously, we update the points via a discrete iterative scheme. The key idea is to move the point cloud toward the target \emph{while keeping the point-to-point correspondences as stable as possible across iterations}, thereby producing a more coherent and stable time-evolving matching. Concretely, at iteration $k$, we take the current point set $X^{(k)}$ as the starting point and compute a matching $\pi^{(k)}$. Denote the matched target set by
$
Y^{(k)}_{\mathrm{match}}=\{\,y_{\pi^{(k)}(i)}\,\}.
$
We then perform one convex-combination update with a step size $\alpha = 0.15$:
$
X^{(k+1)} = (1-\alpha)\,X^{(k)} + \alpha\,Y^{(k)}_{\mathrm{match}}.
$
With this update, each point moves by a small step along the line connecting its current matched pair. This avoids the abrupt ``jump'' that would occur if one directly switched to the matched target set as $t \to 1$, and it also helps maintain consistent correspondences throughout the iterations, making the interpolation process visually and geometrically smoother and more interpretable.The results are reported in Fig.~\ref{fig:three_layout_7_3}.

\begin{figure}[htbp]
	\centering
	
	\newlength{\figH}
	\setlength{\figH}{0.35\textheight}
	
	\newlength{\vGap}
	\setlength{\vGap}{2pt}
	
	\resizebox{0.9\textwidth}{!}{%
		\begin{minipage}[t]{0.18\linewidth}
			\centering
			\hspace*{-0.1\linewidth}
			\vbox to \figH{%
				\vfil
				\includegraphics[
				width=1.2\linewidth,
				height=\dimexpr(\figH-\vGap)/2\relax,
				keepaspectratio
				]{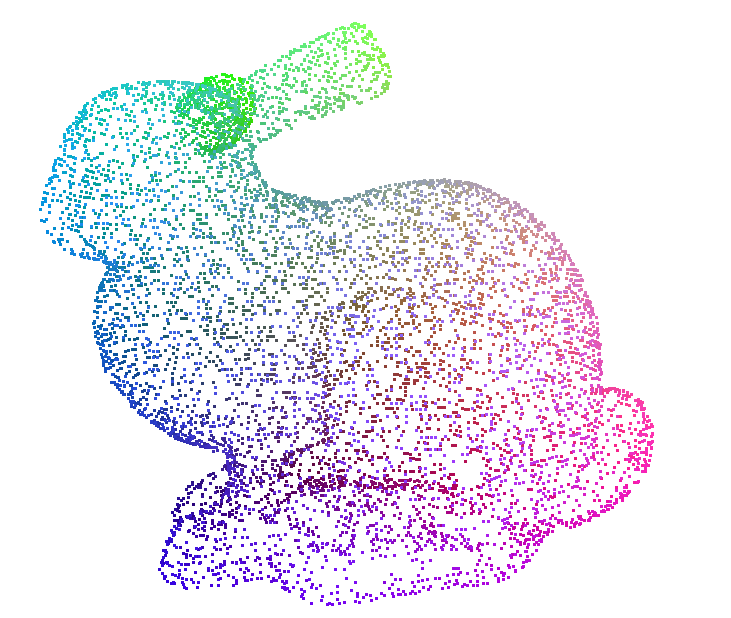}\vspace{\vGap}
				
				\includegraphics[
				width=1.2\linewidth,
				height=\dimexpr(\figH-\vGap)/2\relax,
				keepaspectratio
				]{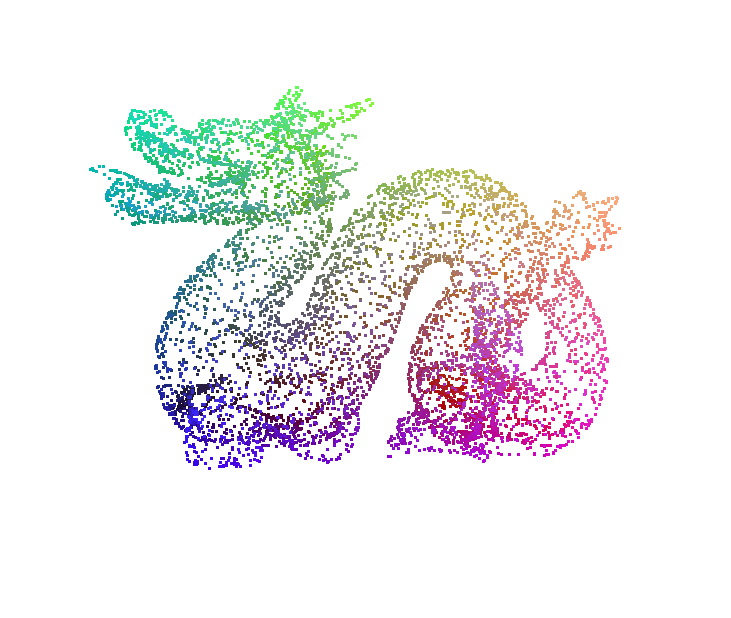}
				\vfil
			}
		\end{minipage}\hfill
		\begin{minipage}[t]{0.9\linewidth}
			\centering
			\hspace*{-0.01\linewidth}
			\vbox to \figH{%
				\vfil
				\begin{overpic}[height=\figH,keepaspectratio]{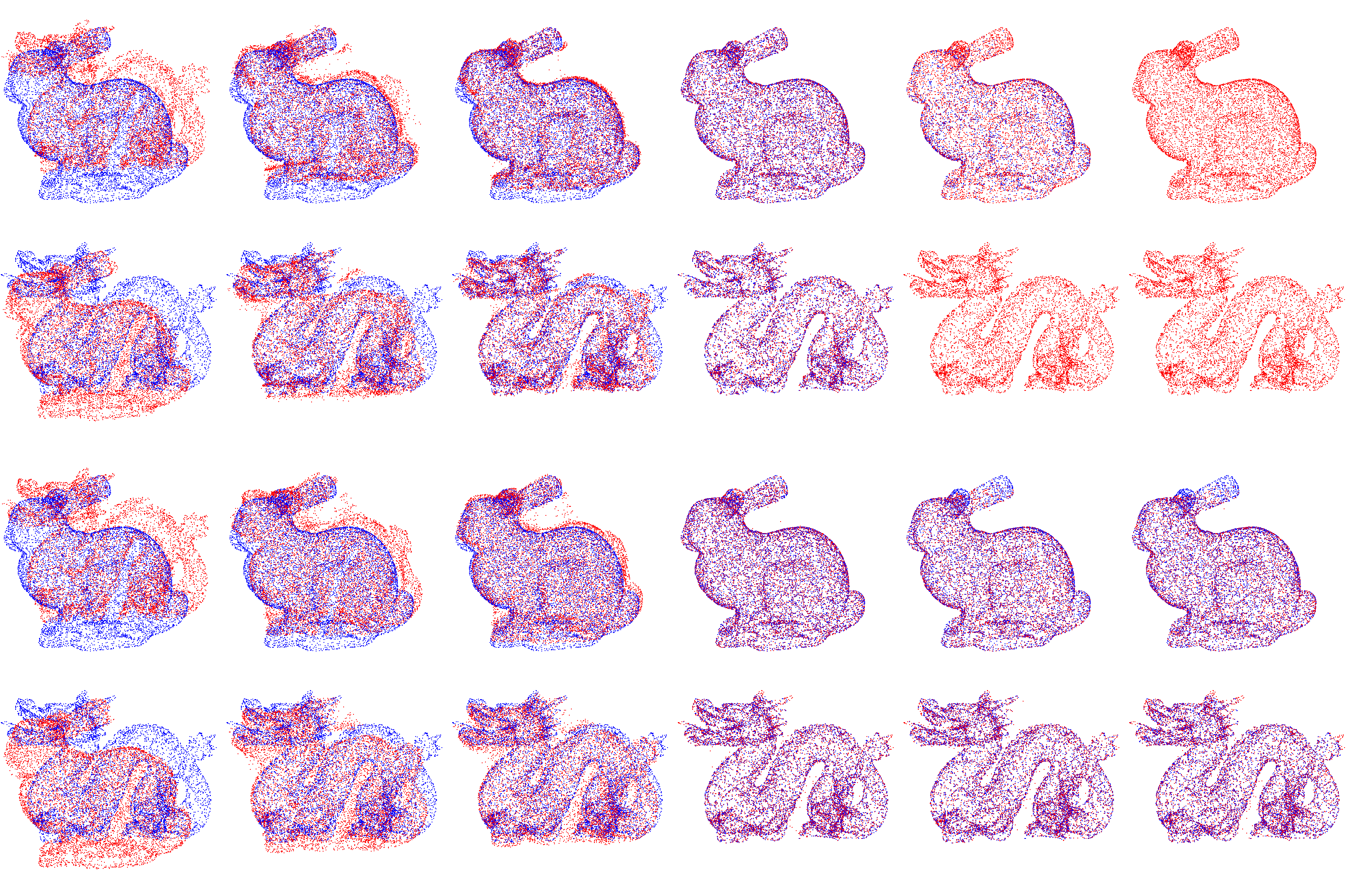}
					\put(0,33){SRRM}
					\put(0,0){BSP}				
				\end{overpic}
				\vfil
			}
		\end{minipage}
	}
	\caption{Comparison with BSP: Results of shape interpolation between two discrete uniform measures,
		a rabbit (top-left) and a dragon (bottom-left), each containing 8k points. Both methods use 16-plan
		merging and 300 iterations; for SRRM, the number of anchors is set to $k=2$ and the screening rounds
		are set to $R=10$. The evolution over iterations is shown in the right panel.}
	\label{fig:three_layout_7_3}
\end{figure}

\subsection{3D point cloud classification}
ModelNet10 \cite{wu2015shapenets} is a dataset for 3D shape classification, which contains 3,991 CAD objects from 10 categories for training and 908 CAD objects for testing. For each category, we randomly select 60 objects for training and 20 objects for testing. For each object, we randomly sample $n \in \{100, 200, 500, 1000, 2000\}$ points to obtain 3D point clouds. We then compute pairwise distances between point clouds under different distance metrics and evaluate the classification accuracy on the test set using the K-NN algorithm (neighbors$=5$).

For MMD, we use the RBF kernel; for SW, we set the number of slices to 10.For the BSP method, we merge the matching results from five independent runs; for SRRM, we similarly merge five runs to improve the stability and overall performance. We note that we do not apply the point-selection heuristic in this experiment. The main reason is that the point-selection heuristic is designed to mitigate the ``last-mile'' issue, i.e., to avoid prematurely separating very close point pairs during recursive partitioning. In the present task, however, the two point clouds are often far apart, in which case the final-stage Hungarian refinement in the iterative procedure becomes the dominant computational bottleneck. As a result, although the point-selection heuristic may further improve matching accuracy, it typically incurs a substantial runtime overhead.
Table~\ref{fig:vertical-four111} summarizes the averaged performance of each metric over 20 runs, showing that SRRM significantly outperforms other methods in both accuracy and efficiency.

\begin{table}[t]
	\centering
	\caption{Comparisons on 3D Point Cloud Classification} 
	\resizebox{\textwidth}{!}{ 
		\begin{tabular}{lccccc|ccccc}
			\hline\hline 
			\textbf{Method} & \textbf{n=100} & \textbf{n=200} & \textbf{n=500} & \textbf{n=1000} & \textbf{n=2000} & \textbf{n=100} & \textbf{n=200} & \textbf{n=500} & \textbf{n=1000} & \textbf{n=2000} \\
			\hline
			& \multicolumn{5}{c|}{\textbf{Accuracy (\%)}} & \multicolumn{5}{c}{\textbf{CPU time(s)}} \\
			\hline
			SRRM & 80.4$\pm$3.7 & 82.6$\pm$2.9 & 84.7$\pm$2.7 & \textbf{85.7}$\pm$2.6 & \textbf{85.5}$\pm$2.6 & 41.6 & 92.9 &250.9 & 545.8 & 1181.5 \\
			BSP & 78.9$\pm$3.6 & 82.9$\pm$2.6 & \textbf{84.8}$\pm$2.6 & 84.6$\pm$3.2 & 85.2$\pm$2.8 & 83.8 & 184.8 & 512.7 & 1102.7 & 2427.5 \\
			HCP & 73.6$\pm$2.9 & 78.8$\pm$2.9 & 81.8$\pm$3.3 & 82.0$\pm$2.9 & 82.6$\pm$2.8 & \textbf{10.2} & \textbf{17.9} & \textbf{51.9} & \textbf{120.6} & \textbf{240.3} \\
			SW & 70.3$\pm$4.1 & 73.1$\pm$3.6 & 76.2$\pm$4.4 & 76.7$\pm$4.1 & 77.4$\pm$3.9 & 94.9 & 159.7 & 365.0 & 764.5 & 1653.4 \\
			GSW(Poly5) & 70.1$\pm$3.3 & 72.9$\pm$3.9 & 75.9$\pm$4.4 & 77.0$\pm$3.5 & / & 782.8 & 986.2 & 1541.9 & 2334.6 & / \\
			MMD & 68.8$\pm$3.9 & 76.9$\pm$3.0 & / & / & / & 126.9 & 655.5 & / & / & / \\
			Sinkhorn & \textbf{81.4}$\pm$3.2 & \textbf{83.0}$\pm$3.2 & 84.5$\pm$2.7 & / & / & 308.4 & 639.6 & 2255.5 & / & / \\
			\hline\hline 
		\end{tabular}
	}
	\caption*{\textit{Note:"/"means that we fail to get result in 5,000 seconds.}} 
	\label{fig:vertical-four111}
\end{table}

\begin{figure}[H]
	\centering
	\includegraphics[width=0.8\linewidth]{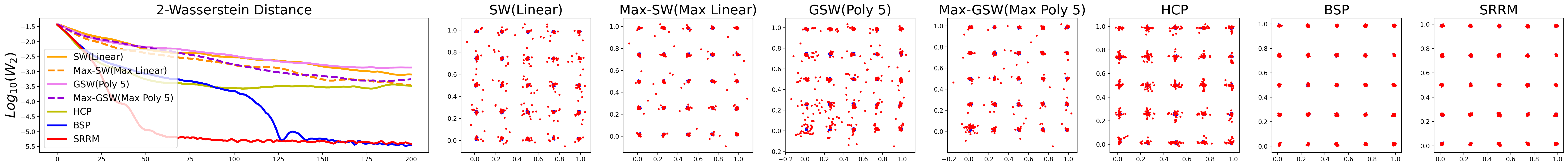}\\[3pt]
	\includegraphics[width=0.8\linewidth]{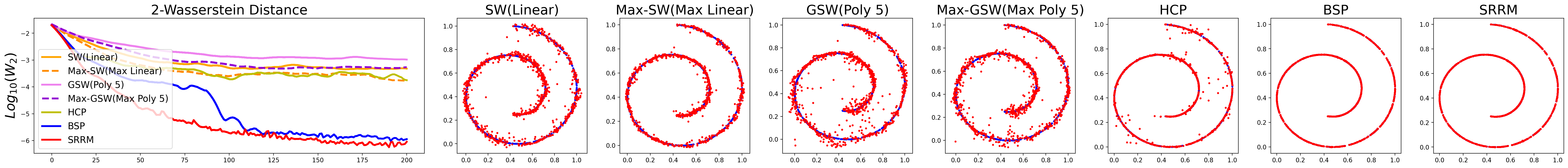}\\[3pt]
	\includegraphics[width=0.8\linewidth]{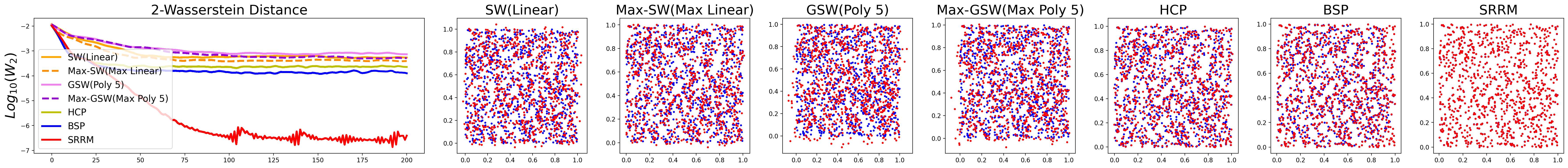}\\[3pt]
	\includegraphics[width=0.8\linewidth]{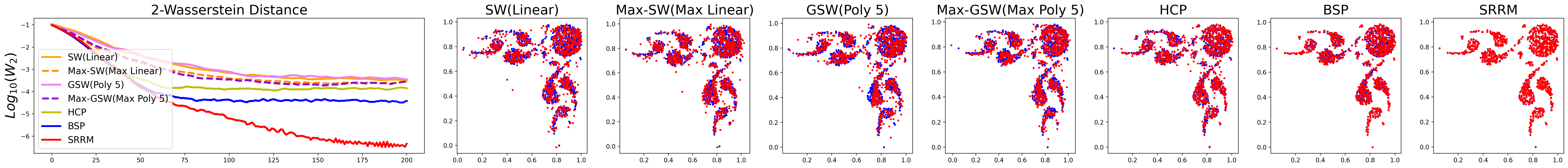}\\[3pt]
	\includegraphics[width=0.8\linewidth]{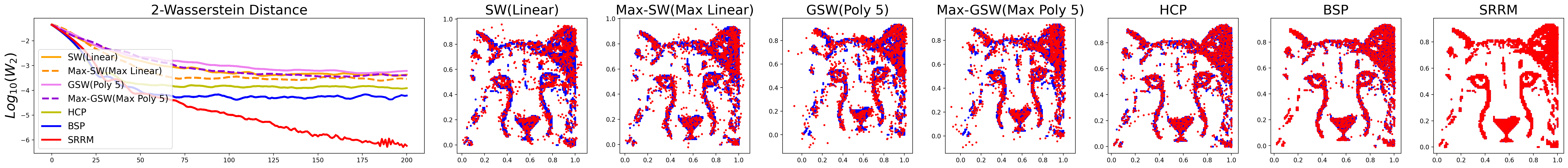}
	\caption{Log 2-Wasserstein distance between the source and target distributions versus the number of iterations t,both BSP and SRRM merge $K=10$ plans. For SRRM, we set the number of screening rounds to $R=10$ and the number of anchors to $k=5$.}
	\label{fig:vertical-four}
\end{figure}

\subsection{ Approximation of Wasserstein flow}

We consider the following problem \cite{kolouri2019gsw}: minimizing $\min_{\mu} W_{2}(\mu,\nu)$, where $\nu$ is a synthetically constructed target distribution and $\mu$ is the source distribution initialized as $\mu_{0}=\mathcal{N}(0,1)$ and updated iteratively via
$
\partial_{t}\mu_{t}=-\nabla W_{2}(\mu_{t},\nu).
$
We consider different choices for the target distribution $\nu$: 25-Gaussian, Swiss Roll, uniform, Decorative Border and puma, and approximate the Wasserstein distance $W_{2}$ using SW, max-SW, GSW, max-GSW, HCP, BSP, and SRRM. Each method applies one projection per iteration and sets the learning rate to $0.01$. The convergence curves of these methods and snapshots of the learned results at $t=200$ are shown in Fig.~\ref{fig:vertical-four}. We observe that SRRM yields better results.

Specifically, SW and its variants can typically recover only the coarse shape of the diffusing red points, while failing to fully align fine-grained structures and individual modes. In contrast, recursive methods derived from point-to-point couplings possess a natural advantage in this regard, including single-recursion coupling (HCP), multi-recursion merging (BSP), and multi-recursion merging with point selection (SRRM). In particular, SRRM further mitigates the \emph{last-mile} issue of BSP, enabling a much more thorough alignment between the red and blue point sets and achieving an almost complete fit. To the best of our knowledge, SRRM yields the best performance under this experimental setting among the methods considered.

For the $8\mathrm{K}$-point setting, we further design a corresponding experiment (see Fig.~\ref{fig:vertical-four1}.). In this larger-scale regime, the proposed SRRM method still demonstrates excellent performance. As a light-hearted nod to the pioneers of optimal transport, L.~Kantorovich and G.~Monge, we include their portraits in the figure.

\begin{figure}[htbp]
	\centering
	\includegraphics[width=0.6\linewidth]{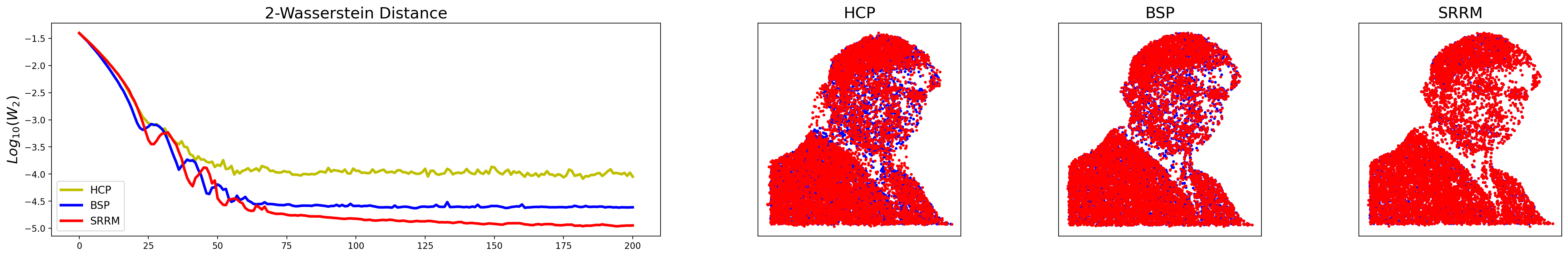}
	\includegraphics[width=0.6\linewidth]{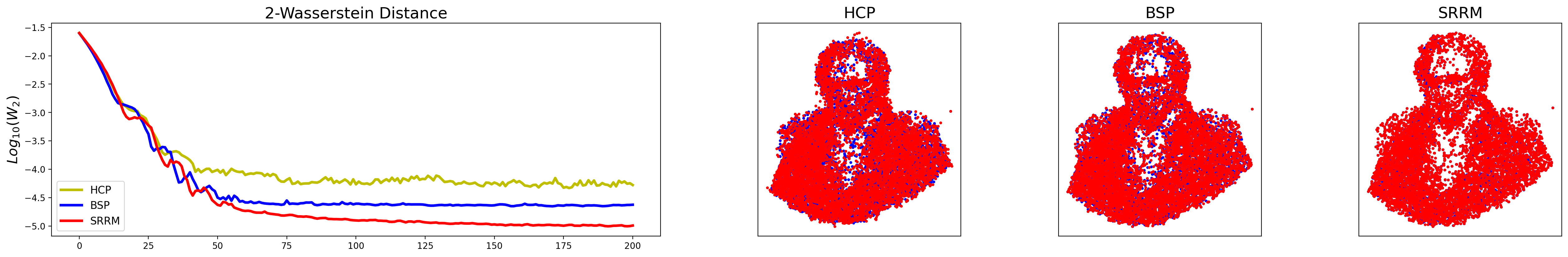}
	\caption{$8\mathrm{K}$-point setting}
	\label{fig:vertical-four1}
\end{figure}

\subsection{  Color transfer for images}
We further consider a real-world color transfer task.We perform bidirectional color transfer between a daytime seascape image and a sunset seascape image. Each image is represented as nearly one million pixels in the RGB space ($d=3$). Given the large sample size, we use the SW distance, HCP distance, BSP distance, and SRRM distance to approximate the Wasserstein flow, while adopting the same learning hyperparameters for all methods. The comparisons of the color transfer results and the optimization trajectories are presented in Fig.~\ref{fig:vertical-four11} and Fig.~\ref{fig:vertical-four22}, from which we observe that the SRRM distance achieves the best performance.

%

\begin{figure}[t]
	\centering
	
	\begin{subfigure}[b]{0.7\linewidth}
		\centering
		\begin{overpic}[width=\linewidth,percent]{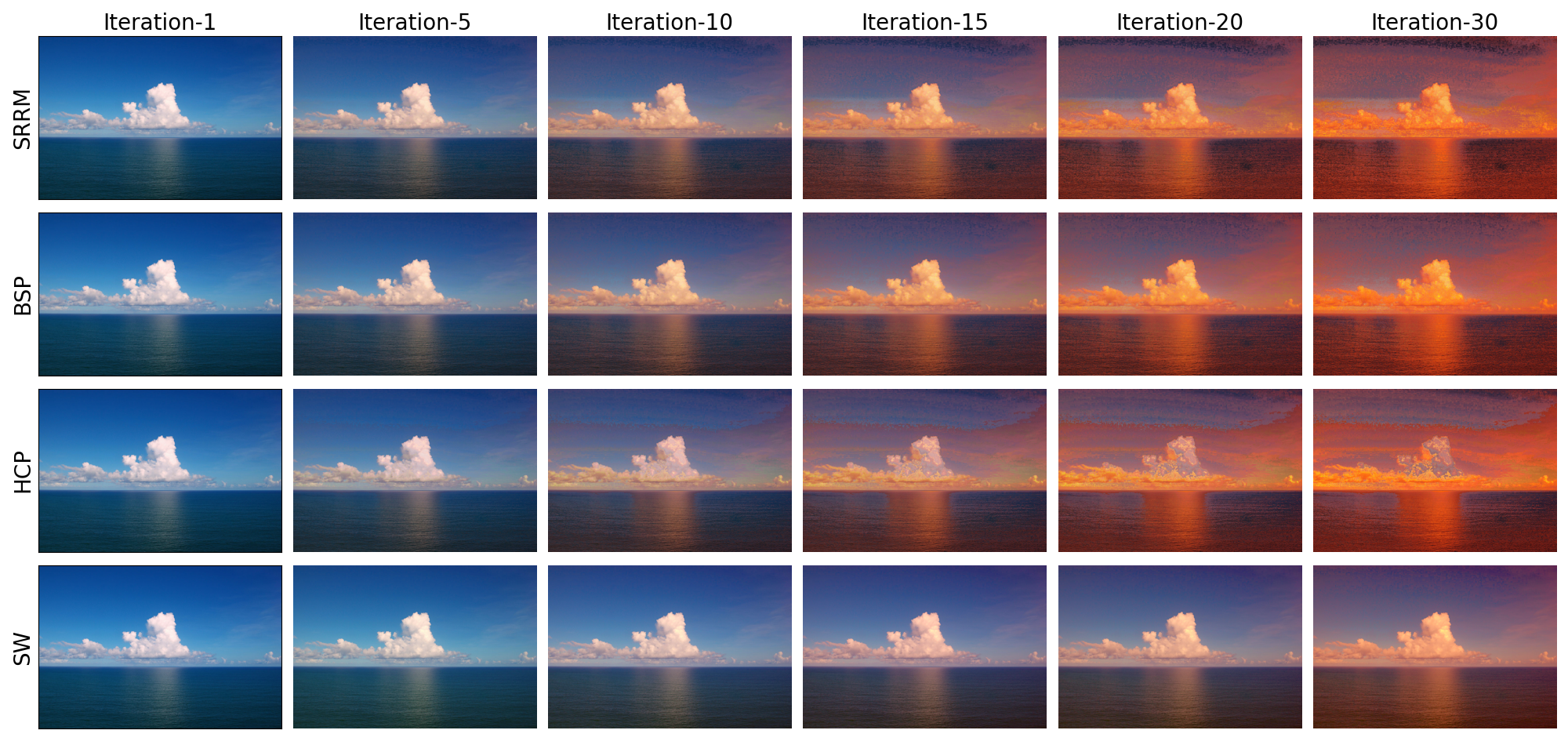}
		\end{overpic}
		\caption{The iteration process for the daytime-to-sunset color transfer}
		\label{fig:vertical-four11}
	\end{subfigure}
	
	\vspace{0.8em}
	
	\begin{subfigure}[b]{0.7\linewidth}
		\centering
		\begin{overpic}[width=\linewidth,percent]{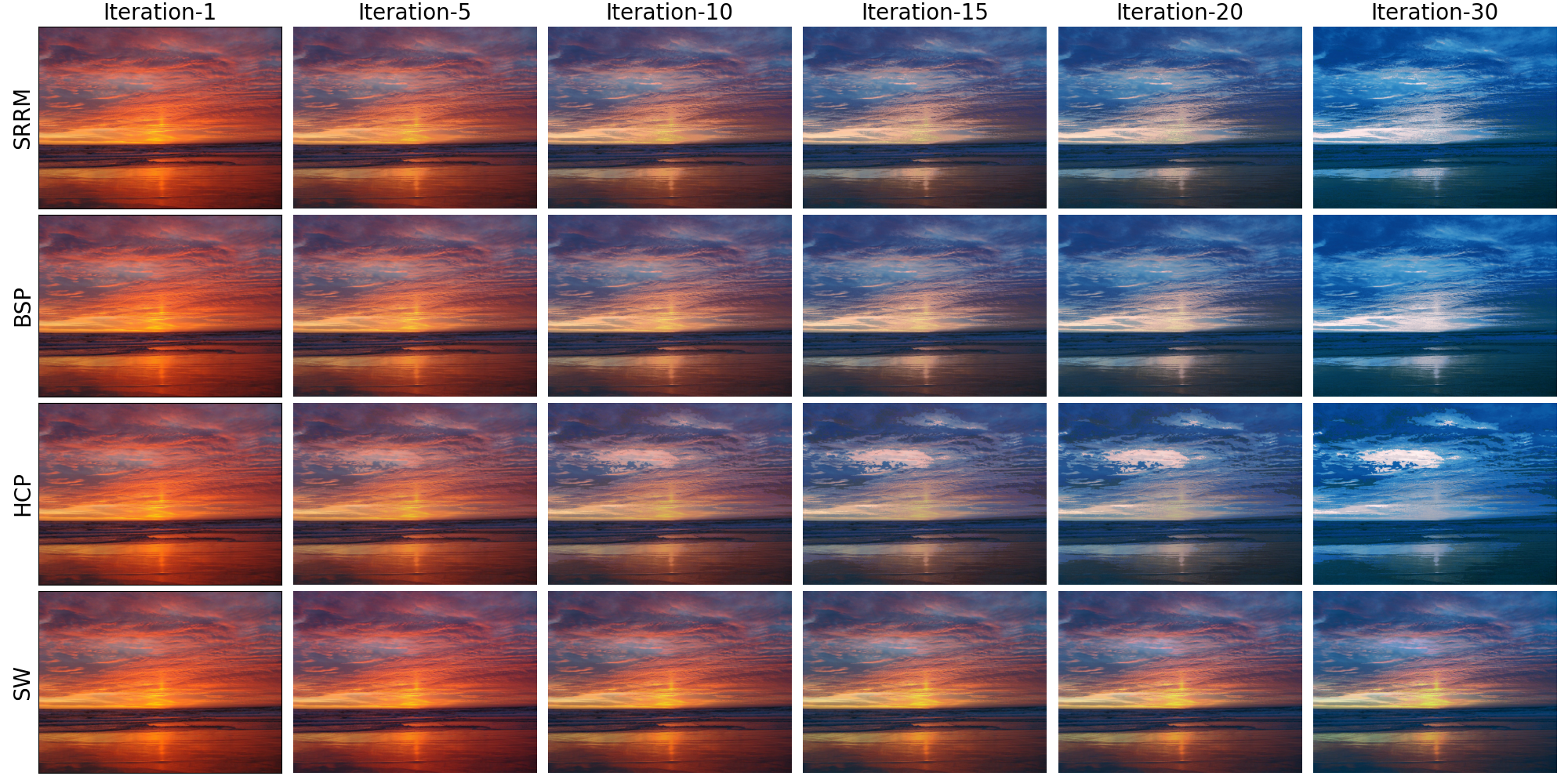}
		\end{overpic}
		\caption{The iteration process for the sunset-to-daytime color transfer}
		\label{fig:vertical-four22}
	\end{subfigure}
	\caption{Both BSP and SRRM merge $K=5$ plans. For SRRM, the number of screening rounds is set to $R=20$ after 27 iterations and $R=0$ for all other iterations, while the number of anchors is fixed at $k=2$.}

\end{figure}

\subsection{Generative modeling}

We can design new variants of the Wasserstein autoencoder (WAE) \cite{refTolstikhin2018} based on the proposed method. In particular, we use SRRM to penalize the discrepancy between the latent prior distribution and the aggregated posterior distribution. In our experiments, we compare SRRM with the HCP-based autoencoder and the BSP-based autoencoder. 

For visualization purposes, we encode the MNIST dataset \cite{refLeCun1998} into a two dimensional latent space. The encoder adopts a symmetric, classical deep convolutional neural network architecture with 2D average pooling and Leaky-ReLU activations, while the decoder includes upsampling layers. The batch size is set to 500. The reconstruction loss combines binary cross-entropy (BCE) and $\ell_1$ losses, while the latent-space regularization term uses the above three methods to measure the discrepancy between the aggregated posterior and the target prior. All loss terms are initialized with weights of 5. We train the model for 40 epochs, and starting from epoch 20, we progressively increase the weight of the latent regularizer by a factor of $1.1$ at each subsequent epoch. To evaluate performance, we randomly select 5{,}000 samples from the encoded test data points and draw random samples from the target prior distribution in the latent space. We observe that, in the latent space, our SRRM converges  faster than the other methods as shown in Fig.~\ref{fig:example21}. We find that although the three methods perform similarly in the image space, our SRRM enforces a substantially stronger regularization in the latent space, yielding latent codes that match the target prior more faithfully.

\begin{figure}[htbp]
	\centering
	
	\begin{subfigure}[b]{0.6\linewidth}
		\centering
		\begin{overpic}[width=\linewidth,percent]{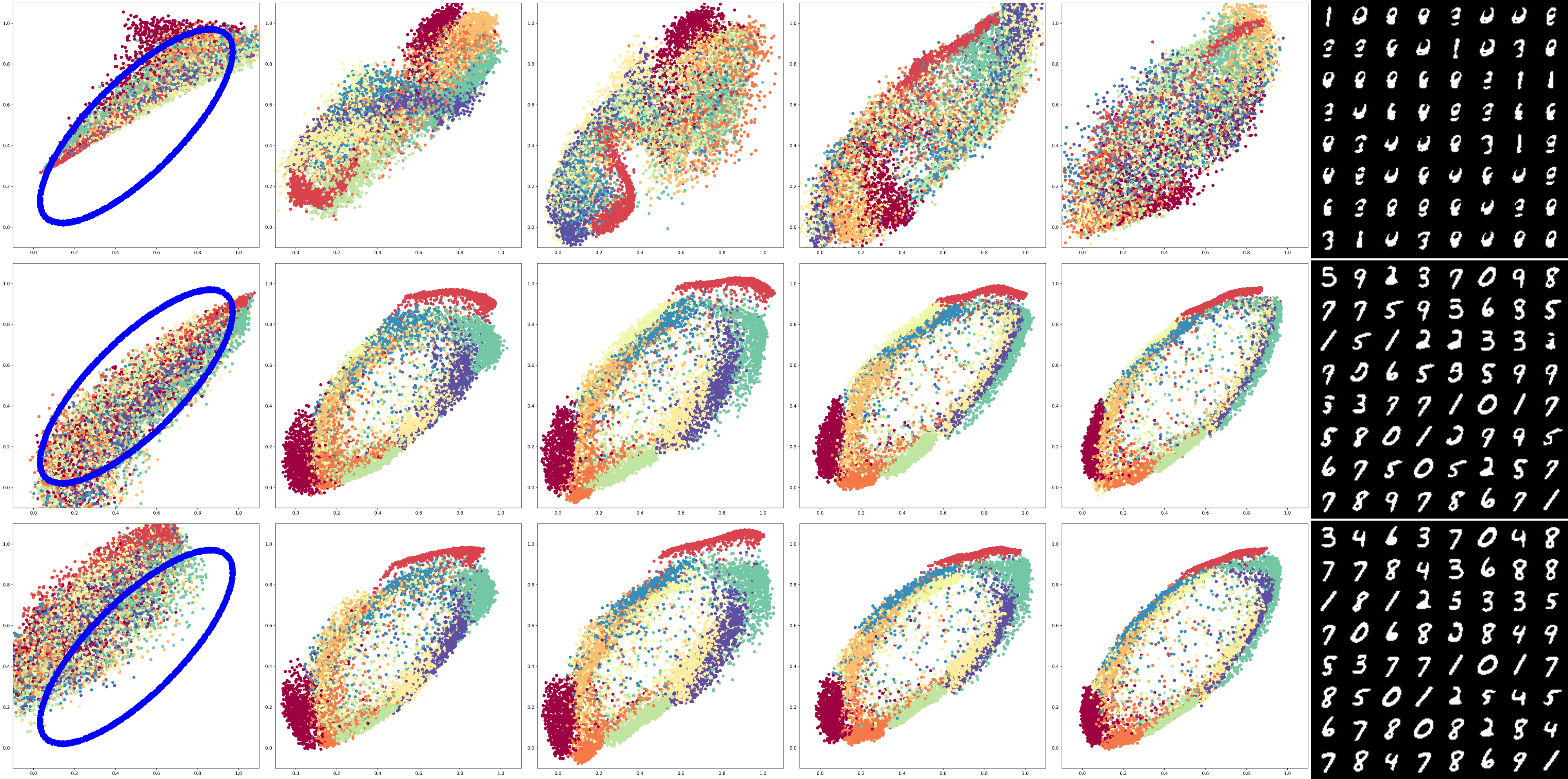}
			\put(0,36){\small HCP}
			\put(0,19){\small BSP}
			\put(0,1){\small SRRM}
		\end{overpic}
	\end{subfigure}
	
	\vspace{0.8em} 
	
	\begin{subfigure}[b]{0.6\linewidth}
		\centering
		\begin{overpic}[width=\linewidth,percent]{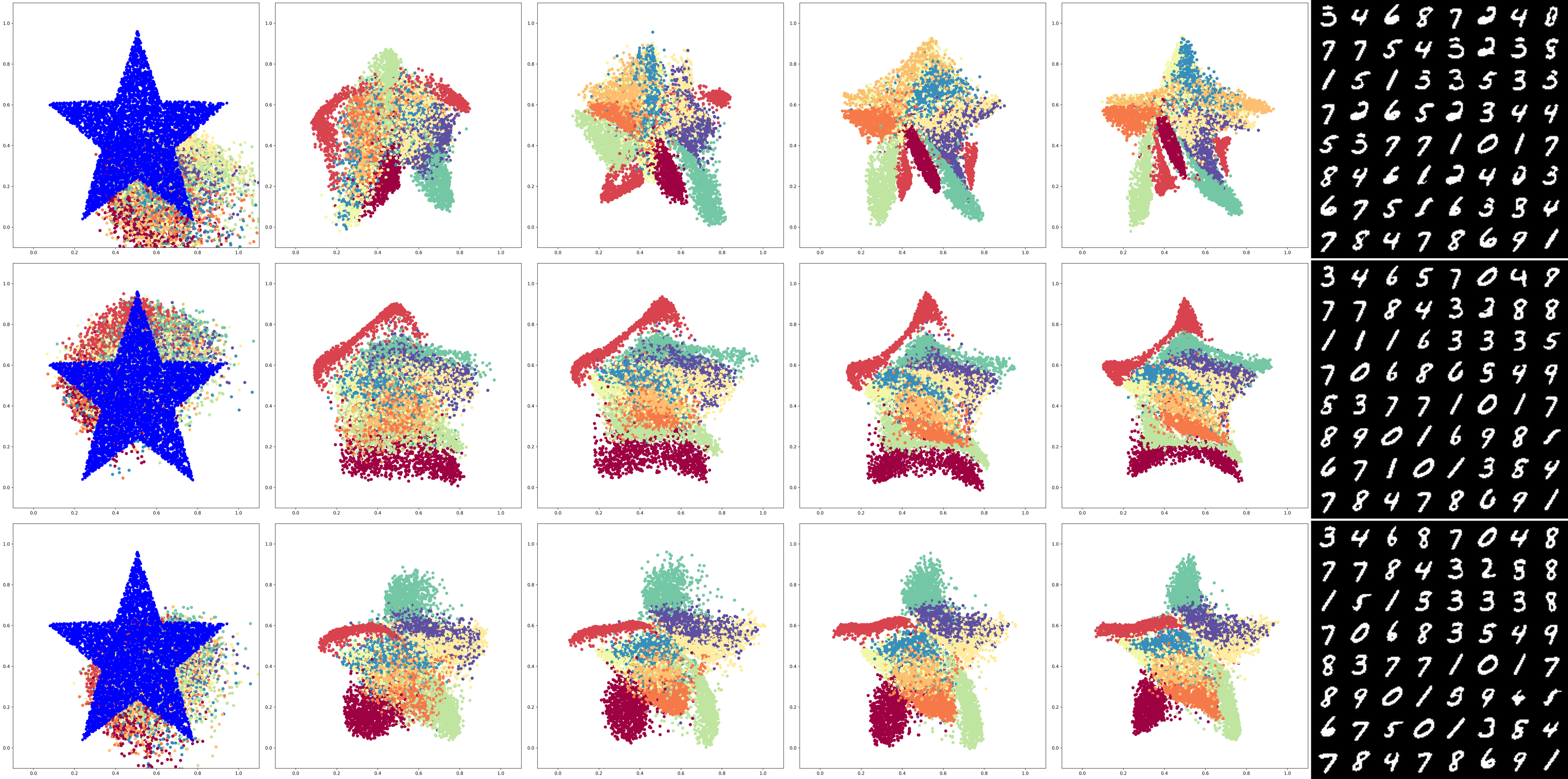}
			\put(0,36){\small HCP}
			\put(0,19){\small BSP}
			\put(0,1){\small SRRM}
		\end{overpic}
	\end{subfigure}
	
	\caption{The first five columns show the latent evolution every 10 epochs; the last column shows the corresponding reconstructions. Both BSP and SRRM merge $K=5$ plans. For SRRM, we set the number of screening rounds to $R=10$ and the number of anchors to $k=5$.}
	\label{fig:example21}  
\end{figure}
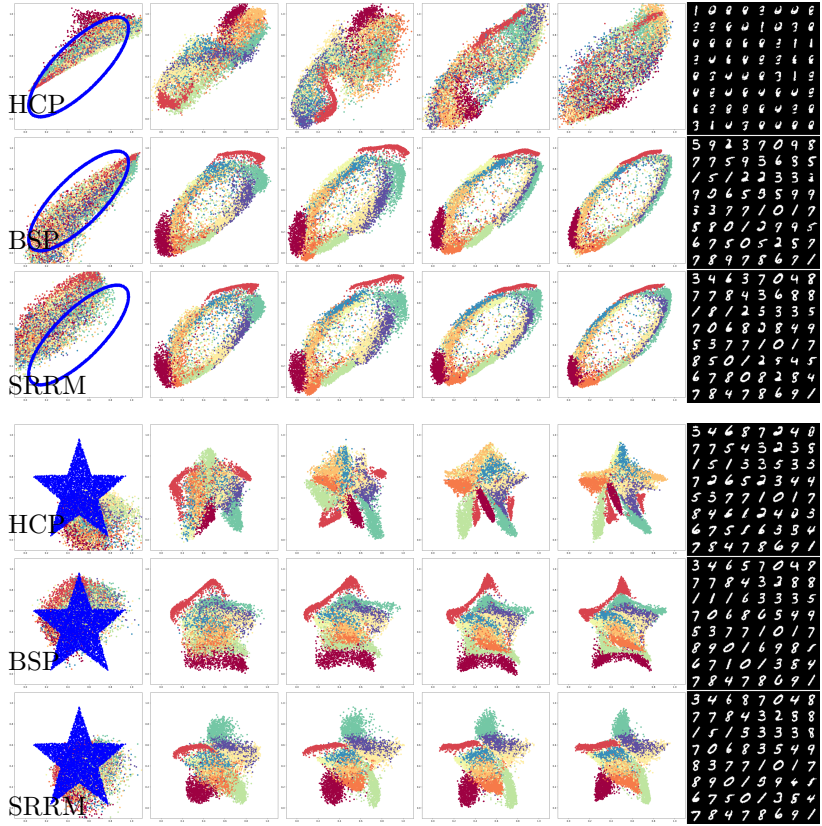

\section{Conclusion}
\label{sec:CONCLUSION}

In this work, we establish the convergence of the Anchored Empirical RRM and provide a proof of its convergence rate, thereby strengthening the theoretical foundations of this fast and reliable approach. Moreover, we give a theoretical explanation for the last-mile issue arising in a class of recursive partitioning methods and, guided by this insight, propose a heuristic algorithm to address this problem.

\textit{Limitations and Future Work:} The proposed SRRM algorithm mainly focuses on discrete measures with equal weights and the same number of support points. In addition, we do not yet address the curse of dimensionality that may arise in high-dimensional settings. We will continue to investigate more comprehensive theoretical properties of the RRM distance and its broader applicability.

\appendix
\section*{Appendix}
\label{sec:Appendix}

	\begin{definition}[Mass-median axis-recursive]Explicit recursive definition \label{definition1}
		
	\paragraph{(1) Axis recursion.}
	We adopt the axis-cycling strategy of \cite{ref10,ref11}. Specifically, at recursion depth $h$, the coordinate axis used for splitting is determined by
	\[
	j(h)=1+(h \bmod d), \qquad h=0,1,2,\ldots
	\]
	where $j(h)\in\{1,\ldots,d\}$ denotes the index of the coordinate axis selected for the split at depth $h$.

	\paragraph{(2) Mass-median bisection.}
	Fix $Q_{0,1}\subset\mathbb R^d$ with $\mathrm{supp}(\mu)\subset Q_{0,1}$ and let $j(h)=1+(h\bmod d)$.
	Given a depth-$h$ cell $Q_{h,k}$, define the conditional mass median along coordinate $j(h)$ by
	\[
	m_{h,k}:=\inf\Big\{m\in\mathbb{R}:\ \mu\big(Q_{h,k}\cap\{x_{j(h)}\le m\}\big)\ge\tfrac12\,\mu(Q_{h,k})\Big\}.
	\]
	By construction, this yields an equal-mass split.
	\[
	\mu\big(Q_{h,k}\cap\{x_{j(h)}\le m_{h,k}\}\big)
	=
	\mu\big(Q_{h,k}\cap\{x_{j(h)}> m_{h,k}\}\big)
	=\tfrac12 \mu(Q_{h,k}).
	\]
	Define the two children by
	\[
	Q_{h+1,2k-1}:=Q_{h,k}\cap\{x_{j(h)}\le m_{h,k}\},
	\qquad
	Q_{h+1,2k}:=Q_{h,k}\cap\{x_{j(h)}> m_{h,k}\}.
	\]
	In particular, for all $h\ge0$ and $k\in\{1,\ldots,2^h\}$, we have $\mu(Q_{h,k})=2^{-h}$.
	
	\paragraph{(3) $0/1$ path coding.}
	For $x\in\mathbb{R}^d$, define the $0/1$ digit at depth $h+1$ by
	\[
	s_{h+1}(x)=
	\begin{cases}
		0, & x\in Q_{h+1,2k-1} ( \text{left child of $Q_{h,k}$ }),\\
		1, & x\in Q_{h+1,2k}( \text{right child of $Q_{h,k}$ }),
	\end{cases}
	\]
	where $Q_{h,k}$ is the unique depth-$h$ cell containing $x$.
	Define the associated dyadic index
	\[
	\hat s_\mu(x):=\sum_{h=1}^\infty s_h(x)\,2^{-h}\in[0,1].
	\]
	
    \end{definition}

\begin{proof}[\textbf{Proof of Lemma~\ref{lem:pushforward_treecurve}}]
	\leavevmode\par\smallskip
	Fix $h\ge 0$. By the dyadic parametrization in Definition~\ref{lem:conditional-holder1},
	each depth-$h$ cell $Q_{h,k}$ is assigned a unique binary prefix
	$(s_1,\dots,s_h)\in\{0,1\}^h$ and hence a unique dyadic interval
	$
	I_{h,k}:=I(s_1,\dots,s_h)
	=\Big[\sum_{i=1}^h s_i2^{-i},\ \sum_{i=1}^h s_i2^{-i}+2^{-h}\Big),
	$
	with Lebesgue length $\lambda(I_{h,k})=2^{-h}$. Moreover, up to the dyadic endpoints,
	$
	T_\mu^{(h)^{-1}}(Q_{h,k}) = I_{h,k}.
	$
	Therefore,
	\[
	\lambda\!\big(T_\mu^{(h)^{-1}}(Q_{h,k})\big)=\lambda(I_{h,k})=2^{-h}.
	\]
	On the other hand, by the mass-median bisection in Definition~\ref{definition1}(2),
	each depth-$h$ cell has equal $\mu$-mass,
	$
	\mu(Q_{h,k})=2^{-h},\qquad k=1,\dots,2^h.
	$
	Hence for every $k$,
	\[
	(T_\mu^{(h)})_\#\lambda(Q_{h,k})=\mu(Q_{h,k}),
	\]
	and by additivity the same holds for every set $A$ in the finite algebra
	$\mathcal A_h$ generated by $\{Q_{h,k}\}_{k=1}^{2^h}$:
	\[
	(T_\mu^{(h)})_\#\lambda(A)=\mu(A),\qquad \forall A\in\mathcal A_h.
	\]
	
	Now define $Q_h(t):=T_\mu^{(h)}(t)$. For every $t\in[0,1]$ outside dyadic endpoints,
	$\{Q_h(t)\}_{h\ge 0}$ is nested and its diameters shrink to $0$ under repeated
	coordinate bisections; hence $\bigcap_{h\ge 0}Q_h(t)$ is a singleton.
	Define $T_\mu(t)$ as the unique point in $\bigcap_{h\ge 0}Q_h(t)$ and assign an arbitrary
	value on dyadic endpoints (a countable $\lambda$-null set). This yields a Borel measurable
	map $T_\mu$.
	
	Let $\mathcal E$ be the set of dyadic endpoints. For each $h,k$, if $t\notin\mathcal E$
	then $t\in I_{h,k}$ implies $T_\mu(t)\in Q_{h,k}$, hence
	$
	T_\mu^{-1}(Q_{h,k})\setminus T_\mu^{(h)^{-1}}(Q_{h,k})\subseteq \mathcal E,
	\qquad
	T_\mu^{(h)^{-1}}(Q_{h,k})\setminus T_\mu^{-1}(Q_{h,k})\subseteq \mathcal E.
	$
	Since $\lambda(\mathcal E)=0$, it follows that
	$
	\lambda\!\big(T_\mu^{-1}(Q_{h,k})\big)=\lambda\!\big(T_\mu^{(h)^{-1}}(Q_{h,k})\big)=\mu(Q_{h,k}).
	$
	and by additivity,
	\[
	\lambda\!\big(T_\mu^{-1}(A)\big)=\mu(A),\qquad \forall A\in\mathcal A_h.
	\]
	
	Finally, as $h$ varies, the algebras $\mathcal A_h$ generate the Borel $\sigma$-algebra
	on $\mathbb{R}^d$ (the partition refines and separates points up to boundaries). Therefore,
	by a standard monotone class argument, the identity
	$\lambda(T_\mu^{-1}(\cdot))=\mu(\cdot)$ extends from $\bigcup_h\mathcal A_h$ to all Borel
	sets $B\subset\mathbb{R}^d$, i.e.,
	\[
	\lambda\!\big(T_\mu^{-1}(B)\big)=\mu(B).
	\]
	Equivalently, $(T_\mu)_\#\lambda=\mu$.
\end{proof}

\begin{proof}[\textbf{Proof of Theorem~\ref{lem:conditional-holder4}}]
	\leavevmode\par\smallskip
	Let $\lambda$ denote the uniform probability measure on $[0,1]$.
	By Lemma~\ref{lem:pushforward_treecurve}, $(T_\mu)_\#\lambda=\mu$ and $(T_\nu)_\#\lambda=\nu$.
	Define a coupling $\pi$ between $\mu$ and $\nu$ by pushing forward $\lambda$
	through the map $t\mapsto (T_\mu(t),T_\nu(t))$:
	\[
	\pi := (T_\mu,T_\nu)_\#\lambda \in \Pi(\mu,\nu).
	\]
	Then, by definition of $W_2$ as the infimum transport cost over couplings,
	\[
	W_2^2(\mu,\nu)
	\le \int_{\mathbb R^d\times\mathbb R^d} \|x-y\|_2^2\,d\pi(x,y)
	= \int_0^1 \|T_\mu(t)-T_\nu(t)\|_2^2\,dt
	= \mathrm{RRM}^2(\mu,\nu).
	\]
	Taking square roots gives $W_2(\mu,\nu)\le \mathrm{RRM}(\mu,\nu)$.
	
	Nonnegativity and symmetry are immediate from the definition.
	For the triangle inequality, let $\rho$ be another measure in the class considered here, with associated $T_\rho$. Then by Minkowski's inequality in $L^2([0,1])$,
	\[
	\mathrm{RRM}(\mu,\nu)
	= \|T_\mu-T_\nu\|_{L^2}
	\le \|T_\mu-T_\rho\|_{L^2} + \|T_\rho-T_\nu\|_{L^2}
	= \mathrm{RRM}(\mu,\rho)+\mathrm{RRM}(\rho,\nu).
	\]
	
	It remains to prove identity of indiscernibles.
	If $\mathrm{RRM}(\mu,\nu)=0$, then $\|T_\mu(t)-T_\nu(t)\|_2=0$ for $\lambda$-a.e.\ $t$,
	so $T_\mu=T_\nu$ $\lambda$-a.e. Hence for any Borel set $B\subset\mathbb R^d$,
	\[
	\mu(B)=(T_\mu)_\#\lambda(B)=\lambda(T_\mu^{-1}(B))
	=\lambda(T_\nu^{-1}(B))=(T_\nu)_\#\lambda(B)=\nu(B),
	\]
	where we used Lemma~\ref{lem:pushforward_treecurve} and the fact that changing a map on a $\lambda$-null set
	does not change its pushforward measure. Thus $\mu=\nu$.
	
\end{proof}

\begin{proof}[\textbf{Proof of Lemma~\ref{lem:conditional-holder}}]
	\leavevmode\par\smallskip
	Fix $t<t'$ in $[0,1)$ and define the largest common-prefix depth 
	\[
	r(t,t'):=\max\{h\ge 0:\ \lfloor 2^h t\rfloor=\lfloor 2^h t'\rfloor\}.
	\]
	By definition of $r(t,t')$, for every $h\le r(t,t')$ we have
	$\lfloor 2^h t\rfloor=\lfloor 2^h t'\rfloor$, hence $t$ and $t'$ belong to the same dyadic prefix interval
	$I_{h,k}$ at depth $h$ . By the construction of the RRM tree-curve order,
	this implies that $t$ and $t'$ are assigned to the same depth-$h$ cell, so that
	\[
	Q_h(t)=Q_h(t')\qquad\text{for all }h\le r(t,t').
	\]
	In particular, both $T_\mu(t)$ and $T_\mu(t')$ belong to the common ancestor cell
	\[
	Q:=Q_{r(t,t')}(t)=Q_{r(t,t')}(t'),
	\]
	and therefore
	\begin{equation}
		\label{eq:Tu_diam_bound2}
		\|T_\mu(t)-T_\mu(t')\|_2 \le \operatorname{diam}(Q).
	\end{equation}
	
	We next control $\operatorname{diam}(Q)$ under the mass-median partition.
	Consider an arbitrary cell $\widetilde Q$ that is split along some coordinate axis.
	Let $[a,b]$ be the projection of $\widetilde Q$ onto that axis, with length $L=b-a$.
	After rearranging coordinates, write $\widetilde Q=[a,b]\times S$, where
	$S\subset\mathbb R^{d-1}$ is the projection onto the remaining coordinates.
	For any sub-interval $I\subset[a,b]$ of length $\delta$, using $f\le M$ we have
	\[
	\mu(I\times S)\le \int_{I\times S} M\,dx = M\,\delta\,|S|,
	\]
	while using $f\ge m$ we have
	\[
	\mu(\widetilde Q)=\mu([a,b]\times S)\ge \int_{[a,b]\times S} m\,dx = m\,L\,|S|.
	\]
	Since the cut is performed at a conditional mass median, each child has mass $\frac12\mu(\widetilde Q)$.
	Hence, if a child has axis-length $\delta$ along the split coordinate, then
	\[
	\frac12\mu(\widetilde Q)=\mu(\text{child})\le M\,\delta\,|S|
	\quad\Longrightarrow\quad
	\delta \ge \frac{m}{2M}L.
	\]
	Therefore both children have axis-length at most
	\[
	L-\frac{m}{2M}L = \Big(1-\frac{m}{2M}\Big)L.
	\]
	Set
	\[
	\rho:=1-\frac{m}{2M}\in(0,1).
	\]
	
	Now consider the depth-$h$ cell $Q_h(t)$ containing $T_\mu(t)$.
	Under the axis-cycling schedule $j(h)=1+(h\bmod d)$, each coordinate is selected either
	$\lfloor h/d\rfloor$ or $\lceil h/d\rceil$ times among the first $h$ refinements.
	Each time a coordinate is selected, the side-length in that coordinate is multiplied by at most $\rho$.
	
	Let $L_i$ be the side-length of the root box $Q_{0,1}$ in coordinate $i$, and set
	$L_{\max}:=\max_{1\le i\le d} L_i$.
	Hence every side-length of $Q_h(t)$ is bounded by $L_{\max}\rho^{\lfloor h/d\rfloor}$, and thus
	\[
	\operatorname{diam}(Q_h(t))
	\le \sqrt d\,\max_{1\le i\le d}\text{(side length in coord.\ $i$)}
	\le \sqrt d\,L_{\max}\rho^{\lfloor h/d\rfloor}
	\le (\sqrt d\,L_{\max}\rho^{-1})\,\rho^{h/d}.
	\]
	Absorbing the factor $\sqrt d\,L_{\max}\rho^{-1}$ into a constant, we obtain
	\[
	\operatorname{diam}(Q_h(t)) \le C\,\rho^{h/d},
	\]
	for some constant $C>0$ depending only on $d,m,M$ and $Q_{0,1}$ (through $L_{\max}$).
	Applying this bound with $h=r(t,t')$ and combining with \eqref{eq:Tu_diam_bound2} yields
	\[
	\|T_\mu(t)-T_\mu(t')\|_2 \le C\,\rho^{\,r(t,t')/d}.
	\]
	
	Let $r:=r(t,t')$. Since $t$ and $t'$ share the same dyadic prefix of length $r$ but not of length $r+1$,
	they belong to the same dyadic interval of length $2^{-r}$ but different dyadic intervals of length $2^{-(r+1)}$.
	Therefore
	\[
	2^{-(r+1)}\le t'-t < 2^{-r},
	\qquad\text{hence}\qquad
	2^{-r}\le 2(t'-t).
	\]
	Define
	\[
	\alpha:=\frac1d\log_2\!\Big(\frac1\rho\Big)>0,
	\]
	so that $\rho^{r/d}=(2^{-r})^\alpha$. Using $2^{-r}\le 2(t'-t)$, we obtain
	\[
	\rho^{r/d}=(2^{-r})^\alpha \le \big(2(t'-t)\big)^\alpha = 2^\alpha (t'-t)^\alpha.
	\]
	Substituting gives
	\[
	\|T_\mu(t)-T_\mu(t')\|_2 \le C\,2^\alpha\,(t'-t)^\alpha.
	\]
	Since the left-hand side is symmetric in $(t,t')$, the same bound holds for all $t,t'\in[0,1)$.
	Setting $C_d:=C\,2^\alpha$ completes the proof.
\end{proof}

\begin{proof}[\textbf{Proof of Theorem~\ref{thm:rrm_pure_good}}]
	\leavevmode\par\smallskip
	Throughout, write $y_t:=g_{\mu_n}^{-1}(t)$. By Lemma~\ref{lem:pushforward_treecurve},
	we have $g_\mu(t)=t$ for all $t\in[0,1]$, hence $g_\mu^{-1}(t)=t$ and therefore
	\[
	|t-y_t|
	=
	\big|g_\mu^{-1}(t)-g_{\mu_n}^{-1}(t)\big|.
	\]
	By the global H\"older control in Lemma~\ref{lem:conditional-holder}, for all $t\in[0,1]$,
	\[
	\big\|T_\mu(t)-T_\mu(y_t)\big\|_2
	\le
	C_d\,|t-y_t|^\alpha.
	\]
	Squaring and integrating yields
	\begin{align*}
		\mathrm{RRM}(\mu,\mu_n)^2
		&=
		\int_0^1\big\|T_\mu(t)-T_\mu(y_t)\big\|_2^2\,dt
		\le
		C_d^2\int_0^1 |t-y_t|^{2\alpha}\,dt,
	\end{align*}
	and hence
	\begin{equation}
		\label{eq:rrm_global_reduction}
		\mathrm{RRM}(\mu,\mu_n)
		\le
		C_d\Bigg(\int_0^1 |t-y_t|^{2\alpha}\,dt\Bigg)^{1/2}.
	\end{equation}
	
	Moreover, by the one-dimensional quantile representation of $W_1$,
	\begin{equation}
		\label{eq:w1_quantile_repr_global}
		\int_0^1 |t-y_t|\,dt
		=
		\int_0^1\big|g_{\mu_n}^{-1}(t)-g_\mu^{-1}(t)\big|\,dt
		=
		W_1(\mu_n,\mu).
	\end{equation}
	Since $t,y_t\in[0,1]$, we have $|t-y_t|\le 1$. Let $h(t):=|t-y_t|\in[0,1]$.
	
	If $2\alpha\le 1$, then $x\mapsto x^{2\alpha}$ is concave on $\mathbb{R}_+$, and Jensen's inequality gives
	\[
	\int_0^1 h(t)^{2\alpha}\,dt
	\le
	\Big(\int_0^1 h(t)\,dt\Big)^{2\alpha}.
	\]
	Combining with \eqref{eq:rrm_global_reduction}--\eqref{eq:w1_quantile_repr_global} yields
	\[
	\mathrm{RRM}(\mu,\mu_n)
	\le
	C\,W_1(\mu_n,\mu)^\alpha.
	\]
	If $2\alpha>1$, using the interpolation inequality
	\[
	\int_0^1 h^{2\alpha}\,dt
	\le
	\|h\|_\infty^{2\alpha-1}\int_0^1 h\,dt,
	\]
	and $\|h\|_\infty\le 1$, we obtain $\int_0^1 h^{2\alpha}\,dt\le \int_0^1 h\,dt$, hence
	\[
	\mathrm{RRM}(\mu,\mu_n)
	\le
	C\,W_1(\mu_n,\mu)^{1/2}.
	\]
	Under our setting $\mu$ is supported on $\mathbb{R}^d$, so $\mu$ has finite first moment and it is classical that
	$W_1(\mu_n,\mu)\to 0$ almost surely. Therefore, in either case,
	$\mathrm{RRM}(\mu,\mu_n)\to 0$ almost surely.
	
	For the expectation bound, if $2\alpha\le 1$, taking expectations and applying Jensen to the concave map
	$x\mapsto x^\alpha$ gives
	\[
	\mathbb{E}\,\mathrm{RRM}(\mu,\mu_n)
	\le
	C\,\mathbb{E}\big[W_1(\mu_n,\mu)^\alpha\big]
	\le
	C\,\big(\mathbb{E}W_1(\mu_n,\mu)\big)^\alpha.
	\]
	If $2\alpha>1$, taking expectations and applying Jensen to $x\mapsto x^{1/2}$ yields
	\[
	\mathbb{E}\,\mathrm{RRM}(\mu,\mu_n)
	\le
	C\,\big(\mathbb{E}W_1(\mu_n,\mu)\big)^{1/2}.
	\]
	Finally, by the standard one-dimensional bound \cite{panaretos2019statistical},
	there exists $C_1>0$ such that
	\[
	\mathbb{E}W_1(\mu_n,\mu)\le C_1 n^{-1/2}.
	\]
	Substituting yields
	\[
	\mathbb{E}\,\mathrm{RRM}(\mu,\mu_n)
	=
	O\!\left(n^{-\alpha/2}\right)
	\quad\text{when }2\alpha\le 1,
	\qquad
	\mathbb{E}\,\mathrm{RRM}(\mu,\mu_n)
	=
	O\!\left(n^{-1/4}\right)
	\quad\text{when }2\alpha>1,
	\]
	or equivalently,
	\[
	\mathbb{E}\,\mathrm{RRM}(\mu,\mu_n)
	=
	O\!\left(n^{-\min(\alpha/2,\,1/4)}\right).
	\]
	This completes the proof.
\end{proof}

\begin{proof}[\textbf{Proof of Corollary~\ref{cor:rrm_two_sample}}]
	\leavevmode\par\smallskip
	Write
	\[
	\mathrm{RRM}(\mu,\nu)=\|X-Y\|_{L_2([0,1])},\qquad
	\mathrm{RRM}(\mu_n,\nu_n)=\|\widehat X-\widehat Y\|_{L_2([0,1])},
	\]
	where $X,Y$ denote the population curve representations associated with $\mu,\nu$ and
	$\widehat X,\widehat Y$ denote the corresponding empirical curve representations associated with
	$\mu_n,\nu_n$. By the reverse triangle inequality in $L_2$,
	\[
	\Big|\mathrm{RRM}(\mu_n,\nu_n)-\mathrm{RRM}(\mu,\nu)\Big|
	=
	\Big|\|\widehat X-\widehat Y\|_{L_2}-\|X-Y\|_{L_2}\Big|
	\le
	\|(\widehat X-\widehat Y)-(X-Y)\|_{L_2}.
	\]
	Applying the triangle inequality again yields
	\[
	\|(\widehat X-\widehat Y)-(X-Y)\|_{L_2}
	\le
	\|\widehat X-X\|_{L_2}+\|\widehat Y-Y\|_{L_2}
	=
	\mathrm{RRM}(\mu,\mu_n)+\mathrm{RRM}(\nu,\nu_n).
	\]
	By Theorem~\ref{thm:rrm_pure_good}, we have almost surely
	$\mathrm{RRM}(\mu,\mu_n)\to 0$ and $\mathrm{RRM}(\nu,\nu_n)\to 0$,
	and therefore the right-hand side converges to $0$ almost surely. This proves
	\[
	\mathrm{RRM}(\mu_n,\nu_n)\to \mathrm{RRM}(\mu,\nu)\qquad\text{almost surely.}
	\]
	
	For the expectation bound, take expectations in the stability inequality to obtain
	\[
	\mathbb{E}\Big|\mathrm{RRM}(\mu_n,\nu_n)-\mathrm{RRM}(\mu,\nu)\Big|
	\le
	\mathbb{E}\,\mathrm{RRM}(\mu,\mu_n)+\mathbb{E}\,\mathrm{RRM}(\nu,\nu_n).
	\]
	Applying again Theorem~\ref{thm:rrm_pure_good} (to $\mu$ and to $\nu$) gives
	\[
	\mathbb{E}\,\mathrm{RRM}(\mu,\mu_n)
	=
	O\!\left(n^{-\min(\alpha_\mu/2,\,1/4)}\right),
	\qquad
	\mathbb{E}\,\mathrm{RRM}(\nu,\nu_n)
	=
	O\!\left(n^{-\min(\alpha_\nu/2,\,1/4)}\right).
	\]
	Hence, with $\alpha_*=\min\{\alpha_\mu,\alpha_\nu\}$,
	\[
	\mathbb{E}\Big|\mathrm{RRM}(\mu_n,\nu_n)-\mathrm{RRM}(\mu,\nu)\Big|
	=
	O\!\left(n^{-\min(\alpha_*/2,\,1/4)}\right),
	\]
	which completes the proof.
\end{proof}

\begin{proof}[\textbf{Proof of Theorem~\ref{thm:finite-depth-tree-consistency}}]
	\leavevmode\par\smallskip
	We divide the proof into three steps.
	\medskip
	\noindent
	\textbf{Step 1: uniform empirical-process control on axis-aligned boxes.}
	Let \(\mathcal R_d\) denote the class of all axis-aligned boxes contained in \(Q_{0,1}\).
	Since \(\mathcal R_d\) is a VC class, the uniform Glivenko--Cantelli theorem yields\cite{ref70}
	\begin{equation}\label{eq:UGC-boxes-final}
		\Delta_n
		:=
		\sup_{R\in\mathcal R_d}
		|\mu_n(R)-\mu(R)|
		\xrightarrow[]{a.s.}0.
	\end{equation}
	In particular,
	$
	\Delta_n\xrightarrow[]{P}0.
	$
	
	\medskip
	\noindent
	\textbf{Step 2: induction on the depth.}
	We prove by induction on \(h=0,1,\dots,H-1\) that
	\begin{equation}\label{eq:induction-threshold-final}
		\max_{1\le k\le 2^h}|m_{h,k}^{(n)}-m_{h,k}|
		\xrightarrow[]{P}0,
	\end{equation}
	and simultaneously
	\begin{equation}\label{eq:induction-cell-final}
		\max_{1\le k\le 2^h}
		\mu\bigl(Q_{h,k}^{(n)}\triangle Q_{h,k}\bigr)
		\xrightarrow[]{P}0.
	\end{equation}
	
	\smallskip
	\noindent
	\emph{Base case \(h=0\).}
	The root cell is deterministic:
	$
	Q_{0,1}^{(n)}=Q_{0,1}.
	$
	Hence \eqref{eq:induction-cell-final} is immediate for \(h=0\).
	
	We now prove \eqref{eq:induction-threshold-final} for \(h=0\).
	Let \(j=j(0)\), and define
	\[
	F_{0,1}(t)
	:=
	\frac{\mu(Q_{0,1}\cap\{x_j\le t\})}{\mu(Q_{0,1})},
	\qquad
	F_{0,1}^{(n)}(t)
	:=
	\frac{\mu_n(Q_{0,1}\cap\{x_j\le t\})}{\mu_n(Q_{0,1})}.
	\]
	Since \(Q_{0,1}\) is fixed and \(\mu_n(Q_{0,1})=\mu(Q_{0,1})=1\), \eqref{eq:UGC-boxes-final} implies
	\[
	\sup_{t\in\mathbb R}|F_{0,1}^{(n)}(t)-F_{0,1}(t)|
	\le \Delta_n
	\xrightarrow[]{P}0.
	\]
	
	By Definition~\ref{definition1}, \(m_{0,1}\) is the lower conditional mass median on \(Q_{0,1}\), so
	$
	F_{0,1}(m_{0,1})=\frac12.
	$
	
	We next verify a quantitative invertibility estimate for \(F_{0,1}\).
	After reordering coordinates so that the first coordinate is \(x_j\), write
	$
	Q_{0,1}=[a,b]\times S,
	\qquad
	S\subset\mathbb R^{d-1}.
	$
	Then for any \(t_1<t_2\) in \([a,b]\),
	\[
	F_{0,1}(t_2)-F_{0,1}(t_1)
	=
	\frac{\int_{t_1}^{t_2}\int_S f(u,y)\,dy\,du}{\mu(Q_{0,1})}
	\ge
	\frac{m|S|}{\mu(Q_{0,1})}(t_2-t_1).
	\]
	Since
	$
	\mu(Q_{0,1})
	\le
	M(b-a)|S|
	\le
	M L_{\max}|S|,
	$
	where \(L_{\max}\) denotes the maximal side length of the root box \(Q_{0,1}\), we obtain
	\[
	F_{0,1}(t_2)-F_{0,1}(t_1)
	\ge
	\frac{m}{M L_{\max}}(t_2-t_1).
	\]
	Therefore \(F_{0,1}\) is strictly increasing on the projection interval of \(Q_{0,1}\), and with
	$
	c_0:=\frac{m}{M L_{\max}},
	$
	we have
	\begin{equation}\label{eq:invertibility-root-final}
		|F_{0,1}(t)-F_{0,1}(m_{0,1})|
		\ge
		c_0|t-m_{0,1}|
		\qquad
		\text{for all }t\text{ in the projection interval of }Q_{0,1}.
	\end{equation}
	
	Since \(m_{0,1}^{(n)}\) is the lower empirical median on \(Q_{0,1}\), the jump size of
	\(F_{0,1}^{(n)}\) is \(1/n\), and therefore
	\[
	\left|F_{0,1}^{(n)}(m_{0,1}^{(n)})-\frac12\right|
	\le \frac1n.
	\]
	Hence
	\[
	\left|F_{0,1}(m_{0,1}^{(n)})-\frac12\right|
	\le
	\left|F_{0,1}(m_{0,1}^{(n)})-F_{0,1}^{(n)}(m_{0,1}^{(n)})\right|
	+
	\left|F_{0,1}^{(n)}(m_{0,1}^{(n)})-\frac12\right|
	\]
	and thus
	\[
	\left|F_{0,1}(m_{0,1}^{(n)})-\frac12\right|
	\le
	\sup_{t\in\mathbb R}|F_{0,1}^{(n)}(t)-F_{0,1}(t)|
	+
	\frac1n
	\xrightarrow[]{P}0.
	\]
	Applying \eqref{eq:invertibility-root-final}, we conclude that
	$
	|m_{0,1}^{(n)}-m_{0,1}|\xrightarrow[]{P}0.
	$
	Hence both \eqref{eq:induction-threshold-final} and \eqref{eq:induction-cell-final} hold at depth \(0\).
	
	\smallskip
	\noindent
	\emph{Induction step.}
	Fix \(h\in\{1,\dots,H-1\}\), and assume that \eqref{eq:induction-threshold-final} and \eqref{eq:induction-cell-final} hold at depth \(h-1\).
	We prove them at depth \(h\).
	
	\medskip
	\noindent
	\emph{Step 2a: convergence of cells at depth \(h\).}
	Fix \(1\le k\le 2^h\), and let \(a=\lceil k/2\rceil\) be the parent index at depth \(h-1\).
	Write \(j=j(h-1)\), and denote the population and empirical parent cells by
	\[
	P:=Q_{h-1,a},
	\qquad
	P_n:=Q_{h-1,a}^{(n)}.
	\]
	By the induction hypothesis,
	\[
	\mu(P_n\triangle P)\xrightarrow[]{P}0,
	\qquad
	m_{h-1,a}^{(n)}-m_{h-1,a}\xrightarrow[]{P}0.
	\]
	
	We first treat the case where \(Q_{h,k}\) is the left child of \(P\), namely
	\[
	Q_{h,k}=P\cap\{x_j\le m_{h-1,a}\},
	\qquad
	Q_{h,k}^{(n)}=P_n\cap\{x_j\le m_{h-1,a}^{(n)}\}.
	\]
	Then
	\[
	Q_{h,k}^{(n)}\triangle Q_{h,k}
	\subset
	(P_n\triangle P)
	\cup
	\Bigl(
	P\cap
	\{\min(m_{h-1,a},m_{h-1,a}^{(n)})<x_j\le \max(m_{h-1,a},m_{h-1,a}^{(n)})\}
	\Bigr).
	\]
	Taking \(\mu\)-mass gives
	\begin{equation}\label{eq:cell-left-bound-final}
		\mu(Q_{h,k}^{(n)}\triangle Q_{h,k})
		\le
		\mu(P_n\triangle P)
		+
		\mu\Bigl(
		P\cap
		\{|x_j-m_{h-1,a}|\le |m_{h-1,a}^{(n)}-m_{h-1,a}|\}
		\Bigr).
	\end{equation}
	
	The first term tends to zero in probability by the induction hypothesis.
	To bound the second term, write \(P\), after reordering coordinates so that the first coordinate is \(x_j\), in the form
	$
	P=[\alpha,\beta]\times S_P,
	\qquad
	S_P\subset\mathbb R^{d-1}.
	$
	Since \(P\subset Q_{0,1}\) and \(P\) is axis-aligned, each side length of \(S_P\) is bounded by
	the corresponding side length of \(Q_{0,1}\). Hence \(|S_P|\) is bounded by a constant depending
	only on \(d\) and \(Q_{0,1}\).
	Using \(f\le M\), we obtain
	\[
	\mu\Bigl(
	P\cap
	\{|x_j-m_{h-1,a}|\le \varepsilon\}
	\Bigr)
	\le
	2M\,|S_P|\,\varepsilon
	\le
	C\,\varepsilon
	\]
	for all \(\varepsilon>0\), where \(C>0\) depends only on \(d\), \(M\), and \(Q_{0,1}\).
	Taking
	$
	\varepsilon=|m_{h-1,a}^{(n)}-m_{h-1,a}|,
	$
	the second term in \eqref{eq:cell-left-bound-final} tends to zero in probability.
	Therefore
	$
	\mu(Q_{h,k}^{(n)}\triangle Q_{h,k})\xrightarrow[]{P}0.
	$
	
	The right-child case,
	\[
	Q_{h,k}=P\cap\{x_j>m_{h-1,a}\},
	\qquad
	Q_{h,k}^{(n)}=P_n\cap\{x_j>m_{h-1,a}^{(n)}\},
	\]
	is entirely analogous.
	Since there are only finitely many \(k\) at depth \(h\), this proves \eqref{eq:induction-cell-final} at depth \(h\).
	
	\medskip
	\noindent
	\emph{Step 2b: convergence of split thresholds at depth \(h\).}
	Fix \(1\le k\le 2^h\), and let \(j=j(h)\) be the splitting coordinate at depth \(h\).
	Define the population conditional distribution function on the fixed cell \(Q_{h,k}\) by
	$
	F_{h,k}(t)
	:=
	\frac{\mu(Q_{h,k}\cap\{x_j\le t\})}{\mu(Q_{h,k})},
	$
	and the empirical conditional distribution function on the random cell \(Q_{h,k}^{(n)}\) by
	$
	F_{h,k}^{(n)}(t)
	:=
	\frac{\mu_n(Q_{h,k}^{(n)}\cap\{x_j\le t\})}{\mu_n(Q_{h,k}^{(n)})}.
	$
	
	We first prove
	\begin{equation}\label{eq:cdf-convergence-depth-h-final}
		\sup_{t\in\mathbb R}|F_{h,k}^{(n)}(t)-F_{h,k}(t)|
		\xrightarrow[]{P}0.
	\end{equation}
	
	For each \(t\in\mathbb R\),
	\[
	\bigl|
	\mu_n(Q_{h,k}^{(n)}\cap\{x_j\le t\})
	-
	\mu(Q_{h,k}\cap\{x_j\le t\})
	\bigr|
	\]
	is bounded by
	\[
	\bigl|
	\mu_n(Q_{h,k}^{(n)}\cap\{x_j\le t\})
	-
	\mu(Q_{h,k}^{(n)}\cap\{x_j\le t\})
	\bigr|
	+
	\bigl|
	\mu(Q_{h,k}^{(n)}\cap\{x_j\le t\})
	-
	\mu(Q_{h,k}\cap\{x_j\le t\})
	\bigr|.
	\]
	The first term is bounded by \(\Delta_n\), because for each fixed \(t\), the set \(Q_{h,k}^{(n)}\cap\{x_j\le t\}\) is an axis-aligned box contained in \(Q_{0,1}\), and \(\Delta_n\to 0\) by \eqref{eq:UGC-boxes-final}.
	The second term is bounded by
	$
	\mu(Q_{h,k}^{(n)}\triangle Q_{h,k}),
	$
	which tends to zero in probability by \eqref{eq:induction-cell-final}.
	Therefore
	\begin{equation}\label{eq:numerator-conv-final}
		\sup_{t\in\mathbb R}
		\bigl|
		\mu_n(Q_{h,k}^{(n)}\cap\{x_j\le t\})
		-
		\mu(Q_{h,k}\cap\{x_j\le t\})
		\bigr|
		\xrightarrow[]{P}0.
	\end{equation}
	
	Likewise, taking \(t=+\infty\) in the same argument yields
	\begin{equation}\label{eq:denominator-conv-final}
		\mu_n(Q_{h,k}^{(n)})-\mu(Q_{h,k})\xrightarrow[]{P}0.
	\end{equation}
	Since \(\mu(Q_{h,k})=2^{-h}>0\), it follows from \eqref{eq:denominator-conv-final} that
	$
	\mu_n(Q_{h,k}^{(n)})\xrightarrow[]{P}2^{-h},
	$
	and hence the denominators are bounded away from zero in probability.
	
	We now pass from \eqref{eq:numerator-conv-final} and \eqref{eq:denominator-conv-final} to the conditional CDFs.
	For every \(t\in\mathbb R\),
	\[
	F_{h,k}^{(n)}(t)-F_{h,k}(t)
	=
	\frac{
		\mu_n(Q_{h,k}^{(n)}\cap\{x_j\le t\})-\mu(Q_{h,k}\cap\{x_j\le t\})
	}{
		\mu_n(Q_{h,k}^{(n)})
	}
	+
	\mu(Q_{h,k}\cap\{x_j\le t\})
	\left(
	\frac{1}{\mu_n(Q_{h,k}^{(n)})}
	-
	\frac{1}{\mu(Q_{h,k})}
	\right).
	\]
	Therefore,
	\[
	|F_{h,k}^{(n)}(t)-F_{h,k}(t)|
	\le
	\frac{
		\bigl|
		\mu_n(Q_{h,k}^{(n)}\cap\{x_j\le t\})-\mu(Q_{h,k}\cap\{x_j\le t\})
		\bigr|
	}{
		\mu_n(Q_{h,k}^{(n)})
	}
	+
	\mu(Q_{h,k})
	\left|
	\frac{1}{\mu_n(Q_{h,k}^{(n)})}
	-
	\frac{1}{\mu(Q_{h,k})}
	\right|.
	\]
	Taking the supremum over \(t\in\mathbb R\), we obtain
	\[
	\sup_{t\in\mathbb R}|F_{h,k}^{(n)}(t)-F_{h,k}(t)|
	\le
	\frac{
		\sup_{t\in\mathbb R}
		\bigl|
		\mu_n(Q_{h,k}^{(n)}\cap\{x_j\le t\})
		-
		\mu(Q_{h,k}\cap\{x_j\le t\})
		\bigr|
	}{
		\mu_n(Q_{h,k}^{(n)})
	}
	+
	\mu(Q_{h,k})
	\left|
	\frac{1}{\mu_n(Q_{h,k}^{(n)})}
	-
	\frac{1}{\mu(Q_{h,k})}
	\right|.
	\]
	The first term converges to zero in probability by \eqref{eq:numerator-conv-final} together with the fact that
	\(\mu_n(Q_{h,k}^{(n)})\) is bounded away from zero in probability.
	For the second term, \eqref{eq:denominator-conv-final} implies
	\[
	\left|
	\frac{1}{\mu_n(Q_{h,k}^{(n)})}
	-
	\frac{1}{\mu(Q_{h,k})}
	\right|
	=
	\frac{
		\bigl|\mu(Q_{h,k})-\mu_n(Q_{h,k}^{(n)})\bigr|
	}{
		\mu_n(Q_{h,k}^{(n)})\,\mu(Q_{h,k})
	}
	\xrightarrow[]{P}0.
	\]
	Hence \eqref{eq:cdf-convergence-depth-h-final} holds.
	
	Now, by Definition~8.1(2), \(m_{h,k}\) is the lower conditional mass median on \(Q_{h,k}\), so
	$
	F_{h,k}(m_{h,k})=\frac12.
	$
	
	We next verify a quantitative invertibility estimate for \(F_{h,k}\).
	Write \(Q_{h,k}\), after reordering coordinates so that the first coordinate is \(x_j\), in the form
	\[
	Q_{h,k}=[a_{h,k},b_{h,k}]\times S_{h,k},
	\qquad
	S_{h,k}\subset\mathbb R^{d-1}.
	\]
	Then for any \(t_1<t_2\) in \([a_{h,k},b_{h,k}]\),
	\[
	F_{h,k}(t_2)-F_{h,k}(t_1)
	=
	\frac{\int_{t_1}^{t_2}\int_{S_{h,k}} f(u,y)\,dy\,du}{\mu(Q_{h,k})}
	\ge
	\frac{m|S_{h,k}|}{\mu(Q_{h,k})}(t_2-t_1).
	\]
	Also,
	\[
	\mu(Q_{h,k})
	\le
	M(b_{h,k}-a_{h,k})|S_{h,k}|
	\le
	M L_{\max}|S_{h,k}|,
	\]
	where \(L_{\max}\) is the maximal side length of the root box \(Q_{0,1}\).
	Hence
	\[
	F_{h,k}(t_2)-F_{h,k}(t_1)
	\ge
	\frac{m}{M L_{\max}}(t_2-t_1).
	\]
	Thus, with the same constant
	$
	c_0:=\frac{m}{M L_{\max}},
	$
	we obtain
	\begin{equation}\label{eq:invertibility-depth-h-final}
		|F_{h,k}(t)-F_{h,k}(m_{h,k})|
		\ge
		c_0|t-m_{h,k}|
		\qquad
		\text{for all }t\text{ in the projection interval of }Q_{h,k}.
	\end{equation}
	
	Let
	$
	N_{h,k}^{(n)}:=n\,\mu_n(Q_{h,k}^{(n)})
	$
	be the number of sample points in \(Q_{h,k}^{(n)}\).
	Since \(m_{h,k}^{(n)}\) is the lower empirical median on \(Q_{h,k}^{(n)}\), the jump size of
	\(F_{h,k}^{(n)}\) is \(1/N_{h,k}^{(n)}\), and therefore
	\[
	\left|F_{h,k}^{(n)}(m_{h,k}^{(n)})-\frac12\right|
	\le
	\frac{1}{N_{h,k}^{(n)}}
	=
	\frac{1}{n\,\mu_n(Q_{h,k}^{(n)})}.
	\]
	Since \(\mu_n(Q_{h,k}^{(n)})\xrightarrow[]{P}2^{-h}>0\), it follows that
	$
	\left|F_{h,k}^{(n)}(m_{h,k}^{(n)})-\frac12\right|
	\xrightarrow[]{P}0.
	$
	Combining this with \eqref{eq:cdf-convergence-depth-h-final}, we obtain
	\[
	\left|F_{h,k}(m_{h,k}^{(n)})-\frac12\right|
	\le
	\left|F_{h,k}(m_{h,k}^{(n)})-F_{h,k}^{(n)}(m_{h,k}^{(n)})\right|
	+
	\left|F_{h,k}^{(n)}(m_{h,k}^{(n)})-\frac12\right|
	\xrightarrow[]{P}0.
	\]
	Applying \eqref{eq:invertibility-depth-h-final}, we conclude that
	$
	|m_{h,k}^{(n)}-m_{h,k}|\xrightarrow[]{P}0.
	$
	Since there are only finitely many \(k\) at depth \(h\), this proves \eqref{eq:induction-threshold-final} at depth \(h\).
	
	By induction, both \eqref{eq:induction-threshold-final} and \eqref{eq:induction-cell-final} hold for all \(0\le h\le H-1\). In particular,
	\[
	\max_{\substack{0\le h\le H-1\\1\le k\le 2^h}}
	|m_{h,k}^{(n)}-m_{h,k}|
	\xrightarrow[]{P}0.
	\]
	
	\medskip
	\noindent
	\textbf{Step 3: convergence of address prefixes.}
	Define
	$
	\delta_n
	:=
	\max_{\substack{0\le h\le H-1\\1\le k\le 2^h}}
	|m_{h,k}^{(n)}-m_{h,k}|.
	$
	By Step 2, \(\delta_n\xrightarrow[]{P}0\).
	
	Suppose that \(x\in\mathcal U_n(H)\). Then there exists a first level
	\[
	h_0\in\{0,\dots,H-1\}
	\]
	at which the empirical and population routing decisions differ.
	By definition of \(h_0\), the first \(h_0\) digits agree, so both trees place \(x\) in the same parent cell at depth \(h_0\), say \(Q_{h_0,k_0}\).
	At level \(h_0\), the population split threshold is \(m_{h_0,k_0}\), while the empirical split threshold is \(m_{h_0,k_0}^{(n)}\).
	Since
	$
	|m_{h_0,k_0}^{(n)}-m_{h_0,k_0}|\le \delta_n,
	$
	a change of routing decision can occur only if
	$
	|x_{j(h_0)}-m_{h_0,k_0}|\le \delta_n.
	$
	Therefore,
	\[
	\mathcal U_n(H)
	\subset
	\bigcup_{h=0}^{H-1}\ \bigcup_{k=1}^{2^h}
	\Bigl\{
	x\in Q_{h,k}:\ |x_{j(h)}-m_{h,k}|\le \delta_n
	\Bigr\}.
	\]
	
	Since \(H\) is fixed, the number of sets in the union is finite. For each such set, because \(Q_{h,k}\) is an axis-aligned box and \(f\le M\), its \(\mu\)-mass is bounded by
	$
	C_{h,k}\,\delta_n
	$
	for some finite constant \(C_{h,k}\) depending only on \(d,M\), and \(Q_{0,1}\).
	Summing over the finitely many indices \((h,k)\) with \(0\le h\le H-1\), \(1\le k\le 2^h\), we obtain
	$
	\mu(\mathcal U_n(H))
	\le
	C_H\,\delta_n
	$
	for some constant \(C_H>0\) depending only on \(H,d,M\), and \(Q_{0,1}\).
	Hence
	$
	\mu(\mathcal U_n(H))\xrightarrow[]{P}0.
	$
	This completes the proof.
\end{proof}

\begin{proof}[\textbf{Proof of Theorem~\ref{thm:nn_plateau}}]
	\leavevmode\par\smallskip
	Fix any depth-$H$ address-restricted permutation $\sigma$.
	For each $i\in[n]$, since $\delta_i=\min_{j\in[n]}\|x_i-y_j\|_2$, we have
	$\|x_i-y_{\sigma(i)}\|_2\ge \delta_i$, hence
	$\Gamma_i^{\mathrm{NN}}=(\|x_i-y_{\sigma(i)}\|_2^2-\delta_i^2)_+
	=\|x_i-y_{\sigma(i)}\|_2^2-\delta_i^2$.
	Therefore,
	\[
	\mathrm{RRM}_n^2
	=\frac1n\sum_{i=1}^n\|x_i-y_{\sigma(i)}\|_2^2
	=\frac1n\sum_{i=1}^n\delta_i^2+\frac1n\sum_{i=1}^n\Gamma_i^{\mathrm{NN}}
	\ge \frac1n\sum_{i=1}^n\delta_i^2+\frac1n\sum_{i\in I_+}\Gamma_i^{\mathrm{NN}}
	=\frac1n\sum_{i=1}^n\delta_i^2+\alpha_H\overline{\Gamma}_H^{\mathrm{NN}},
	\]
	with $\overline{\Gamma}_H^{\mathrm{NN}}=0$ when $I_+=\emptyset$.
\end{proof}

%

\bibliographystyle{siamplain}
\bibliography{references}

\end{document}